\documentclass{article}
\usepackage{arxiv, times}
\usepackage{natbib}
\usepackage{amsmath,amssymb}
\usepackage{graphicx}
\usepackage{float}
\usepackage{algorithm,algpseudocode}
\usepackage{booktabs}
\usepackage{longtable,array}
\usepackage{titletoc}
\usepackage{hyperref}
\newtheorem{mtheorem}{Theorem}[section]
\newtheorem{mcorollary}[mtheorem]{Corollary}
\newtheorem{mlemma}[mtheorem]{Lemma}
\newtheorem{mcondition}[mtheorem]{Condition}
\newtheorem{massumption}[mtheorem]{Assumption}
\hypersetup{hidelinks,pdfauthor={},pdftitle={Tensor Bandits}}
\titlecontents{appsection}[1.8em]
  {\bfseries}{\contentslabel{1.8em}}{}
  {\titlerule*[.5pc]{.}\contentspage}
\titlecontents{appsubsection}[4em]
  {}{\contentslabel{2.8em}}{}
  {\titlerule*[.5pc]{.}\contentspage}

\newcommand{\R}{\mathbb{R}}

\newcommand{\Regret}{R_T}
\newcommand{\inner}[2]{\left\langle #1,#2\right\rangle}
\newcommand{\norm}[1]{\left\lVert #1\right\rVert}
\newcommand{\entone}[1]{\norm{#1}_{\mathrm{ent},1}}
\newcommand{\maxnorm}[1]{\norm{#1}_{\max}}
\newcommand{\frobnorm}[1]{\norm{#1}_{F}}
\newcommand{\vecop}{\operatorname{vec}}

\title{Low-Rank Single-Index Bandits with\\Unknown Links: From Matrices to Tensors}

\author{  Zhongxuan Liu \\
  Department of Statistics\\
  University of California, Davis\\
  Davis, CA 95616 \\
  \texttt{bkbliu@ucdavis.edu} \\
  \And
  {Yue Kang} \\
  {Microsoft}\\
  {Redmond, WA 98052}\\
  {\texttt{yuekang@microsoft.com}} \\
  \And
  {Thomas C. M. Lee} \\
  {Department of Statistics}\\
  {University of California, Davis}\\
  Davis, CA 95616 \\
  {\texttt{tcmlee@ucdavis.edu}} \\
}

\begin{document}
\maketitle
\begin{abstract}
Low-rank matrix and tensor bandits exploit structured interactions but typically assume a known reward link. Recent single-index bandit methods accommodate unknown links without directly exploiting matrix or tensor rank. We address this gap by studying stochastic matrix and tensor bandits with an unknown shared Lipschitz link and a low-rank index parameter under known regular candidate distributions and finite-variance noise. For monotone links, T-ESTOR combines robust, rank-adaptive Stein estimation with epoch-based greedy selection. Under exact selected-score access and a uniformly positive selected-design Stein signal, it achieves square-root regret with dimension dependence determined by the low-rank structure. For every admissible design, the monotone lower bound matches the rank, dimension, and horizon dependence up to logarithmic factors at large horizons, for fixed menu size and model/design constants. For nonmonotone links under a nonzero base-law Stein signal, T-BSTOR combines structured estimation with robust bin-based learning and attains the optimal $\widetilde{O}(T^{2/3})$ horizon rate for fixed dimensions, menu size, and model/design constants. Synthetic and CCLE-based experiments illustrate the benefits of structured estimation relative to vectorized and competing single-index baseline methods.
\end{abstract}

\section{Introduction}
\label{sec:introduction}

Contextual bandits model decisions with feedback only for the selected action. In recommendation and experimental design, rewards often depend on interactions among users, items, and contexts, or compounds, cell lines, and doses. Matrix and tensor representations expose these interactions, while low rank reduces the number of unknown parameters. Bilinear and generalized low-rank matrix bandits exploit this structure \citep{jun2019bilinear,lu2021lowrank,kang2022efficient}; tensor bandits extend it to additional feature groups \citep{shi2023tensor,yi2024efficient}.

Low rank and knowledge of the reward response are different assumptions: interactions may admit a compact representation while their effect on the outcome is unknown. Single-index bandits model this uncertainty as $f(\inner{X}{\Theta_*})$, with both the index $\Theta_*$ and shared link $f$ unknown. \citet{kang2026single} and \citet{dey2026optimal} develop estimation-based methods for monotone and nonmonotone rewards. Their vector formulations do not directly exploit matrix or tensor rank: a low-rank parameter can be dense, unlike the coordinate-sparse extension of \citet{kang2026single}. Can regret depend on this interaction structure while retaining single-index horizon rates?

We study low-rank matrix and transformed-slice tensor bandits with an unknown Lipschitz link, known regular joint candidate laws, and finite-variance noise. Building on Stein and spectral estimation \citep{pmlr-v70-yang17a,minsker2018sub,kang2022efficient,kang2026single}, paired observations remove the unrestricted reward offset from moment bounds, while nuclear-norm regularization preserves dependence on matrix rank or total transformed-slice rank without knowing the ranks. We bound selected-score second moments uniformly over ranking directions, connecting these estimation guarantees to greedy decisions. The joint-law analysis permits dependence within candidates and among transformed slices, while preserving the stated rank dependence. T-ESTOR attains $\widetilde O(\sqrt T)$ regret for monotone links under exact selected-score access and a uniformly positive selected-design Stein signal. For nonmonotone links, T-BSTOR controls robust scalar-bin estimates under adaptive sampling, attaining $\widetilde O(T^{2/3})$ regret under a nonzero base-law Stein signal and the stated distributional conditions. For each admissible design, the monotone lower bound matches rank, dimension, and horizon dependence up to logarithmic factors at sufficiently large horizons, for fixed menu size and model/design constants uniformly controlled across dimensions. The nonmonotone lower bound fixes the menu size and positive signal floor, establishing horizon optimality for fixed dimensions and model/design constants under its stated signal and noise conditions. Synthetic and CCLE-based experiments compare structured estimation with vectorized and competing single-index methods.

\subsection{Related Work}
\label{sec:related-work}

\noindent\textbf{Low-rank matrix bandits.} \citet{jun2019bilinear} estimate row and column spaces for bilinear rewards before regularized linear-bandit learning. \citet{lu2021lowrank} extend this approach to matrix actions and known generalized-linear links. \citet{kang2022efficient} use Stein-based subspace estimation; their estimator already combines spectral robustification and a nuclear-norm-penalized quadratic objective, with a single-index extension discussed in \citet{kang2024frameworks}. \citet{pmlr-v235-jang24e} account for arm-set geometry through experimental design, and \citet{kang2024heavytailed} handle heavy-tailed linear rewards by truncation and dynamic exploration.

\noindent\textbf{Tensor bandits.} \citet{zhou2024stochastic} select entries or entity combinations in a low-rank reward tensor, rather than dense array actions with an unknown shared link. \citet{shi2023tensor} use Tucker structure and optimistic selection for linear rewards. \citet{yi2024efficient} use transformed t-product rank and nuclear-norm regularization with a specified generalized-linear link. \citet{li2025unified} accommodate several tensor structures in a generalized-linear explore-then-commit framework.

\noindent\textbf{Single-index bandits.} \citet{kang2026single} study fresh menus and a shared unknown link: ESTOR uses Stein estimation and greedy epochs for monotone links under regular designs and finite-variance noise, with a sparse-vector extension; GSTOR uses kernel regression for nonmonotone links under Gaussian designs. \citet{dey2026optimal} combine index estimation and scalar bins for $T^{2/3}$ regret under a positive-signal condition and conditionally sub-Gaussian noise. These are the algorithmic foundations of T-ESTOR and T-BSTOR. The lower-bound construction of \citet{dey2026optimal} uses a growing menu and a decreasing signal floor; Appendix~\ref{app:nonmonotone-lower} fixes both for each admissible design under its stated conditions. \citet{ma2025nonparametric} study arm-specific rewards under monotonicity, smoothness, and margin conditions, with minimax regret and smoothness adaptation under self-similarity. \citet{arya2026batched} study batched learning with a shared index and arm-specific links; \citet{arya2026kernel} develop kernel/Stein learning and inference under adaptive sampling with persistent arms and arm-specific indices and links. These models differ from our shared low-rank index and common link on fresh menus.

\noindent\textbf{Structured single-index estimation.} \citet{plan2016generalized} analyze the generalized Lasso for offline nonlinear observations under Gaussian measurements. For known non-Gaussian designs, Stein methods cover sparse vectors and low-rank matrices \citep{pmlr-v70-yang17a}, and fourth-order tensors via square matricization \citep{yang2018stein}. \citet{fan2023implicit} study implicit regularization for sparse vectors and symmetric low-rank matrices. These offline guarantees require further analysis for bandits: selection changes the observation law, and estimation errors must be converted into sequential regret.
\section{Preliminaries}
\label{sec:preliminaries}

\subsection{Problem Setting}
\label{sec:bandit-model}

We study stochastic bandits with actions in $\R^{d_1\times\cdots\times d_h}$, where $h=2$ for matrices and $h\geq3$ for higher-order tensors. At round $t\in[T]$, $T\geq1$, the learner observes a set $\mathcal X_t=\{X_{t,a}:a\in[K]\}$ of $K\geq3$ candidates drawn iid from a known joint density $p$, independently of the past. Entries within each candidate may be dependent. Using the observed history, $\mathcal X_t$, and private randomness, the learner selects $X_t\in\mathcal X_t$ and observes only its reward, $y_t=f(\inner{X_t}{\Theta_*})+\eta_t.$ The unknown low-rank parameter $\Theta_*$ and unknown Lipschitz link $f$ are shared across actions and rounds; $\eta_t$ is conditionally mean-zero noise with a variance bound $\sigma^2$, drawn independently from a fixed distribution depending only on $X_t$, and $\inner XZ$ is the entrywise inner product. We first consider nondecreasing links, for which maximizing the index maximizes expected reward; Section~\ref{sec:nonmonotone} extends to nonmonotone links. The goal is to minimize expected cumulative regret $\mathbb E[\Regret]$ across $T$ rounds, where $\Regret$ is defined as,
\begin{equation}
  \Regret=\sum_{t=1}^{T}\left[
    \max_{X\in\mathcal X_t}f\!\left(\inner X{\Theta_*}\right)
    -f\!\left(\inner{X_t}{\Theta_*}\right)\right].
  \label{eq:regret}
\end{equation}

\subsection{Low-Rank Structure}
\label{sec:notation-structures}

For a positive integer $d$, write $[d]=\{1,\ldots,d\}$. An order-$h$ tensor with dimensions $d_1,\ldots,d_h$ has $N=\prod_{\ell=1}^h d_\ell$ entries. We use the inner product $\inner{X}{Z}=\vecop(X)^\top\vecop(Z)$, the Frobenius norm $\frobnorm{X}=\sqrt{\inner{X}{X}}$, the entrywise norms, and the $\ell_1$-ball for some $B>0$:
\[
  \maxnorm X=\norm{\vecop(X)}_\infty,\qquad
  \entone X=\norm{\vecop(X)}_1,\qquad
  \mathcal B(B)=\{Z:\entone Z\leq B\}.
\]
Throughout, we assume that $\Theta_*\in\mathcal B(B)$, with the entrywise $\ell_1$ budget imposed in the original coordinates. We next describe the low-rank structure imposed on $\Theta_*$ in the matrix and tensor settings.

\noindent\textbf{Matrices.}
In the matrix setting, $X,\Theta_*\in\R^{d_1\times d_2}$, and we write $d=\max(d_1,d_2)$. The operator norm $\norm{\cdot}_{\mathrm{op}}$ is the largest singular value, while the nuclear norm $\norm{\cdot}_*$ is the sum of the singular values. For an integer $1\leq r\leq\min(d_1,d_2)$, we consider the class of low-rank matrices with rank at most $r$,
\[
 \mathcal C_{\mathrm M}(r,B)
 =\{Z\in\mathcal B(B):\operatorname{rank}(Z)\leq r\},
 \qquad Z\in\R^{d_1\times d_2}.
 \label{eq:matrix-class}
\]
\noindent\textbf{Transformed-slice tensors.}
Following tensor-product factorizations \citep{kilmer2011factorization} and their transform-based extension \citep{kernfeld2015tensor}, we group the last $h-2$ modes into $q=\prod_{\ell=3}^h d_\ell$, giving a $d_1\times d_2\times q$ array. A fixed, known real orthogonal matrix $U\in\R^{q\times q}$ defines a transform along the grouped mode through $(\mathcal T_U X)^{(k)}_{ij}=\sum_{a=1}^q U_{ka}X_{ija}$, where $(k)$ denotes the $k$th frontal slice. Because $U$ is orthogonal, $\mathcal T_U$ preserves inner products and Frobenius norms, and its inverse is $\mathcal T_{U^\top}$. We impose a rank bound on each transformed slice of $\Theta_*$, allowing the bound to differ across slices. For integer caps $0\leq\rho_k\leq\min(d_1,d_2)$, define $\boldsymbol\rho =( \rho_1, \ldots, \rho_q )$, and consider the class of transformed-slice tensors,
\[
 \mathcal C_U(\boldsymbol\rho,B)
 =\{Z\in\mathcal B(B):
   \operatorname{rank}((\mathcal T_U Z)^{(k)})\leq\rho_k
   \text{ for every }k\in[q]\}.
 \label{eq:slice-class}
\]
The $\ell_1$ budget remains in the original coordinates; only rank constraints are imposed after the transform. Write $r_U=\sum_{k=1}^q\rho_k$ for the sum of the rank caps, so $r_U=qr$ when every slice has the same cap $r$. In both settings, the rank restrictions apply only to $\Theta_*$; the actions $X$ need not be low rank.

\subsection{Score Function and Assumptions}
\label{sec:bandit-assumptions}
Let $p:\R^{d_1\times\cdots\times d_h}\to[0,\infty)$ be the joint Lebesgue density of the candidates. The score function $S^p: \R^{d_1\times\cdots\times d_h} \to \R^{d_1\times\cdots\times d_h}$ associated with the density $p$ is defined as: 
\[
    S^p(X)=-\nabla_X\log p(X) =-\nabla_X p(X) / p(X), \qquad p(X)>0.
\]
The gradient acts on all entries and has the same shape as $X$, with the $p$-a.e. weak interpretation in Appendix~\ref{sec:candidate-law}. We impose the following conditions on the candidate density and its score.
\begin{massumption}[Candidate distribution]
\label{ass:candidate-distribution}
There exists some $A > 0$ such that $\maxnorm X < A$ almost surely, and the joint Lebesgue density $p$ is supported in $\{X:\maxnorm X<A\}$. The zero extension of $p$ has integrable weak first derivatives on $\mathbb R^N$. For a known finite upper bound $M>0$,
\[
  \mathbb E_p[\vecop(S^p(X))\vecop(S^p(X))^\top]\preceq MI_N.
  \label{eq:score-moment-bound}
\]
Here $N=\prod_{\ell=1}^h d_\ell$ and $I_N$ is the identity and $\preceq$ is the positive-semidefinite order.
\end{massumption}
This bounded-support model excludes Gaussian candidate laws. Under Assumption~\ref{ass:candidate-distribution}, integration by parts gives Stein's identity, justified in Appendix~\ref{sec:candidate-law}:
\begin{equation}
  \mathbb E_p[g(X)S^p(X)]=\mathbb E_p[\nabla_X g(X)],
  \label{eq:candidate-density}
\end{equation}
for every smooth compactly supported scalar-valued function $g$ on $\mathbb R^{d_1\times\cdots\times d_h}$. The assumed weak regularity of the zero-extended density ensures that no boundary term appears. Appendix~\ref{app:density-examples} gives the product-beta example used in Section~\ref{sec:experiments}.
\begin{massumption}[Link]
\label{ass:boundedness-link}
 The unknown link $f$ is $L_f$-Lipschitz on $[-AB,AB]$ for some $L_f>0$.
\end{massumption}
We assume that the learner knows the candidate law $p$ and valid model bounds $A$, $B$, together with the transform $U$ in the tensor setting. T-ESTOR additionally requires exact evaluation of the score under its selection rule, although its spectral updates do not require the rank caps or the signal lower bound; see Section~\ref{sec:epoch-framework}. T-BSTOR uses the base score $S^p$ and requires a structural bound and a projection certificate for its exploration schedule, as described in Section~\ref{sec:nonmonotone}. Appendix~\ref{app:unknown-law} gives an estimated-score extension for unknown smooth product laws with slower constructive rates; its learned-selected-score epoch analysis requires additional certificates.

%This implies $\frobnorm{\Theta_*}\leq B$. Combined with the action bound $\maxnorm{X}\leq A$, it also gives
%\[
% |\inner X{\Theta_*}|\leq\maxnorm X\,\entone{\Theta_*}\leq AB,
%\]
%Thus, the single index always lies in $[-AB,AB]$, the interval on which we require the link function to be Lipschitz. 
\providecommand{\EnableBibliography}{}
\section{Methods}
\label{sec:methods}

\subsection{Structured Single-Index Estimation}
\label{sec:structured-estimation}
\label{sec:paired-moments}
\label{sec:matrix-estimation}
\label{sec:tensor-estimation}
\label{sec:slice-estimation}

We use Stein's identity to estimate a scalar multiple of $\Theta_*$ without estimating the unknown link $f$. Fix a known sampling density $p$ with the support and zero-extension regularity of Assumption~\ref{ass:candidate-distribution}, and retain the parameter class of Section~\ref{sec:notation-structures} and the link condition of Assumption~\ref{ass:boundedness-link}. Stein's identity gives
\[
 \mathbb E_p[yS^p(X)]=\mu_*\Theta_*,\qquad
 \mu_*=\mathbb E_p[f'(\inner X{\Theta_*})]\in[-L_f,L_f]
 \quad(\Theta_*\ne0).
\]
For $\Theta_*=0$, we set $\mu_*=0$ without evaluating $f'(0)$. Condition~\ref{cond:selected-signal} ensures uniform positivity under selection. Section~\ref{sec:epoch-framework} applies the results conditionally with $p=p_i$.

To estimate this moment, we must account for the unrestricted offset $f(0)$ and potentially heavy-tailed observations. Although $\mathbb E_p S^p(X)=0$ cancels the offset in the mean, it can still affect the second moment. We remove it by differencing independent observations. Specifically, given $2n$ iid observations $\{(X_j,y_j)\}_{j=1}^{2n}$ from the reward model, form $n$ disjoint paired statistics
\begin{equation}
 Z_j=\frac{y_{2j-1}-y_{2j}}2
 [S^p(X_{2j-1})-S^p(X_{2j})],
 \qquad \mathbb E Z_j=\mu_*\Theta_*,\quad j\in[n].
 \label{eq:paired-observation}
\end{equation}
Each pair is unbiased, independent of the others, and differencing cancels the offset $f(0)$. To control their second moments, set $V=2L_f^2A^2B^2+\sigma^2$, and let $w>0$ bound score second moments under the current observation law:
\[
 \max\{\norm{\mathbb E HH^\top}_{\rm op},
        \norm{\mathbb E H^\top H}_{\rm op}\}\leq w,
\]
Here $H=S^p(X)$ in the matrix case, or $H=(\mathcal T_US^p(X))^{(k)}$ for every transformed slice $k$. Lemma~\ref{lem:paired} then bounds both raw left and right second moments of $Z_j$, or each transformed slice of $Z_j$, by $Vw$ in operator norm. This provides the moment control needed for our estimation guarantee. Section~\ref{sec:epoch-framework} derives a uniform choice of $w$ under selection from the candidate-law bound $M$.

\noindent\textbf{Matrix construction.}
We first reduce the influence of large singular values using $\psi(u)=\operatorname{sign}(u)\log(1+|u|+u^2/2)$ \citep{minsker2018sub}. For an SVD $H=P\operatorname{diag}(s_\ell)Q^\top$ and a tuning parameter $\nu>0$, define
\[
 \widetilde\psi_\nu(H)
 =\frac1\nu P\operatorname{diag}(\psi(\nu s_\ell))Q^\top.
\]
We apply this transformation to each pair before averaging, then use singular-value thresholding to encourage low rank. Specifically, $\operatorname{SVT}_\tau$ replaces each singular value $s$ by $(s-\tau)_+$. The threshold $\lambda/2$ corresponds to a nuclear-norm penalty with coefficient $\lambda$. \citet{kang2022efficient,kang2024frameworks} use the same influence function and penalized quadratic objective, including a single-index extension. Here pairing removes the unrestricted offset from the moment bounds.

\noindent\textbf{Transformed slices, and a common estimator.}
For tensors, we apply the matrix construction to each slice under the supplied orthogonal transform $U$. For $q\geq1$, the sum nuclear norm $\norm Z_{*,U}=\sum_k\norm{(\mathcal T_UZ)^{(k)}}_*$ encodes this slice geometry \citep{lu2019low,song2020robust,lund2020tensor}. We first transform each paired statistic, robustify its slices, and average the corresponding slices:
\[
 \widehat Z_{\nu,k}=\frac1n\sum_{j=1}^n
 \widetilde\psi_\nu\bigl((\mathcal T_UZ_j)^{(k)}\bigr).
\]
Then fit
\begin{equation}
 \widehat\Theta=\mathop{\arg\min}_{Z\in\R^{d_1\times d_2\times q}}
 \left\{\sum_{k=1}^q
 \frobnorm{(\mathcal T_UZ)^{(k)}-\widehat Z_{\nu,k}}^2
 +\lambda\norm Z_{*,U}\right\}.
 \label{eq:slice-fit}
\end{equation}
Because both the loss and penalty separate across transformed slices, we solve this problem by thresholding each slice and applying the inverse transform:
\[
 \widehat\Theta=\mathcal T_U^{-1}
       \bigl([\operatorname{SVT}_{\lambda/2}(\widehat Z_{\nu,k})]_{k=1}^q\bigr).
\]

For $q=1$, $U=[1]$, and $r_U=r$, this is exactly the matrix estimator, with $\norm{\cdot}_{*,U}=\norm{\cdot}_*$ and the tuning parameters below agree with those of the matrix case.

\begin{mtheorem}[Structured estimation rate]
\label{thm:structured-estimation}
For a fixed observation law and $2n$ iid observations satisfying the reward/noise model, Stein regularity, and score-moment bound $w$ of Section~\ref{sec:structured-estimation}, with $\Theta_*\in\mathcal C_U(\boldsymbol\rho,B)$ (or $\mathcal C_{\mathrm M}(r,B)$ in the matrix case), $n\geq1$, and $\delta\in(0,1)$, in Equation~\eqref{eq:slice-fit} take
\[
 \nu=\sqrt{\frac{2\log(2q(d_1+d_2)/\delta)}{nVw}},\qquad
 \lambda=2\sqrt{\frac{2Vw\log(2q(d_1+d_2)/\delta)}n}.
\]
Then, with probability at least $1-\delta$,
\[
 \frobnorm{\widehat\Theta-\mu_*\Theta_*}^2
 \leq\frac{16r_UVw\log(2q(d_1+d_2)/\delta)}n = \widetilde O\left(\frac {r_UVw}{n}\right).
\]
\end{mtheorem}

\subsection{T-ESTOR with Structured Updates}
\label{sec:epoch-framework}

For T-ESTOR and its regret analysis, we assume that $f$ is nondecreasing. This is without loss of generality within the monotone class: replacing $(f,\Theta_*)$ by $(u\mapsto f(-u),-\Theta_*)$ preserves the reward model and parameter constraints.

T-ESTOR applies the preceding estimators within the ESTOR epoch scheme \citep{kang2026single}. Epoch $i=0,1,\ldots$ has $2n_i=2^{i+1}$ rounds, with $n_i=2^i$. Starting from $\widehat\Theta_{-1}=0$, the learner fixes the direction at the beginning of epoch $i$:

\begin{equation}
 v_i=\begin{cases}
 \widehat\Theta_{i-1}/\frobnorm{\widehat\Theta_{i-1}},&\widehat\Theta_{i-1}\ne0,\\
 0,&\widehat\Theta_{i-1}=0.
 \end{cases}
 \label{eq:epoch-direction}
\end{equation}

Throughout the epoch, it selects a candidate maximizing $\inner{v_i}X$, choosing uniformly when $v_i=0$. After completing the epoch, it computes $\widehat\Theta_i$ using only that epoch's observations. Conditional on the epoch-start history, the fixed ranking rule, iid candidate sets, and fresh draws from the fixed conditional noise law yield iid selected action--reward pairs.

Greedy selection changes the action distribution, so the update must use the score of the selected law. For $v_i\ne0$, let $F_i$ be the CDF of $\inner{v_i}X$ under the candidate law $p$. Selecting the largest projection among $K$ candidates gives the density $p_i$; differentiating its log density gives the corresponding score:
\begin{equation}
 \begin{aligned}
 p_i(X)&=Kp(X)F_i(\inner{v_i}X)^{K-1},\\
 S^{p_i}(X)&=S^p(X)-(K-1)\frac{F_i'(\inner{v_i}X)}{F_i(\inner{v_i}X)}v_i
                    \quad(p_i\text{-a.s.}).
 \end{aligned}
 \label{eq:selected-score}
\end{equation}

The factor $F_i^{K-1}$ is the probability that the other $K-1$ candidates have smaller projections, while the factor $K$ accounts for the possible selected indices. See Appendix~\ref{app:stein-pair} for the derivation. When $v_i=0$, we draw a candidate index uniformly using fresh randomness, giving $p_i=p$ and $S^{p_i}=S^p$. We assume exact access to the selected scores in~\eqref{eq:selected-score}. To tune the epoch updates, we need a score second-moment bound that holds uniformly over the ranking directions. For $K\geq3$, Appendix~\ref{app:score-moments} provides
\[
 w=MG_K(d),\qquad
 G_K(d)=2Kd+\frac{K(K-1)^2}{2(K-2)}.
\]

This bound applies to both matrix and transformed-slice scores, including uniform selection when $v_i=0$. Let $Z_{i,j}$ denote the paired statistics in \eqref{eq:paired-observation} formed from epoch $i$'s observations using $S^{p_i}$. Theorem~\ref{thm:structured-estimation} then gives a high-probability bound on $\frobnorm{\widehat\Theta_i-\mu_i\Theta_*}$, where $\mu_i=\mathbb E_{p_i}[f'(\inner X{\Theta_*})]$. After each completed epoch, we compute $\widehat\Theta_i$ from these pairs using Theorem~\ref{thm:structured-estimation}, with $n=n_i$, $\delta=\delta_E=(T+1)^{-2}$, and the bound $w$ above. Algorithm~\ref{alg:T-ESTOR} summarizes the procedure.

\begin{algorithm}[!ht]
\caption{Tensor Epoched Stein's Oracle Single Index Bandit (T-ESTOR)}
\label{alg:T-ESTOR}
\begin{algorithmic}[1]
\algrenewcommand{\algorithmicrequire}{\textbf{Input:}}
\Require Dimensions, $T,K,A,B,L_f,\sigma^2,M,p$;
selected scores~\eqref{eq:selected-score}; $U$ for tensors.
\State $t\leftarrow1$, $i\leftarrow0$, $\widehat\Theta_{-1}\leftarrow0$,
       $\delta_E\leftarrow(T+1)^{-2}$
\While{$t\le T$}
  \State $n_i\leftarrow2^i$, $m_i\leftarrow\min\{2n_i,T-t+1\}$; reset the epoch buffer
  \State Fix $(v_i,p_i,S^{p_i})$ by \eqref{eq:epoch-direction}--\eqref{eq:selected-score}
  \For{$j=1,\ldots,m_i$}
    \State Observe the arm set $\mathcal X_t$
    \If{$v_i=0$}
      \State Draw $a\sim\operatorname{Unif}([K])$;
    \Else
      \State $a\longleftarrow\arg\max_{b\in[K]}\inner{v_i}{X_{t,b}}$;
    \EndIf
    \State $X_t\leftarrow X_{t,a}$; Play $X_t$, observe $y_t$ and append $(y_t,S^{p_i}(X_t))$ to the epoch buffer
    \State $t\leftarrow t+1$
  \EndFor
  \State \textbf{if} $m_i<2n_i$ \textbf{then} \Return \Comment{No partial-epoch update}
  \State Form $Z_{i,1},\ldots,Z_{i,n_i}$ by \eqref{eq:paired-observation} from this epoch's buffer
  \State Fit $\widehat\Theta_i$ by~\eqref{eq:slice-fit} with Theorem~\ref{thm:structured-estimation}'s tuning; Epoch $i\leftarrow i+1$
\EndWhile
\end{algorithmic}
\end{algorithm}

For $\Theta_*\ne0$, the estimate $\widehat\Theta_i$ targets $\mu_i\Theta_*$, where $\mu_i=\mathbb E_{p_i}[f'(\inner X{\Theta_{*}})]$. This multiplier varies with the sampling law. Since $\widehat\Theta_i$ determines the rankings in epoch $i+1$, we require a uniform positive lower bound to translate estimation accuracy into a regret guarantee.

\begin{mcondition}[Selected-design signal]
\label{cond:selected-signal}
For $\Theta_*\ne0$, there exists $\mu_{\min}>0$, possibly depending on $K$, such that $\mu_i\geq\mu_{\min}$ for every sampling law in~\eqref{eq:selected-score} induced by a unit direction $v_i$, and for $p_i=p$ under uniform selection.
\end{mcondition}
For example, the condition holds with $\mu_{\min}=c$ if $f'(u)\geq c>0$ almost everywhere on $[-AB,AB]$. The algorithm does not require knowledge of $\mu_i$ or $\mu_{\min}$.

\subsection{Regret Analysis}
\label{sec:regret-results}

We now derive regret bounds from estimation error under the nondecreasing-link assumption of Section~\ref{sec:epoch-framework} and Condition~\ref{cond:selected-signal}. For $\Theta_*\ne0$, the estimate $\widehat\Theta_{i-1}$ used in epoch $i\geq1$ targets $\mu_{i-1}\Theta_*$, where $\mu_{i-1}\geq\mu_{\min}>0$. Centering candidates at their common mean preserves rankings and allows us to control error projections through the covariance $\Sigma_X=\operatorname{Cov}_p(\vecop(X))$. Lemma~\ref{lem:transfer} then gives
\begin{equation}
 \mathbb E[r_t\mid\mathcal F_{t-1}]
 \leq\frac{2L_f}{\mu_{\min}}\sqrt{K\norm{\Sigma_X}_{\rm op}}\,
                 \frobnorm{\widehat\Theta_{i-1}-\mu_{i-1}\Theta_*}.
 \label{eq:transfer}
\end{equation}

Here $r_t$ is the instantaneous regret in~\eqref{eq:regret}, and $\mathcal F_{t-1}$ is the observed history. The bound also covers uniform selection when $\widehat\Theta_{i-1}=0$; when $\Theta_*=0$, regret is identically zero. Centering is used only in the proof: the algorithm operates on the original candidates and requires no covariance input.

\begin{mtheorem}[Structured epoch regret]
\label{thm:regret}
Assume the sampling and noise model of Section~\ref{sec:bandit-model}, Assumptions~\ref{ass:candidate-distribution} and~\ref{ass:boundedness-link}, and Condition~\ref{cond:selected-signal}, with $f$ nondecreasing. Let $K\geq3$, $q\geq1$, and $\Theta_*\in\mathcal C_U(\boldsymbol\rho,B)$, and write $r_U=\sum_{k=1}^q\rho_k$ and $d=\max(d_1,d_2)$. Run Algorithm~\ref{alg:T-ESTOR} with exact selected scores and the tuning of Theorem~\ref{thm:structured-estimation}, using $\delta_E=(T+1)^{-2}$ in each epoch. If $\Theta_*\ne0$, then
\[
 \mathbb E R_T\leq C\,\frac{L_fK}{\mu_{\min}}
 \sqrt{VM\norm{\Sigma_X}_{\rm op}\,
       r_U(d+K)T\log\!\bigl(2q(d_1+d_2)T\bigr)} = \widetilde O(\sqrt{dr_UT}), 
 \label{eq:structured-regret}
\]
where $C>0$ is universal. The matrix case follows by taking $q=1$, $U=[1]$, and $\boldsymbol\rho=(r)$, giving $r_U=r$. When $\Theta_*=0$, regret is zero. Every policy also satisfies $\mathbb E R_T\leq2L_fABT$.
\end{mtheorem}

\textbf{Proof idea.} Conditional on the history at the start of each epoch, the frozen policy produces iid observations. The uniform selected-score moment bound $w=MG_K(d)$ from Appendix~\ref{app:score-moments} therefore lets us apply Theorem~\ref{thm:structured-estimation} in every epoch. Condition~\ref{cond:selected-signal} and~\eqref{eq:transfer} convert the resulting estimation error into regret during the next epoch. Summing over doubling epochs yields the $\sqrt T$ dependence; the full proof appears in Appendix~\ref{app:regret-proof}.

For fixed $K$ and model constants, the rates are $\widetilde O(\sqrt{drT})$ for matrices and $\widetilde O(\sqrt{dr_UT})$ for tensors, without rank inputs. Appendix~\ref{app:structural-comparisons} compares these upper bounds with vectorization and unfolding: dense rank-one-slice tensors can give an order-$\sqrt d$ advantage over mode unfoldings for $K=O(d)$ under the same law and a linear link.

Worst-case reward and selected-score moment bounds can make the tuning conservative at finite horizons. Section~\ref{sec:experiments} evaluates tuned implementations documented in Appendix~\ref{app:experiments}.

\subsection{Lower Bounds}
\label{sec:lower-bounds}

We next ask whether Theorem~\ref{thm:regret}'s rank and dimension dependence can be improved. A known linear link isolates the difficulty of learning the structured parameter. Fix any candidate law $p$ satisfying Assumption~\ref{ass:candidate-distribution}, retain the fresh iid menus of Section~\ref{sec:bandit-model}, and take $f(z)=\beta z$ with $0<\mu_{\min}\leq\beta\leq L_f$ and independent $\mathcal N(0,\sigma^2)$ noise, $\sigma^2>0$. For every such $p$ this submodel satisfies Condition~\ref{cond:selected-signal} with $\mu_{\min}=\beta$, so it lies in the model of Theorem~\ref{thm:regret} for the same $p$.

Consider $d=\max(d_1,d_2)\geq2$ with a positive rank cap, or total slice rank cap, so $s_{\max}\geq2$. For the admissible ranks in Section~\ref{sec:notation-structures}, set $s_{\max}=dr$ and $\chi=1$ in the matrix case. In the tensor case, assume $r_U\geq1$ and set $s_{\max}=dr_U$. If all slice rank caps equal some $r\geq1$, take $\chi=1$ for any supplied orthogonal transform $U$; otherwise, take $\chi=\max_{k:\rho_k>0}\sum_{\ell=1}^{q}|U_{k\ell}|$. Appendix~\ref{app:lower-proof} builds $s_{\rm lb}$ rank-feasible directions orthogonal to $\mathbb E_pX$, with $s_{\rm lb}\geq\lfloor s_{\max}/2\rfloor\geq s_{\max}/3$ and $s_{\rm lb}=s_{\max}$ for centered designs.

\begin{mtheorem}[Large-horizon lower bound for every design]
\label{cor:lower-large}
Under the model above, let $\mathcal C$ be $\mathcal C_{\mathrm M}(r,B)$ or $\mathcal C_U(\boldsymbol\rho,B)$ with structural scale $s_{\max}\geq2$ as defined above. There is a universal $c>0$ such that, for every admissible candidate law $p$ and every integer $T\geq\max\{1,T_+(p)\}$,
\[
 \inf_\pi\sup_{\Theta_*\in\mathcal C}
       \mathbb E_{\Theta_*,\pi}R_T
 \geq\frac{c\,\sigma}{K\,M\norm{\Sigma_X}_{\rm op}}\sqrt{s_{\max}\,T}.
\]
Here $T_+(p)$ is given in Equation~\eqref{eq:lower-threshold-general}. The infimum ranges over all randomized adaptive policies, which may know $p$, the linear link, and the structural subspace used in the proof.
\end{mtheorem}

For fixed $K$ and model/design constants, uniformly controlled across dimensions, the $\sqrt{s_{\max}T}$ dependence matches Theorem~\ref{thm:regret} up to logarithms for $T\geq\max\{1,T_+(p)\}$. $T_+(p)$ marks a construction branch, not algorithmic burn-in or an established minimax transition.

\subsection{T-BSTOR for arbitrary Links}
\label{sec:nonmonotone}

For nonmonotone links, a larger index need not yield a higher reward, so estimating the parameter direction alone is insufficient for action selection. T-BSTOR first estimates the index up to an unknown scale using uniformly selected candidates, then learns rewards along the fitted index through the estimate-then-bin approach of \citet{dey2026optimal}.

We retain the parameter classes, candidate-law regularity, and finite-variance noise model of Section~\ref{sec:preliminaries}, with $f$ now any $L_f$-Lipschitz link on $[-AB,AB]$. For $\Theta_*\ne 0$, assume $|\mu_*|\geq\gamma>0$, where $\mu_*=\mathbb E_p[f'(\inner X{\Theta_*})]$. This signal may have either sign. Uniform exploration uses only the base score $S^p$, so Condition~\ref{cond:selected-signal} is unnecessary. To control prediction error from the fitted index, write $m=\mathbb E_pX$ and let $\kappa_X\geq1$ satisfy
\begin{equation}
 \mathbb E_p e^{\lambda\inner{X-m}Z}
 \leq\exp\!\left(\frac{\lambda^2A^2\kappa_X^2\frobnorm Z^2}{2}\right)
 \quad\text{for every }Z\text{ and }\lambda\in\mathbb R.
 \label{eq:nonmonotone-joint-concentration}
\end{equation}
Independent bounded entries admit $\kappa_X=1$, and every bounded-entry law admits $\kappa_X=\sqrt N$; Appendix~\ref{app:nonmonotone-model} discusses dependent designs.

For an even exploration length $n_0<T$, select candidates uniformly during the first $n_0$ rounds and fit the structured estimator using $n_0/2$ disjoint observation pairs. Project the fit in Frobenius norm onto the entrywise ball $\mathcal B(L_fB)$ and hold it fixed. This keeps every fitted index in $[-AL_fB,AL_fB]$, which we partition into bins of width at most $w_{\rm bin}$. On each subsequent round, T-BSTOR chooses among bins containing current candidates: it prioritizes bins with too few observations and otherwise uses median-of-means upper confidence bounds to select a bin and plays a candidate in it. Algorithm~\ref{alg:nonmonotone} gives the exact policy and tuning. The concentration bound above converts estimation error into prediction error, while Lipschitz continuity controls reward variation within each bin; Appendix~\ref{app:nonmonotone-algorithm} makes this approximation precise.

\begin{samepage}
\begin{mtheorem}[Horizon rate for arbitrary links]
\label{thm:nonmonotone-main}
Under these local assumptions, run Algorithm~\ref{alg:nonmonotone} with a supplied bound $D\geq1$, $D\geq dr$ (matrices) or $D\geq dr_U$ (transformed slices), the confidence allocation~\eqref{eq:nonmonotone-schedule}, and
\[
 n_0=\min\!\left\{T,\;2\left\lceil\frac{(\kappa_X^2D)^{1/3}T^{2/3}}2\right\rceil\right\},\qquad
 w_{\rm bin}=2AL_fB\,T^{-1/3}.
\]
If $n_0=T$, select uniformly throughout; otherwise, run the estimation and binning steps above. Then
\[
 \mathbb E[R_T]=\widetilde O\!\left(
 \big[(\kappa_X^2D)^{1/3}+1\big]T^{2/3}\right).
\]
Moreover, $R_T\leq2L_fABT$. Here $A,B,L_f,M,\gamma,\sigma,K,h$ are held fixed, $D$ and $\kappa_X$ remain explicit, and only logarithms in dimensions and $T$ are suppressed.
\end{mtheorem}
\end{samepage}
Only the exploration schedule needs the supplied $D$ and $\kappa_X$; spectral fitting needs no rank caps. Under the same projection certificate and comparable model constants, the structured fit replaces the vectorized leading upper-bound factor $(\kappa_X^2N)^{1/3}$ by $(\kappa_X^2D)^{1/3}$, improving the bound when $D<N$.

For each admissible candidate law, Theorem~\ref{thm:nonmonotone-lower} gives an $\Omega(T^{2/3})$ lower bound at large horizons with fixed $K\geq3$ and positive signal floor, even when a nonzero structured parameter is disclosed. Dey et al.'s lower-bound construction instead grows the menu and decreases the signal floor \citep{dey2026optimal}. Thus the horizon exponent is optimal for fixed dimensions, $K\geq3$, and model/design constants, with $\sigma>0$ and $0<\gamma<L_f$. This comparison concerns only the horizon rate; its constants depend on the design. Links with zero base signal remain outside the guarantee, while a zero parameter has zero regret.

\begin{figure}[!t]
\centering
\includegraphics[width=.90\linewidth]{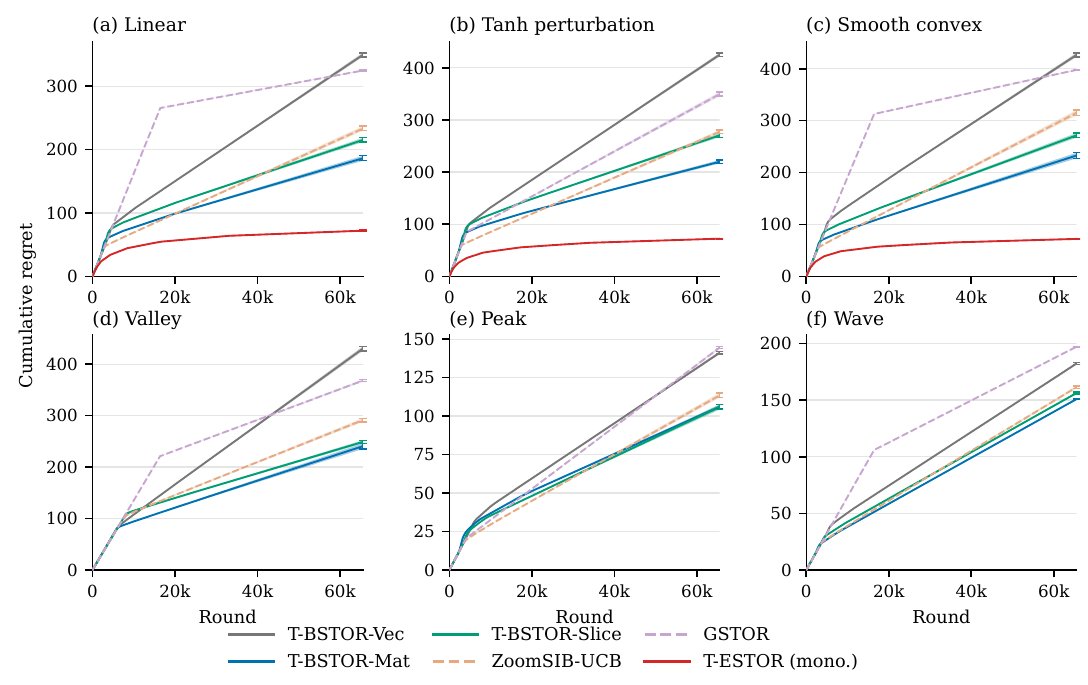}
\caption{Mean cumulative regret over 50 paired seeds, with pointwise 95\% bootstrap bands and endpoint whiskers. Top: monotone links; bottom: nonmonotone. Solid: proposed methods; dashed: external baselines.}
\label{fig:comparison-synthetic}
\end{figure}

\section{Experiments}
\label{sec:experiments}

\textbf{Design.} We use a shared dense $16\times16\times4$ parameter with an entrywise $\ell_1$ budget of one. Its $16\times64$ matrix representation and each frontal slice under the identity transform have rank two, giving both representations a structural scale of 128, compared with an ambient dimension of 1024. We consider six links, three monotone and three nonmonotone, with $K=10$ candidates per round. Candidate entries are drawn independently from $2\operatorname{Beta}(5,5)-1$, and rewards are corrupted by independent scaled Student-$t_3$ noise. Appendix~\ref{app:experiments} provides parameters, links, and noise scaling.

\textbf{Methods and evaluation.}
T-BSTOR-Vec, T-BSTOR-Mat, and T-BSTOR-Slice combine vector, matrix, and transformed-slice estimation, respectively, with the same robust scalar-bin learner. We compare these methods with GSTOR and ZoomSIB-UCB, and evaluate T-ESTOR, which exploits monotonicity, only on the monotone links. We tune influence and threshold multipliers; T-BSTOR also uses adaptive exploration \citep{dey2026optimal}, normalized-index bins, and tuned confidence multipliers (Appendix~\ref{app:experiments}). The guarantees concern the stated theoretical policies and tuning. T-ESTOR uses numerical selected scores; Appendix~\ref{app:numerical-oracle} reports accuracy and computational cost. Tuning, validation, and testing use separate seeds. With configurations frozen, we evaluate 50 fresh paired seeds through $T=65{,}534$. Methods share instances but observe only their chosen rewards. Regret uses each menu's best action and includes all exploration, initialization, and fallback rounds.

\textbf{Results.}  Both structured T-BSTOR variants achieve lower mean final regret than T-BSTOR-Vec, GSTOR, and ZoomSIB-UCB on all six links (Figure~\ref{fig:comparison-synthetic}). All reported unadjusted paired 95\% intervals for the prespecified Slice comparisons favor Slice (Table~\ref{tab:comparison-v2-synthetic-paired}). Mean direction errors are 0.41--0.63 for structured fits and 0.66--0.80 for vectorization (Table~\ref{tab:comparison-v2-synthetic-fit-diagnostics}). Exploration lengths and tuning also differ, so regret gains cannot be attributed to structure alone. On monotone links, T-ESTOR's mean final regret is approximately 72, versus 190--230 for the best binning variant.

\textbf{CCLE-based evaluation.}
We use Cancer Cell Line Encyclopedia (CCLE) drug-response data to construct a coefficient for simulated dense actions and feedback. On Linear, T-ESTOR improves on ESTOR and achieves mean regret comparable to G-LowTESTR despite not knowing the link. On Wave, T-BSTOR-Slice-DCT has lower mean regret than vectorization, GSTOR, and ZoomSIB-UCB; G-LowTESTR and the Cell and Dose unfoldings have lower means than Slice-DCT. These comparisons are descriptive. Appendix~\ref{app:ccle} gives the setup and detailed results.

%We also evaluate the methods using the Cancer Cell Line Encyclopedia (CCLE) data, which provides molecular profiles and drug-response measurements for human cancer cell lines. We construct a $16\times24\times8$ coefficient block from its response data without projecting it onto an exact low-rank model, and simulate dense actions and feedback as detailed in Appendix~\ref{app:ccle}. At $T=20{,}000$, mean regret under the Linear link is 57.1 for T-ESTOR, 59.2 for G-LowTESTR with a correctly specified linear link, and 90.6 for ESTOR. Thus, despite not knowing the reward link, T-ESTOR attains similar empirical performance to a baseline supplied with the correct linear link. T-ESTOR and ESTOR are evaluated only on Linear. Under the nonmonotone Wave link, T-BSTOR-Slice-DCT also outperforms vectorization, GSTOR, and ZoomSIB-UCB in mean regret, extending the favorable comparison with these baselines beyond monotone rewards. However, G-LowTESTR achieve lower means than T-BSTOR-Slice-DCT, even though G-LowTESTR assumes an identity link that is misspecified for Wave.
\section{Conclusion}

We studied low-rank matrix and tensor bandits with an unknown shared Lipschitz link and finite-variance noise. Under the stated assumptions, T-ESTOR achieves square-root regret for monotone links with dimension dependence determined by low-rank structure. For every admissible design, the monotone lower bound matches rank, dimension, and horizon dependence up to logarithms at large horizons, for fixed menu size and model/design constants. T-BSTOR attains the optimal $\widetilde O(T^{2/3})$ nonmonotone horizon rate for fixed dimensions, menu size, and model/design constants. Synthetic and CCLE-based experiments illustrate the benefits of structured estimation.

\paragraph{Limitations and future work.} The main-policy guarantees rely on known regular candidate distributions, supplied tensor transforms, and the stated Stein-signal conditions; T-ESTOR additionally requires exact selected scores. Future work includes certifying learned selected scores, handling zero first-order signals, and establishing guarantees for the adaptive exploration used in our experiments.

%\clearpage
%\input{sections/statements}
%\ifdefined\EnableBibliography
\bibliography{references}
\bibliographystyle{unsrtnat}
%\fi
\clearpage
\appendix
% The named partial contents starts here, independently of the last main section.
\startcontents[appendix]
\section*{Appendix Contents}
\printcontents[appendix]{app}{1}[2]{\setlength{\parskip}{0pt}}

\section*{Roadmap}
Appendix~\ref{app:notation-guide} collects notation. Appendix~\ref{app:candidate-distribution} establishes the Stein identities and admissible designs. Appendix~\ref{app:unknown-law} develops the estimated-score interface, a constructive extension to unknown product candidate laws, and a conditional learned-score epoch analysis. Appendix~\ref{app:method-proofs} proves the structured estimation guarantees; Appendices~\ref{app:stein-pair}--\ref{app:regret-proof} establish selected-score bounds, monotone regret, and structural comparisons. Appendix~\ref{app:lower-proof} proves the large-horizon lower bound for every admissible design, with a sharp-constant product-beta corollary. Appendix~\ref{app:nonmonotone} gives the nonmonotone analysis, bin-confidence bounds, exploration-schedule sensitivity, and horizon lower bound for fixed dimensions, menu size, and model/design constants. Appendices~\ref{app:experiments}--\ref{app:ccle} describe the synthetic and CCLE-derived experiments.

\clearpage
\section{Notation}
\label{app:notation-guide}

Principal notation used in the main text.

% Keep this single table on one page. Include only important main-text
% symbols, group related uses, and keep each meaning to one short sentence.
\begingroup
\raggedbottom
\renewcommand{\arraystretch}{1.15}
\setlength{\LTleft}{0pt}
\setlength{\LTright}{0pt plus 1fill}
\begin{longtable}{@{}>{\raggedright\arraybackslash}p{0.31\linewidth}
  >{\raggedright\arraybackslash}p{\dimexpr0.69\linewidth-2\tabcolsep\relax}@{}}
\toprule
\textbf{Symbol} & \textbf{Meaning}\\
\midrule
\endfirsthead
\toprule
\textbf{Symbol} & \textbf{Meaning (continued)}\\
\midrule
\endhead
\bottomrule
\endlastfoot
$T,\ t,\ K$ & Time horizon, round index, and candidates per round.\\
$\mathcal X_t,\ X_{t,a},\ X_t$ & Candidate menu, candidate $a$, and chosen action at round $t$.\\
$\Theta_*,\ \widehat\Theta$ & True parameter and structured fit targeting $\mu_*\Theta_*$.\\
$f,\ L_f$ & Reward link and its Lipschitz bound.\\
$y_t,\ \eta_t,\ \sigma^2$ & Reward, noise, and conditional noise-variance bound.\\
$R_T,\ r_t$ & Cumulative regret and its round-$t$ summand.\\
$A,\ B$ & Action entry bound and parameter entrywise-$\ell_1$ budget.\\
$h,\ d_\ell,\ N$ & Array order, mode sizes, and ambient size $N=\prod_\ell d_\ell$.\\
$d_1,\ d_2,\ d,\ q$ & Leading dimensions, $d=\max(d_1,d_2)$, and slice count.\\
$r,\ \rho_k,\ r_U$ & Rank cap, slice rank caps, and their sum $r_U=\sum_k\rho_k$.\\
$U,\ \mathcal T_U,\ X^{(k)}$ & Orthogonal transform, its tensor action, and frontal slice $k$.\\
$\mathcal C_{\mathrm M},\ \mathcal C_U$ & Matrix and transformed-slice parameter classes.\\
$\mathcal B(B)$ & Entrywise-$\ell_1$ ball of radius $B$.\\
$p,\ p_i$ & Base or locally fixed sampling density, and epoch sampling density.\\
$S^p,\ S^{p_i}$ & Full joint scores of these densities, in array shape.\\
$M,\ \Sigma_X$ & Known joint score-moment bound and covariance $\Sigma_X=\operatorname{Cov}_p(\vecop(X))$.\\
$m,\ \kappa_X$ & Candidate mean $\mathbb E_pX$ and centered joint-projection concentration certificate.\\
$\mu_*,\ \mu_i$ & Actual Stein multipliers for a fixed sampling law and epoch $i$.\\
$\mu_{\min},\ \gamma$ & Uniform positive selected-signal and absolute base-signal lower bounds.\\
$n,\ j,\ Z_j$ & Disjoint pair count ($2n$ observations), pair index, and paired reward--score statistic.\\
$V,\ w$ & Paired-reward factor $2L_f^2A^2B^2+\sigma^2$ and score second-moment bound.\\
$\widehat Z_\nu,\ \widehat Z_{\nu,k}$ & Robust empirical Stein moment and its transformed-slice counterpart.\\
$\psi,\ \widetilde\psi_\nu$ & Scalar influence and its spectral extension including the factor $1/\nu$.\\
$\nu,\ \lambda$ & Spectral influence tuning and nuclear-penalty coefficient.\\
$\delta,\ \delta_E$ & Failure probability and per-epoch choice $\delta_E=(T+1)^{-2}$.\\
$i,\ n_i,\ v_i,\ \widehat\Theta_i$ & Epoch, $n_i=2^i$ pairs, frozen direction, and epoch-$i$ fit targeting $\mu_i\Theta_*$ for epoch $i+1$.\\
$F_i,\ F_i',\ G_K(d)$ & Projection CDF and density; selected-score inflation factor.\\
$\mathcal C,\ s_{\max}$ & Chosen lower-bound parameter class and hard-subspace dimension.\\
$c_X,\ M_0$ & Exact beta-design covariance and score moments $A^2/11$ and $12/A^2$ (Appendix~\ref{app:density-examples}, Corollary~\ref{cor:lower-beta}).\\
$\chi,\ T_+(p)$ & Entrywise cost of the hard subspace and large-horizon threshold.\\
$s_{\rm lb},\ \chi_{\rm lb}$ & Number and entrywise cost of the mean-orthogonal hard directions.\\
$\beta$ & Disclosed linear-link slope in the lower bound.\\
$n_0,\ w_{\rm bin},\ D$ & Actual exploration count (even when $n_0<T$), bin width, and a supplied bound on $dr$ or $dr_U$.\\
\end{longtable}
\endgroup

\clearpage
\section{Candidate Distribution and Stein Identities}
\label{app:candidate-distribution}

\subsection{Joint Regularity and the Base Identity}
\label{sec:candidate-law}
Write $x=\vecop(X)$ and extend the density $p$ by zero to all of $\mathbb R^N$. The space $W^{1,1}(\mathbb R^N)$ consists of integrable functions whose first distributional derivatives are integrable. The zero-extension condition and weak score in Assumption~\ref{ass:candidate-distribution} mean that
\[
 p\in W^{1,1}(\mathbb R^N),\qquad
 \nabla_x p(x)=-p(x)\,\vecop(S^p(X))\quad\text{almost everywhere}.
\]
Under score integrability, these properties are equivalent to the compact-test Stein identity in Equation~\eqref{eq:candidate-density}. Indeed, for every $g\in C_c^\infty(\mathbb R^N)$, the identity says $\int g p\,\vecop(S^p)=\int p\nabla g$. By the definition of a distributional derivative this is precisely $\nabla p=-p\vecop(S^p)$. The right-hand side is integrable, so $p$ belongs to $W^{1,1}$. Conversely, this weak derivative equation gives the compact-test identity by integration by parts. The full joint second-moment bound implies the required score integrability by Cauchy--Schwarz. Because the derivative identity holds on all of $\mathbb R^N$, it includes the support boundary and permits no additional boundary measure. On $\{p>0\}$ it identifies the weak score with $-\nabla_X\log p$; the weak derivative of a nonnegative Sobolev function is zero almost everywhere on its zero set. A smooth cutoff equal to one on the bounded support then gives $\mathbb E_p\vecop(S^p(X))=-\int\nabla p=0$.

For $\Theta_*\ne0$, orthogonal coordinates and Fubini show that $\inner X{\Theta_*}$ has a density: its law is a scaled one-dimensional marginal of the joint density. Extend $f$ as a Lipschitz function outside $[-A B,A B]$ and apply weak integration by parts to its composition with this index, multiplied by the same cutoff. The Sobolev chain rule gives gradient $f'(\inner X{\Theta_*})\vecop(\Theta_*)$ almost everywhere. The index density ensures that the exceptional set where $f'$ is undefined has probability zero. The index is bounded, and the finite score second moment makes every product integrable for each finite, unrestricted value of $f(0)$. Conditional mean-zero noise contributes zero, proving the base identity
\[
 \mathbb E_{p}[y S^p(X)]=\mu_*\Theta_*,\qquad
 \mu_*=\mathbb E_{p}
 [f'(\inner X{\Theta_*})]\in[-L_f,L_f].
 \label{eq:population-stein}
\]
When $\Theta_*=0$, the signal part is constant and the moment vanishes by score and noise centering, without evaluating $f'(0)$.

\subsection{Explicit Admissible Densities}
\label{app:density-examples}
\noindent\textbf{Product-beta design.}
\label{ex:beta-design}
With $\mathcal I=[d_1]\times\cdots\times[d_h]$, a concrete admissible density is
\begin{equation}
  p_\beta(X)=\prod_{\boldsymbol i\in\mathcal I}p_{A}(X_{\boldsymbol i}),\qquad
  p_{A}(u)=\frac{315}{256A}
    \left(1-\frac{u^2}{A^2}\right)^4\mathbf1\{|u|<A\}.
  \label{eq:beta-density}
\end{equation}
Write $x=\vecop(X)$. The score and moments are
\begin{equation*}
 \begin{aligned}
 (S^{p_\beta}(X))_{\boldsymbol i}
   &=\frac{8X_{\boldsymbol i}}{A^2-X_{\boldsymbol i}^2},
 &c_X&=A^2/11,\quad M_0=12/A^2,\\
 \operatorname{Cov}_{p_\beta}(x)&=c_XI_N,
 &\mathbb E_{p_\beta}[\vecop(S^{p_\beta}(X)) \vecop(S^{p_\beta}(X))^\top]&=M_0I_N.
 \end{aligned}
 \label{eq:beta-covariance}
\end{equation*}
The constants $c_X$ and $M_0$ refer only to this beta design, including Corollary~\ref{cor:lower-beta}, the product-beta remark in Appendix~\ref{app:nonmonotone-lower}, and the experiments. With $u=x_i/A$, normalization and the moments follow from $\int_{-1}^1u^{2j}(1-u^2)^a\,du=\mathrm B(j+1/2,a+1)$; the fourth moment is $\mathbb E x_i^4=3A^4/(11\cdot13)$. The fourth-order taper has zero boundary trace and integrable first derivative, so the zero-extended product density is in $W^{1,1}$. Symmetry centers coordinates and scores; independence gives the exact joint covariance $c_XI_N$ and score second moment $M_0I_N$. For this example only and any nonzero $v$, the density of $\inner vX$ is the convolution over nonzero coordinates $v_i$ of $|v_i|^{-1}p_{A}(u/v_i)$. It is positive in the interior of $[-A\entone v,A\entone v]$ and vanishes at its endpoints. Taking $\norm{\Sigma_X}_{\rm op}=c_X$ and $M=M_0$ in Theorem~\ref{thm:regret} recovers the beta-design guarantee, subject to its nondecreasing-link, selected-score access, and signal assumptions.

\noindent\textbf{Correlated, noncentered designs.}
\label{ex:dependent-design}
Let $z$ have the product-beta law~\eqref{eq:beta-density} with $A_0$ in place of $A$, and set $x=\vecop(X)=b+Hz$ for known $b$ and invertible $H$ with $\norm b_\infty+A_0\norm H_{\infty\to\infty}\leq A$. Translation and an invertible linear change of variables preserve the zero-extension $W^{1,1}$ regularity, and
\[
 p_X(x)=\frac{p_\beta(H^{-1}(x-b))}{|\det H|},\qquad
 \vecop(S^{p_X}(X))=H^{-\top}\vecop(S^{p_\beta}(Z)),
\]
where $\vecop(Z)=H^{-1}(x-b)$. Hence $\Sigma_X=(A_0^2/11)HH^\top$ and $M=(12/A_0^2)\norm{H^{-1}}_{\rm op}^2$, so $\norm{\Sigma_X}_{\rm op}M=(12/11)\operatorname{cond}_2(H)^2$. Hoeffding's lemma applied to $z$ gives \eqref{eq:nonmonotone-joint-concentration} with $\kappa_X=\max\{1,(A_0/A)\norm H_{\rm op}\}$. For example, take block-diagonal $H$ with blocks
\[
 \frac1{1+\rho}\begin{pmatrix}1&\rho\\\rho&1\end{pmatrix},
 \qquad 0<\rho<1,
\]
with at least one two-coordinate block and an identity block for an unmatched coordinate. Its row-sum and operator norms are one, and its condition number is $(1+\rho)/(1-\rho)$, independent of $N$. With $A_0=A/2$, any $\norm b_\infty\leq A/2$ is allowed; paired coordinates have covariance $2\rho A_0^2/\{11(1+\rho)^2\}>0$, and $\kappa_X=1$ with a dimension-free $M$ for fixed $\rho$. The structured parameter is not transformed; selected-score access and the signal conditions remain assumptions.

\section{Unknown Distributions and Estimated Scores}
\label{app:unknown-law}

\subsection{Scope and Overview}
\label{sec:unknown-law}

This appendix treats an unknown product candidate law on $(-A,A)^N$ with zero-boundary, absolutely continuous marginal densities and square-integrable scores. Retain the support, entrywise budgets, and absolutely continuous $L_f$-Lipschitz link regularity of Section~\ref{sec:preliminaries}, together with the fixed conditional noise kernel used in Section~\ref{sec:epoch-framework}. For the monotone regret results below, impose the nondecreasing orientation of Section~\ref{sec:epoch-framework} and require $\mu_*:=\mathbb E_p f'(\inner X{\Theta_*})\geq\mu_{\min}>0$ for $\Theta_*\ne0$. This is a positive \emph{base-law} signal condition. The learner sees all menu candidates but knows neither $p$ nor its score.

Independent unlabeled calibration yields score-error and fitted-score covariance certificates. Theorem~\ref{thm:estimated-score-error} transfers them to matrix and transformed-slice estimation, and Sections~\ref{app:explore-commit}--\ref{app:calibration-rates} give a calibration--exploration--commit policy with a finite score-matching construction~\citep{hyvarinen2005estimation} on an explicit smooth product class. For fixed class, model and structural quantities, its regret is
\begin{equation}
 \mathbb E R_T=
 \widetilde O\!\left(T^{2/3}+T^{(2s+14)/(3s+14)}\right).
 \label{eq:unknown-smooth-rate}
\end{equation}
An exactly specified fixed-dimensional score family gives $\widetilde O(T^{2/3})$. These are constructive, not asserted optimal, unknown-density rates. They concern the separate calibration--exploration--commit policy; extending greedy T-ESTOR also requires direction-valid selected-score certificates (Section~\ref{app:unknown-greedy}). A zero parameter has zero regret without evaluating $f'(0)$.

\subsection{Calibration Interface and Perturbed Structured Estimation}
\label{app:score-perturbation}

Let $p$ be a fixed observation density whose zero extension lies in $W^{1,1}(\mathbb R^N)$ with square-integrable score, as in Appendix~\ref{sec:candidate-law}. Retain the support, Lipschitz regularity, budget and noise assumptions above, but impose neither monotonicity nor a nonzero Stein signal for the estimation result; here $p$ need not be a product law and may be a frozen selected law. Conditional on a calibration history $\mathcal G$, $p$ and $\widehat S$ are fixed, and subsequent action--reward observations are iid with that law and reuse no calibration observations. On a calibration event of probability at least $1-\delta_{\rm cal}$, suppose deterministic certificates $\varepsilon\geq0$, $\widehat J>0$ satisfy
\begin{equation}
 \operatorname{Cov}_p(\vecop(\widehat S-S^p))\preceq\varepsilon^2I_N,
 \qquad
 \operatorname{Cov}_p(\vecop\widehat S)\preceq\widehat JI_N.
 \label{eq:score-certificates}
\end{equation}
All expectations below condition on $\mathcal G$. Write $V=2L_f^2A^2B^2+\sigma^2$ and, for $\Theta_*\ne0$, put $\mu_*=\mathbb E_p f'(\inner X{\Theta_*})\in[-L_f,L_f]$; otherwise set $\mu_*=0$. The estimation bound allows signed or zero multipliers.

\begin{mtheorem}[Structured error with estimated scores]
\label{thm:estimated-score-error}
Conditional on successful calibration, form $\widehat Z_j=\tfrac12(y_{2j-1}-y_{2j}) [\widehat S(X_{2j-1})-\widehat S(X_{2j})]$ from $n\geq1$ fresh pairs. For $\Theta_*\in\mathcal C_{\mathrm M}(r,B)$ or $\mathcal C_U(\boldsymbol\rho,B)$, apply the corresponding spectral estimator with the tuning in Equation~\eqref{eq:matrix-tuning} or~\eqref{eq:slice-tuning}, using $w=d\widehat J$ and $d=\max(d_1,d_2)$. For $\delta\in(0,1)$, with conditional probability at least $1-\delta$, respectively,
\begin{equation*}
 \frobnorm{\widehat\Theta-\mu_*\Theta_*}
 \leq
 \begin{cases}
 4\sqrt{Vrd\widehat J\log(2(d_1+d_2)/\delta)/n}
       +L_fAB\varepsilon,&\text{matrix},\\[2pt]
 4\sqrt{Vr_Ud\widehat J\log(2q(d_1+d_2)/\delta)/n}
       +L_fAB\varepsilon,&\text{transformed slices}.
 \end{cases}
 \label{eq:estimated-score-radius}
\end{equation*}
\end{mtheorem}

\noindent\textbf{Proof.}
Write $F(X)=f(\inner X{\Theta_*})$ and $e=\widehat S-S^p$. Conditional mean-zero noise and Stein integration by parts give $\mathbb E\widehat Z_j=\operatorname{Cov}_p(F,\widehat S) =\mu_*\Theta_*+B_p$, where $B_p=\operatorname{Cov}_p(F,e)$. For a unit array $D$, Cauchy--Schwarz and the range bound for $F$ imply
\[
 |\inner{B_p}D|\leq
 \sqrt{\operatorname{Var}_p F}\sqrt{\operatorname{Var}_p\inner eD}
 \leq L_fAB\varepsilon.
\]
Thus $\frobnorm{B_p}\leq L_fAB\varepsilon$ without a dimension factor. An independent score difference has second moment twice its covariance; Lemma~\ref{lem:paired}'s paired-moment proof therefore gives
\[
 \mathbb E[\vecop(\widehat Z_j)\vecop(\widehat Z_j)^\top]
 \preceq V\operatorname{Cov}_p(\vecop\widehat S)
 \preceq V\widehat JI_N.
 \label{eq:approximate-pair-second}
\]
The left and right raw moments of each matrix or transformed slice are bounded by $dV\widehat J$. Spectral concentration controls the influenced average about $\mu_*\Theta_*+B_p$. Subtract $B_p$ from that average only for analysis and apply the thresholding argument of Appendices~\ref{app:matrix-proof}--\ref{app:slice-proof} to the structured target. Frobenius nonexpansiveness of the nuclear-norm proximal map bounds the change on restoring $B_p$ by $\frobnorm{B_p}$, without a rank assumption on the bias. This comparison is analytical: the learner never subtracts $B_p$. Pairing cancels constant reward or score offsets.

\subsection{Calibration--Exploration--Commit}
\label{app:explore-commit}

Return to the nondecreasing link, the unknown product \emph{base} law $p$, and its positive signal $\mu_*\geq\mu_{\min}$ for $\Theta_*\ne0$; write $\Sigma_X=\operatorname{Cov}_p(\vecop X)$, used only in the regret analysis. Choose deterministic integers $b,n\geq1$ with $b+2n\leq T$. During $b$ uniform-selection calibration rounds, discard rewards and fit scores from all menus. For the next $2n$ rounds, choose uniform indices independently of menus and past, giving $n$ disjoint iid pairs under $p$ independent of calibration. Fit once by Theorem~\ref{thm:estimated-score-error}, then commit greedily to this fixed fit for $H_T=T-b-2n$ rounds, selecting uniformly if the fit is zero. Commit rewards are not reused.

For this subsection and Section~\ref{app:unknown-greedy}, let $e_{\rm fit}(n,\delta,J,\varepsilon)$ be the deterministic right-hand side of the norm bound in Theorem~\ref{thm:estimated-score-error} for the chosen fit, with $J$ in place of $\widehat J$. It includes both the statistical term and its stated bias constant. Put $\Delta=2L_fAB$. If calibration certifies $(\varepsilon_b,\widehat J)$ with failure probability $\delta_{\rm cal}$, and the reward fit uses confidence $\delta_{\rm rew}$, then
\begin{equation}
 \begin{aligned}
 \mathbb E R_T\leq\min\biggl\{\Delta T,\;&\Delta(b+2n)
  +\Delta H_T(\delta_{\rm cal}+\delta_{\rm rew})\\
 &+\frac{2L_f\sqrt{K\norm{\Sigma_X}_{\rm op}}}{\mu_{\min}}H_T
     e_{\rm fit}(n,\delta_{\rm rew},\widehat J,\varepsilon_b)\biggr\}.
 \end{aligned}
 \label{eq:unknown-regret}
\end{equation}
The certificates are deterministic bounds valid on the calibration event.

To prove the bound, center a fresh menu by $\overline X=\mathbb E_pX$ in the transfer argument. This preserves fitted rankings and true index differences, giving
\[
 \mathbb E[r_t\mid\mathcal F_{t-1}]
 \leq\frac{2L_f}{\mu_{\min}}\sqrt{K\norm{\Sigma_X}_{\rm op}}\,
       \frobnorm{\widehat\Theta-\mu_*\Theta_*}.
\]
The transfer uses the nondecreasing link and positive base multiplier; no selected-law signal condition is needed because the commit-stage fit is not updated. The learner need not know or subtract $\overline X$. The first $b+2n$ rounds cost at most $\Delta$ each. Conditioning on calibration and the reward fit bounds failure by $\delta_{\rm cal}+\delta_{\rm rew}$. Apply the norm bound on success and $\Delta$ on failure in each commit round; the deterministic cap $\Delta T$ supplies the minimum.

\subsection{An Explicit Smooth Product Class}
\label{app:smooth-class}

Fix known constants $0<c\leq1\leq C$, $s>0$, and $R\geq0$. Let $\phi_k$, $k=0,1,\ldots$, be the Jacobi $(4,4)$ polynomials normalized to an orthonormal basis in $L^2(\omega)$ and positive at $1$, where $\omega(u)=(315/256)(1-u^2)^4\mathbf1\{|u|<1\}$. The unknown marginal $j$ has density
\[
 p_j(x)=A^{-1}\omega(x/A)r_j(x/A),\quad
 c\leq r_j\leq C,\quad \int_{-1}^1\omega r_j=1,
\]
with $\log r_j$ locally absolutely continuous and
\begin{equation}
 h_j=-(\log r_j)'=\sum_{k\geq0}a_{jk}\phi_k\quad\text{in }L^2(\omega),
 \qquad \sum_{k\geq0}(1+k)^{2s}a_{jk}^2\leq R^2.
 \label{eq:jacobi-smoothness}
\end{equation}
Equation~\eqref{eq:jacobi-smoothness} specifies the approximation class; the positive lower bound is on the density ratio, not the density at the boundary. For $c<1<C$ and $R>0$ the class is infinite-dimensional: sufficiently small, rapidly decaying coefficient sequences give $r(u)\propto\exp(-\int_0^u h(v)\,dv)$ with the required ratio bounds.

Put $s_{\rm ref}(u)=8u/(1-u^2)$. The $j$th coordinate of the full score $S^p$ is the scalar function $s_j(x)=A^{-1}(s_{\rm ref}(x/A)+h_j(x/A))$. Known certificates are
\[
 \norm{\Sigma_X}_{\rm op}\leq\min(A^2,CA^2/11),\qquad
 \mathbb E s_j(x_j)^2\leq\frac C{A^2}(\sqrt{12}+R)^2.
\]
Reference moments, Minkowski's inequality and independence give these bounds and the joint score certificate $M=C(\sqrt{12}+R)^2/A^2$. The ratio bound makes the boundary trace zero, and the score bound makes $p_j'=-s_jp_j$ integrable. Thus the zero extension is absolutely continuous, the score is centered and Stein integration by parts is valid, including for distinct marginals with nonzero means.

\noindent\textbf{Known basis bounds.}
Writing $P_k=P_k^{(4,4)}$, the Rodrigues norm and endpoint value give
\[
 \int P_k^2\omega=\frac{630}{2k+9}\frac{(k+4)!^2}{k!(k+8)!},\qquad
 \norm{\phi_k}_\infty^2=\frac{(2k+9)\prod_{i=1}^8(k+i)}{9!}.
\]
The endpoint maximum follows from the beta-moment identity
\[
 \frac{P_k^{(a,a)}(u)}{P_k^{(a,a)}(1)}
 =\mathbb E(u+i\sqrt{1-u^2}Z_a)^k,
\]
where $Z_a$ has density proportional to $(1-z^2)^{a-1/2}$ for $a>0$. The complex quantity has modulus at most one. Using the derivative identity for Jacobi polynomials then gives $\norm{\phi_k'}_\infty\leq k(k+9)\norm{\phi_k}_\infty/10$. For a dictionary size $d_{\rm cal}\geq1$, set $H_{d_{\rm cal}}=\max_{k<d_{\rm cal}}\norm{\phi_k}_\infty$; thus
\[
 H_{d_{\rm cal}}^2\leq d_{\rm cal}^9,\qquad
 \max_{k<d_{\rm cal}}\norm{\phi_k'}_\infty
 \leq d_{\rm cal}^2H_{d_{\rm cal}}.
\]
The dictionary size is distinct from the matrix side length $d$ and the lower-bound dimension $s_{\max}$ in Section~\ref{sec:lower-bounds}.

\subsection{A Finite Calibration Algorithm and Its Proof}
\label{app:calibration}

For one marginal, write $U_0=X/A$, $\Phi=(\phi_0,\ldots,\phi_{d_{\rm cal}-1})^\top$, and define known moment integrands
\[
 G=\mathbb E[\Phi(U_0)\Phi(U_0)^\top],\qquad
 (\boldsymbol b)_a=\mathbb E[\phi_a'(U_0)-s_{\rm ref}(U_0)\phi_a(U_0)].
\]
Integration of $(\phi_a \omega r)'$ shows $\boldsymbol b=\mathbb E[h(U_0)\Phi(U_0)]$, while $cI\preceq G\preceq CI$. Therefore $\beta^*=G^{-1}\boldsymbol b$ is the $L^2(\omega r)$ projection coefficient of $h$. With $Q=\sqrt{C/c}\,R$,
\[
 \norm{\beta^*}_2\leq Q,\qquad
 \norm{h-\Phi^\top\beta^*}_{L^2(\omega r)}
 \leq\sqrt C R\,d_{\rm cal}^{-s}.
\]
Projection and the lower Gram eigenvalue give the first bound; comparison with the first $d_{\rm cal}$ reference coefficients gives the second.

Given $m_{\rm cal}$ iid scalar samples, choose an odd $k_{\rm blk}$ with $m_{\rm cal}\geq2k_{\rm blk}$ and split the first $k_{\rm blk}\lfloor m_{\rm cal}/k_{\rm blk}\rfloor$ samples into equal blocks. Estimate every upper-triangular entry of $G$ and every coordinate of $\boldsymbol b$ by the median of its block means, then reflect the Gram entries to obtain a symmetric $\widehat G$. If $\lambda_{\min}(\widehat G)\geq c/2$, set $\widetilde\beta=\widehat G^{-1}\widehat{\boldsymbol b}$; otherwise set it to zero. Project onto the Euclidean ball of radius $Q$ to obtain $\widehat\beta$, and return
\[
 \widehat s(x)=A^{-1}\bigl(s_{\rm ref}(x/A)+\Phi(x/A)^\top\widehat\beta\bigr).
\]
The eigenvalue check is necessary because entrywise medians need not preserve positive semidefiniteness. This score corresponds to a genuine density proportional to
\[
 \omega(u)\exp\left(-\sum_{a<d_{\rm cal}}\widehat\beta_a\int_0^u\phi_a(v)\,dv\right).
\]
Its normalizing constant is not needed for scoring.

\begin{mtheorem}[Finite-sample calibration]
\label{thm:calibration}
Fit $J_{\rm marg}$ marginals, with $m_{\rm cal}$ iid samples for each. For odd $k_{\rm blk}\geq 8\log(2J_{\rm marg}(d_{\rm cal}+1)^2/\delta_{\rm cal})$, assume
\[
 m_{\rm cal}\geq(32C/c^2)k_{\rm blk}d_{\rm cal}^{11}.
\]
Then, with probability at least $1-\delta_{\rm cal}$, simultaneously,
\begin{equation}
 \begin{aligned}
 \norm{\widehat s_j-s_j}_{L^2(p_j)}
 &\leq e_{m_{\rm cal},d_{\rm cal}}\\
 &:=\frac1A\left[
 \sqrt C R d_{\rm cal}^{-s}
 +\frac{2C}{c}\sqrt{\frac{8k_{\rm blk}}{m_{\rm cal}}}
 H_{d_{\rm cal}}\!
 \left\{\sqrt{2d_{\rm cal}(d_{\rm cal}^4+12)}+d_{\rm cal}Q\right\}\right]\\
 &\leq\frac1A\left[
 \sqrt C R d_{\rm cal}^{-s}
       +D d_{\rm cal}^7\sqrt{k_{\rm blk}/m_{\rm cal}}\right],
 \qquad D=\frac{2C}{c}\sqrt8(\sqrt{26}+Q).
 \end{aligned}
 \label{eq:calibration-error}
\end{equation}
On every calibration history, including failures,
\begin{equation}
 \mathbb E_{p_j}\widehat s_j(x_j)^2
 \leq M_{\rm fit}:=\frac C{A^2}(\sqrt{12}+Q)^2.
 \label{eq:fitted-score-moment}
\end{equation}
On the same calibration event of probability at least $1-\delta_{\rm cal}$, stacking the fitted scalar functions into $\widehat S$ gives, for a fresh product array, Equation~\eqref{eq:score-certificates} with $\varepsilon=e_{m_{\rm cal},d_{\rm cal}}$, $\widehat J=M_{\rm fit}$.
\end{mtheorem}

\noindent\textbf{Proof.}
A block mean of a variable with variance at most $v$ has variance at most $2vk_{\rm blk}/m_{\rm cal}$. Chebyshev bounds the probability of error greater than $\sqrt{8vk_{\rm blk}/m_{\rm cal}}$ by $1/4$. The median exceeds that error only if at least half the independent blocks are bad, with probability at most $e^{-k_{\rm blk}/8}$. Union over at most $J_{\rm marg}(d_{\rm cal}+1)^2$ moment estimates gives simultaneous accuracy.

For $a,b<d_{\rm cal}$, the needed integrand bounds are
\[
 \mathbb E_{\omega r}(\phi_a\phi_b)^2\leq CH_{d_{\rm cal}}^2,\qquad
 \mathbb E_{\omega r}(\phi_a'-s_{\rm ref}\phi_a)^2
 \leq2CH_{d_{\rm cal}}^2(d_{\rm cal}^4+12).
\]
The singular reference score is handled by $\mathbb E_{\omega}s_{\rm ref}^2=12$, not a nonexistent supremum bound. Put $t_m=8k_{\rm blk}/m_{\rm cal}$. On the simultaneous event,
\[
 \begin{aligned}
 \norm{\widehat G-G}_{\rm op}
 &\leq\sqrt{t_m}\,d_{\rm cal}\sqrt C H_{d_{\rm cal}},\\
 \norm{\widehat{\boldsymbol b}-\boldsymbol b}_2
 &\leq\sqrt{2t_mCd_{\rm cal}H_{d_{\rm cal}}^2(d_{\rm cal}^4+12)}.
 \end{aligned}
\]
The sample floor makes the first bound at most $c/2$, so the inverse branch is used, with $\norm{\widetilde\beta-\beta^*}_2 \leq(2/c)(\norm{\widehat{\boldsymbol b}-\boldsymbol b}_2+ Q\norm{\widehat G-G}_{\rm op})$. Projection onto the radius-$Q$ ball cannot increase the error. Multiplying by $\sqrt C/A$ and adding the approximation error proves the first bound in Equation~\eqref{eq:calibration-error}. The inequalities $H_{d_{\rm cal}}\leq d_{\rm cal}^{9/2}$ and $d_{\rm cal}^4+12\leq13d_{\rm cal}^4$ give the second. On every history, $\norm{\widehat\beta}_2\leq Q$ proves Equation~\eqref{eq:fitted-score-moment} by reference orthonormality and Minkowski.

Conditional on calibration, each fitted score and error coordinate is a fixed function of its own fresh independent coordinate. Their covariance matrices are therefore diagonal even when fits share training data; the scalar squared-error and fitted second-moment bounds control the diagonals, proving the joint certificates without a factor $N$.

\subsection{Allocation and the Resulting Rates}
\label{app:calibration-rates}

For distinct marginal laws use $J_{\rm marg}=N$, $m_{\rm cal}=Kb$. Only when all coordinates have an explicitly common marginal may they be pooled, giving $J_{\rm marg}=1$, $m_{\rm cal}=NKb$. Write $m_{\rm cal}=c_{\rm samp} b$, with $c_{\rm samp}=K$ or $NK$ respectively. Product structure alone does not justify the latter sample count.

Set $\delta_{\rm cal}=\delta_{\rm rew}=(T+1)^{-2}$, $d_{\max}=\max(1,\lceil c_{\rm samp} T\rceil)$, and let $k_T$ be the smallest positive odd integer at least $8\log(2J_{\rm marg}(d_{\max}+1)^2/\delta_{\rm cal})$. Define
\[
 \begin{aligned}
 \alpha&=\frac{s}{2s+14},&
 \rho_0&=\max\{2,(32C/c^2)^{(2s+14)/(2s+3)}\},\\
 b_0&=\lceil k_T\rho_0/c_{\rm samp}\rceil,&
 E_s&=2^s\sqrt C R+D.
 \end{aligned}
\]
For $b_0\leq b\leq T$, choose $d_{\rm cal}=\lfloor(c_{\rm samp} b/k_T)^{1/(2s+14)}\rfloor$ and $k_{\rm blk}=k_T$. Then
\begin{equation}
 \varepsilon_b\leq\frac{E_s}{A}
                \left(\frac{k_T}{c_{\rm samp} b}\right)^\alpha,\qquad
 \widehat J=M_{\rm fit}.
 \label{eq:allocated-score-error}
\end{equation}
Indeed, $z=c_{\rm samp} b/k_T\geq\rho_0$ implies $z\geq(32C/c^2)z^{11/(2s+14)} \geq(32C/c^2)d_{\rm cal}^{11}$ and $z\geq2$. Thus the sample floor and block requirements hold. Using $\lfloor u\rfloor\geq u/2$ for $u\geq1$ in Equation~\eqref{eq:calibration-error} proves the bound, hence $\varepsilon_b=\widetilde O(b^{-s/(2s+14)})$ for fixed class quantities.

For $T\geq8$, take
\[
 n=\lfloor T^{2/3}/4\rfloor,\qquad
 b=\max\{b_0,\lfloor T^{1/(1+\alpha)}/4\rfloor\}.
\]
Use this allocation if $b+2n\leq T$; otherwise select uniformly throughout. For $T<8$ also select uniformly. For fixed class and model constants, $b_0=O(\log(T+1))$, so the allocation eventually fits. Substituting Equations~\eqref{eq:allocated-score-error} and \eqref{eq:fitted-score-moment} into Equation~\eqref{eq:unknown-regret} gives exploration costs, a statistical term of order $Tn^{-1/2}$, and a bias term of order $Tb^{-\alpha}$, with the explicit prefactors in Equation~\eqref{eq:unknown-regret}. The failure cost is at most $2\Delta T/(T+1)^2$. This proves Equation~\eqref{eq:unknown-smooth-rate}. The fallback has the linear cap without invoking a calibration certificate whose sample floor fails.

The exponent $14$ comes from squaring the conservative $d_{\rm cal}^7$ estimation bound; it is not asserted optimal. For example, $s=1$ gives $T^{16/17}$. The bias has no additional rank multiplier, while the statistical terms retain their matrix or transformed-slice dimensions.

\noindent\textbf{Exactly specified finite-dimensional families.}
If every $h_j$ belongs exactly to the known span of the first $d_0$ basis functions, with fixed known $d_0$, take $d_{\rm cal}=d_0$. The approximation term is then zero, and the same algorithm gives $\varepsilon_b=O(\sqrt{\log(NT)/(c_{\rm samp} b)})$, subject to the fixed-$d_0$ sample floor and displayed constants. Taking $b,n$ of order $T^{2/3}$ proves $\widetilde O(T^{2/3})$ regret. Choosing to fit $d_0$ coefficients does not itself establish exact inclusion.

For example, $p_{\beta_j,\lambda_j}(x)\propto p_A(x)\exp(\beta_j(x/A)^2/2+\lambda_jx/A)$ has $h_j(u)=-\beta_ju-\lambda_j$ in the span of $\phi_0=1$, $\phi_1=\sqrt{11}u$. Known bounds $|\beta_j|\leq H_0$, $|\lambda_j|\leq H_1$ give $C=c^{-1}=\exp(H_0/2+2H_1)$ and $R^2=H_1^2+2^{2s}H_0^2/11$. Both coefficients are fitted by the specified two-dimensional system. If only the linear tilt is unknown and $\beta_j$ is known, its constant score contribution cancels in uniform pairs, so calibration can be omitted. This cancellation does not identify the greedy projection distribution.

\subsection{Conditional Return to Greedy Epochs}
\label{app:unknown-greedy}

Retain the nondecreasing-link assumption for the greedy regret analysis. Consider deterministic planned epoch lengths and deterministic scored-pair counts $n_i\geq1$. Freeze the direction throughout epoch $i$, including calibration, and construct $\widehat S_i$ before collecting those fresh action--reward pairs from the selected law $p_i$. Conditional on the post-calibration history, these observations are iid and reuse no calibration observations. The completed fit determines the next epoch's direction. Assume Stein regularity and, for $\Theta_*\ne0$, $\mu_i=\mathbb E_{p_i}f'(\inner X{\Theta_*})\geq\mu_{\min}>0$. This selected-law condition is additional to the base-law condition in Sections~\ref{app:explore-commit}--\ref{app:calibration-rates}.

Suppose deterministic certificates $(\varepsilon_i,\widehat J_i)$ satisfy Equation~\eqref{eq:score-certificates} under $p_i$, except on a calibration event of marginal probability at most $\beta_i$, and use conditional reward-fit failure probability $\delta_i$. Recall $e_{\rm fit}$ from Section~\ref{app:explore-commit}, including its score-bias term, and put $\Delta=2L_fAB$. If $h_0$ is the initial planned length and $h_{i+1}$ the actual following epoch length, including calibration and a partial final epoch, then
\[
 \begin{aligned}
 \mathbb E R_T\leq{}&\Delta\min\{T,h_0\}
 +\frac{2L_f\sqrt{K\norm{\Sigma_X}_{\rm op}}}{\mu_{\min}}
       \sum_i h_{i+1}e_{\rm fit}(n_i,\delta_i,\widehat J_i,\varepsilon_i)\\
 &+\Delta\sum_i h_{i+1}(\beta_i+\delta_i).
 \end{aligned}
\]
Here $\Sigma_X$ is the covariance of the base menu law. The sums concern completed fits used in a following epoch. The bound may also be capped by $\Delta T$; $\Theta_*=0$ has zero regret. To prove it, condition on successful calibration and apply Theorem~\ref{thm:estimated-score-error}. The fit event is measurable before the following epoch; Lemma~\ref{lem:transfer} controls regret on success using the preceding target $\mu_i\Theta_*$, and $\Delta$ controls failure.

For doubling epochs, a fixed fraction of each completed epoch allocated to scored pairs, uniformly bounded $\widehat J_i$, and $\beta_i,\delta_i\leq(T+1)^{-2}$, the statistical term is $\widetilde O(\sqrt T)$ for fixed structural and model quantities. A separately established certificate $\varepsilon_i=O(n_i^{-\alpha})$, $0<\alpha<1$, adds $O(T^{1-\alpha})$; $\alpha=1/2$ preserves the square-root horizon rate up to logarithms. These certificates must cover the adaptively chosen direction through a uniform argument or independent calibration after freezing it. The base-score construction of Sections~\ref{app:explore-commit}--\ref{app:calibration-rates} does not itself certify the projection correction $F_i'/F_i$ or the fitted selected-score covariance.

\section{Structured Single-Index Estimation Proofs}
\label{app:method-proofs}

We prove the paired-moment identity and Theorem~\ref{thm:structured-estimation}, first for matrices and then for transformed slices. The result applies conditionally to each fixed-policy T-ESTOR batch with $n=n_i$, $w=MG_K(d)$, and $\mu_*=\mu_i$; the fit from completed epoch $i$ is used in epoch $i+1$.

\subsection{Paired Moments}
\begin{mlemma}[Paired mean and second moment]
\label{lem:paired}
Let $p$ be a fixed known observation density, supported in $\maxnorm X<A$, whose zero extension belongs to $W^{1,1}(\mathbb R^N)$. Write $S^p=-\nabla_X\log p$ and assume its score has a finite joint second moment. Take $2n$ independent action--reward observations with this density, the budget of Section~\ref{sec:notation-structures}, and the link condition of Assumption~\ref{ass:boundedness-link} and the conditional noise bounds of Section~\ref{sec:bandit-model}, where $n\geq1$. The $n$ paired statistics in Equation~\eqref{eq:paired-observation} are iid. For $\Theta_*\ne0$, put $\mu_*=\mathbb E_{p}f'(\inner X{\Theta_*})\in[-L_f,L_f]$. For $\Theta_*=0$, set $\mu_*=0$ without evaluating $f'(0)$. With $V=2L_f^2A^2B^2+\sigma^2$,
\begin{equation*}
 \begin{aligned}
 \mathbb E Z_j&=\mu_*\Theta_*,\\
 \mathbb E[\vecop(Z_j)\vecop(Z_j)^\top]
 &\preceq V\,\mathbb E_{p}
       [\vecop(S^p)\vecop(S^p)^\top].
 \end{aligned}
 \label{eq:paired-moments}
\end{equation*}
For every fixed linear map $\mathcal L$ from arrays to matrices,
\[
 \begin{aligned}
 \mathbb E[\mathcal L(Z_j)\mathcal L(Z_j)^\top]
 &\preceq V\mathbb E[\mathcal L(S^p)
                            \mathcal L(S^p)^\top],\\
 \mathbb E[\mathcal L(Z_j)^\top\mathcal L(Z_j)]
 &\preceq V\mathbb E[\mathcal L(S^p)^\top
                            \mathcal L(S^p)].
 \end{aligned}
\]
These conclusions also hold conditionally on a history fixing the law and score, including each frozen T-ESTOR epoch with $p=p_i$.
\end{mlemma}

For the appendix proofs, let $J_{\rm obs}>0$ bound the joint score moment by $J_{\rm obs}I_N$. The second-moment bound $w>0$ from Section~\ref{sec:paired-moments} bounds both left and right score moments by $w$ times their identity matrices. The latter applies to the matrix score or, for the supplied transform, each transformed slice of the full score. These bounds need not be equal or minimal. Appendix~\ref{app:score-moments} verifies them for $p=p_i$ with $w=MG_K(d)$ and $J_{\rm obs}=\overline J=MG_K(1)$, including the uniform fallback.

\noindent\textbf{Proof of Lemma~\ref{lem:paired}.}
The weak integration-by-parts argument of Appendix~\ref{sec:candidate-law} applies verbatim to $p$: $\mathbb E S^p=0$ and $\mathbb E[y S^p(X)]=\mu_*\Theta_*$. Its bounded support gives the cutoff, and its finite score second moment ensures integrability. For independent observations $(X_1,y_1),(X_2,y_2)$, put $S_j=S^p(X_j)$. Expanding $(y_1-y_2)(S_1-S_2)/2$ gives $\mathbb E Z=\mathbb E[yS]-\mathbb E y\,\mathbb E S=\mu_*\Theta_*$. Conditional independence of the noise draws gives
\[
 \mathbb E[(y_1-y_2)^2\mid X_1,X_2]
 =(f(\inner{X_1}{\Theta_*})-f(\inner{X_2}{\Theta_*}))^2
   +\mathbb E[\eta_1^2\mid X_1]
   +\mathbb E[\eta_2^2\mid X_2]\leq2V.
\]
For any unit array $D$ this implies $\mathbb E\inner ZD^2\leq(V/2)\mathbb E\inner{S_1-S_2}D^2 =V\mathbb E\inner SD^2$, proving the joint PSD bound. For a matrix-valued linear map, apply the same conditional inequality to $(\mathcal L(S_1)-\mathcal L(S_2)) (\mathcal L(S_1)-\mathcal L(S_2))^\top$; score centering removes its cross terms. The right second moment follows by transposition. These bounds control raw second moments of $Z$, without subtracting its population target. Pairing cancels the unrestricted reward level $f(0)$. The nonzero-parameter index has a density under any such $p$; the zero-parameter case follows directly by centering. The moment statement requires neither monotonicity nor a nonzero Stein signal: $\mu_*$ may be negative or zero.

\subsection{Matrix Estimation Proofs}
\label{app:matrix-proof}

For the matrix case of Theorem~\ref{thm:structured-estimation}, set $q=1$, $U=[1]$, and $r_U=r$ in Equation~\eqref{eq:slice-fit}. Its tuning specializes to
\begin{equation}
 \nu=\sqrt{\frac{2\log(2(d_1+d_2)/\delta)}{nVw}},\qquad
 \lambda=2\sqrt{\frac{2Vw\log(2(d_1+d_2)/\delta)}{n}}.
 \label{eq:matrix-tuning}
\end{equation}
We prove, with probability at least $1-\delta$,
\begin{equation}
 \norm{\widehat Z_\nu-\mu_*\Theta_*}_{\mathrm{op}}\leq\lambda/2,
 \qquad
 \frobnorm{\widehat\Theta-\mu_*\Theta_*}^2
 \leq\frac{16rVw\log(2(d_1+d_2)/\delta)}{n}.
 \label{eq:matrix-error}
\end{equation}
The bound may use the true rank, but does not justify substituting a fitted rank into a confidence bound or assert exact rank recovery.

\noindent\textbf{Dilation and concentration.}
The rectangular influence from Section~\ref{sec:matrix-estimation} is equivalently expressed by the symmetric dilation
\begin{equation}
 \mathcal H(H)=\begin{pmatrix}0&H\\H^\top&0\end{pmatrix},\qquad
 \widetilde\psi_\nu(H)=\nu^{-1}[\psi(\nu\mathcal H(H))]_{12}.
 \label{eq:spectral-influence}
\end{equation}
The subscript $12$ selects the upper-right block, and the symmetric matrix function applies $\psi$ to eigenvalues. Thus the robust average is $\widehat Z_\nu=n^{-1}\sum_{j=1}^n\widetilde\psi_\nu(Z_j)$; the factor $\nu^{-1}$ is already part of the scaled influence map. The scalar influence function defined in Section~\ref{sec:paired-moments} is odd and satisfies
\begin{equation}
 -\log(1-x+x^2/2)\leq\psi(x)\leq\log(1+x+x^2/2).
 \label{eq:influence-sandwich}
\end{equation}
For $x\geq0$ the upper inequality is equality, and the lower follows from $(1+x+x^2/2)(1-x+x^2/2)=1+x^4/4\geq1$; oddness gives $x<0$. In an SVD basis the dilation has eigenvalues $\pm s_\ell$, which proves its off-diagonal influence formula.

For $n$ independent self-adjoint $d_0\times d_0$ matrices $A_j$ with common mean and $\norm{\mathbb E A_j^2}_{\rm op}\leq v_0$, the matrix influence inequality of \citet{minsker2018sub} gives
\begin{equation}
 \mathbb P\left\{\norm{\frac1{n\nu}\sum_{j=1}^{n}\psi(\nu A_j)
                         -\mathbb E A_1}_{\rm op}>\varepsilon\right\}
 \leq2d_0\exp(-n\nu\varepsilon+n\nu^2v_0/2).
 \label{eq:matrix-influence-tail}
\end{equation}
Indeed, functional calculus applies Equation~\eqref{eq:influence-sandwich} to each matrix. The trace exponential bound for independent self-adjoint sums controls the upper tail by $d_0\exp(-n\nu\varepsilon+n\nu^2v_0/2)$; applying it to the negative sum gives the other tail.

Take $A_j=\mathcal H(Z_j)$, where $\mathcal H$ denotes the symmetric dilation in Equation~\eqref{eq:spectral-influence}. Lemma~\ref{lem:paired} gives $v_0=Vw$ and $d_0=d_1+d_2$. With $\varepsilon=\sqrt{2v_0\log(2(d_1+d_2)/\delta)/n}$ and the stated $\nu$, Equation~\eqref{eq:matrix-influence-tail} gives the operator bound in Equation~\eqref{eq:matrix-error}.

For completeness, let $\widehat Z_\nu$ have operator error at most $\varepsilon$ from a target $\mu_*\Theta_*$ of rank at most $r$, and soft threshold it at $\varepsilon$. The $(r+1)$st singular value of $\widehat Z_\nu$ is at most $\varepsilon$, so $\operatorname{rank}(\widehat\Theta)\leq r$ on this event. Soft thresholding changes a matrix by operator norm at most $\varepsilon$, hence $\norm{\widehat\Theta-\mu_*\Theta_*}_{\rm op}\leq2\varepsilon$. The difference has rank at most $2r$, giving $\frobnorm{\widehat\Theta-\mu_*\Theta_*}^2\leq8r\varepsilon^2$. Substitution proves the constant $16$ in Equation~\eqref{eq:matrix-error}. This proof requires neither a lower bound on the nonzero singular values nor knowledge of $r$ in the tuning.

\noindent\textbf{Arithmetic cost.}
Conditional on score access, the robust moment uses one thin SVD per paired $d_1\times d_2$ matrix, followed by one final SVD for soft thresholding. The dense arithmetic cost is $O((n+1)d_1d_2\min(d_1,d_2))$, with $O(d_1d_2)$ accumulation storage apart from the SVD workspace. This count does not include computing the score or claim exact finite-bit evaluation of the logarithmic influence.

\subsection{Transformed-Tensor Estimation Proofs}
\label{app:slice-proof}

For clarity, define the per-pair transformed influence map
\[
 \Psi_\nu^U(Z)=\mathcal T_U^{-1}
 \left(\left[\widetilde\psi_\nu
       \bigl((\mathcal T_U Z)^{(k)}\bigr)\right]_{k=1}^q\right).
\]
The moment in Section~\ref{sec:slice-estimation} is equivalently $\widehat Z_\nu=n^{-1}\sum_{j=1}^{n}\Psi_\nu^U(Z_j)$; linearity of $\mathcal T_U$ gives $(\mathcal T_U\widehat Z_\nu)^{(k)}=\widehat Z_{\nu,k}$ for every $k$.

For the transformed-slice case of Theorem~\ref{thm:structured-estimation}, Equation~\eqref{eq:slice-fit} uses
\begin{equation}
 \nu=\sqrt{\frac{2\log(2q(d_1+d_2)/\delta)}{nVw}},\qquad
 \lambda=2\sqrt{\frac{2Vw\log(2q(d_1+d_2)/\delta)}{n}}.
 \label{eq:slice-tuning}
\end{equation}
We prove the simultaneous bounds, with probability at least $1-\delta$,
\begin{equation}
 \begin{aligned}
 \max_{k\in[q]}\norm{(\mathcal T_U(\widehat Z_\nu-\mu_*\Theta_*))^{(k)}}_{\mathrm{op}}
 &\leq\lambda/2,\\
 \frobnorm{\widehat\Theta-\mu_*\Theta_*}^2
 &\leq\frac{16Vr_Uw\log(2q(d_1+d_2)/\delta)}{n}.
 \end{aligned}
 \label{eq:tensor-error}
\end{equation}
The dual of the sum nuclear norm is the maximum slice operator norm, because $\mathcal T_U$ preserves inner products. Its Frobenius isometry also gives the square-completion identity
\[
 \frobnorm Z^2-2\inner Z{\widehat Z_\nu}
 =\sum_{k=1}^q
   \frobnorm{(\mathcal T_U Z)^{(k)}-\widehat Z_{\nu,k}}^2
   -\frobnorm{\widehat Z_\nu}^2.
\]
The last term is independent of $Z$, so adding the sum nuclear penalty separates the objective into slice problems with quadratic coefficient one. Each problem has solution $\operatorname{SVT}_{\lambda/2}(\widehat Z_{\nu,k})$; the inverse transform then gives the estimator in Section~\ref{sec:slice-estimation}. Define the block-diagonal dilation
\[
 \mathcal H_U(Z)=
 \operatorname{diag}_{k=1}^q\mathcal H((\mathcal T_U Z)^{(k)}).
\]
Its dimension is $q(d_1+d_2)$ and the operator norm of its raw second moment is at most $Vw$. Applying Equation~\eqref{eq:matrix-influence-tail} to this one block matrix gives the simultaneous dual bound in Equation~\eqref{eq:tensor-error}; independence between its blocks is not required. Apply the soft-thresholding argument of Appendix~\ref{app:matrix-proof} in slice $k$ with rank at most $\rho_k$, then sum squared errors. Frobenius isometry proves Equation~\eqref{eq:tensor-error}.

One can accumulate influenced slices in transformed coordinates and invert only after fitting. Dense multiplication by $U$ and $q$ thin SVDs cost $O(d_1d_2q^2+qd_1d_2\min(d_1,d_2))$ per pair; the final fit uses $q$ more SVDs and one inverse transform. This is an arithmetic-operation statement conditional on scores, not an exact finite-bit claim for transcendental function evaluation. For $q=1$, $U=[1]$ and $r_U=r$, both the estimator and tuning reduce to the matrix case.

\noindent\textbf{Epoch buffering and streaming.}
Algorithm~\ref{alg:T-ESTOR} literally retains an epoch buffer, using $O(n_iN)$ observation storage. Its fixed epoch length and tuning also permit an equivalent streaming implementation: retain one pending observation, form each disjoint pair when the next arrives, and add its influenced statistic to the running matrix or slice sums. At the epoch end, divide by the prescribed pair count $n_i$ and apply the same threshold. Observation and accumulator storage is then $O(N)$, apart from score-oracle and linear-algebra workspace. This equivalence concerns the theoretical policy with tuning fixed before sampling; it does not automatically apply to the adaptive experimental procedure, which recomputes tuning from stored data at checkpoints.

\section{T-ESTOR with Structured Updates}
\label{app:stein-pair}

We verify the conditional sampling model, selected-score identities, and moment bounds used in Section~\ref{sec:epoch-framework}.

\subsection{Conditional Probability Model}

Include the learner's private random seed in the history. Conditional on each epoch-start history, the direction and score rule are fixed. Independent menus and fresh noise from the fixed conditional kernel used in Section~\ref{sec:epoch-framework} then give iid selected action--reward observations. Concentration holds after every such history. Within-epoch reward-dependent policy changes would invalidate this reduction and are not part of Algorithm~\ref{alg:T-ESTOR}.

\subsection{Projection Laws and Selected Scores}
Fix any epoch-start history and any unit-Frobenius direction $v_i$; the following argument is uniform over all such directions. Choose an orthonormal complement $Q$ to $\vecop(v_i)$ and write $x=u\vecop(v_i)+Qz$. If $F_i$ is the CDF of $\inner{v_i}X$ under $p$, orthogonal change of variables gives its density and weak derivative:
\[
 \begin{aligned}
 F_i'(u)&=\int p(u\vecop(v_i)+Qz)\,dz,\\
 F_i''(u)&=\int\vecop(v_i)^\top\nabla p(u\vecop(v_i)+Qz)\,dz.
 \end{aligned}
\]
Fubini and the zero-extension regularity $p\in W^{1,1}$ in Assumption~\ref{ass:candidate-distribution} imply $F_i'\in W^{1,1}(\mathbb R)$; use its continuous, absolutely continuous representative. The weak score identity implies, for almost every $u$ with $F_i'(u)>0$,
\[
 -\frac{F_i''(u)}{F_i'(u)}
 =\mathbb E_p[\inner{v_i}{S^p(X)}\mid\inner{v_i}X=u].
 \label{eq:projection-score-identity}
\]
Jensen's inequality and the full joint score certificate give $\int_{\{F_i'>0\}}(F_i'')^2/F_i'\leq M$, and Cauchy--Schwarz gives $\int|F_i''|\leq\sqrt M$. The density is supported on $[-A\entone{v_i},A\entone{v_i}]$, so it vanishes outside this interval. An absolutely continuous nonnegative density must rise from and return to zero; hence
\begin{equation}
 \norm{F_i'}_\infty\leq\frac12\int|F_i''|\leq\frac{\sqrt M}{2}.
 \label{eq:projection-peak}
\end{equation}
No strict positivity throughout that interval is required.

For a nonzero vectorized direction $v\in\mathbb R^N$, let $F_v$ be the CDF of $\langle v,\vecop(X)\rangle$. Rescaling Equation~\eqref{eq:projection-peak} gives
\[
 \norm{F_v'}_\infty\leq\frac{\sqrt M}{2\norm v_2},\qquad
 \mathbb E_p\!\left[F_v'(\langle v,\vecop(X)\rangle)^2\right]
 \leq\frac{M}{4\norm v_2^2}\leq\frac M4
 \quad\text{if }\norm v_\infty=1.
\]
Thus our joint-score assumption implies the projection-density moment condition in Assumption~3.3 of \citet{kang2026single}. For dependent laws, the joint certificate is stronger than coordinatewise score-moment control.

The selected density in Equation~\eqref{eq:selected-score} is $p_i(X)=Kp(X)F_i(\inner{v_i}X)^{K-1}$. Since $F_i'$ is bounded, the CDF factor is a bounded Lipschitz multiplier, so the zero extension of $p_i$ remains in $W^{1,1}$. Its score is
\[
 S^{p_i}(X)=S^p(X)-(K-1)\alpha_i(\inner{v_i}X)v_i,\qquad
 \alpha_i(u)=\frac{F_i'(u)}{F_i(u)}\quad\text{on }\{F_i>0\}.
\]
Set $\alpha_i=0$ on $\{F_i=0\}$, a set of zero probability under the selected projection law. Thus the formula is defined $p_i$-almost surely, including when the support has gaps. Applying the weak integration-by-parts argument of Appendix~\ref{sec:candidate-law} gives
\[
 \mathbb E_{p_i}S^{p_i}=0,\qquad
 \mathbb E_{p_i}[yS^{p_i}(X)]=\mu_i\Theta_*,\qquad
 \mu_i=\mathbb E_{p_i}[f'(\inner X{\Theta_*})]
\]
for $\Theta_*\ne0$. The selected law is absolutely continuous, and its score is square integrable by the bounds below, so the same index and integrability arguments apply. For the uniform fallback $v_i=0$, set $p_i=p$ and $S^{p_i}=S^p$; no projection density is needed. For a zero parameter the moment is zero directly without evaluating $f'(0)$.

\subsection{Uniform Selected-Score Moments}
\label{app:score-moments}

For a unit direction, the selected projection has density $KF_i'(u)F_i(u)^{K-1}$. For $K\geq3$, Equation~\eqref{eq:projection-peak} and the probability integral transform give
\[
 \begin{aligned}
 \mathbb E_{p_i}\alpha_i(\inner{v_i}X)^2
 &=K\int_{\{F_i>0\}}F_i'(u)^3F_i(u)^{K-3}\,du\\
 &\leq\frac{KM}{4}\int F_i'(u)F_i(u)^{K-3}\,du
 =\frac{KM}{4(K-2)}.
 \end{aligned}
\]
This substitution is valid for any continuous CDF, including one with flat segments. At $K=3$, the power $F_i^{K-3}$ is simply one; this is why we require $K\geq3$. Since $p_i\leq Kp$, the base-score contribution in any unit array direction is at most $KM$. Applying $(a-b)^2\leq2a^2+2b^2$ to the score formula gives
\[
 \mathbb E_{p_i}[\vecop(S^{p_i})\vecop(S^{p_i})^\top]
 \preceq\overline J I_N,\qquad
 \overline J=MG_K(1),\qquad
 G_K(u)=2Ku+\frac{K(K-1)^2}{2(K-2)}.
\]
For a $d_1\times d_2$ base score and a unit $a\in\mathbb R^{d_1}$,
\[
 a^\top\mathbb E_p[S^p(S^p)^\top]a
 =\sum_{b=1}^{d_2}\mathbb E_p\inner{S^p}{ae_b^\top}^2\leq d_2M.
\]
Thus the left and right base moments are bounded by $d_2MI_{d_1}$ and $d_1MI_{d_2}$, respectively. Since $v_iv_i^\top\preceq I_{d_1}$ and $v_i^\top v_i\preceq I_{d_2}$, the PSD inequality $(C-D)(C-D)^\top\preceq2CC^\top+2DD^\top$ gives selected left and right moment bounds $wI_{d_1}$ and $wI_{d_2}$, where $d=\max(d_1,d_2)$ and $w=MG_K(d)$. These are certificates, not isotropic moment equalities.

For a transformed slice, put $B_k=(\mathcal T_US^p)^{(k)}$. The joint base moment bound $MI_N$ is preserved by the full orthogonal transform. Extracting a slice and applying the row and column calculation gives $\mathbb E B_kB_k^\top\preceq d_2MI_{d_1}$ and $\mathbb E B_k^\top B_k\preceq d_1MI_{d_2}$. Neither transformed entries nor slices need be independent. Moreover, $\frobnorm{(\mathcal T_Uv_i)^{(k)}}\leq1$. The same PSD argument slice by slice proves the same $w$ bound, without a factor of $q$. Here $B_k$ is a slice of the transformed full joint score, not a marginal slice score. Orthogonality also preserves the selected joint vector moment bound. For the uniform fallback, the smaller base moments are covered by all these certificates.

\noindent\textbf{Why the selected-signal condition is separate.}
Strict monotonicity alone gives no uniform signal floor over the parameter class: for $f(u)=u^3$ the base multiplier is $3\mathbb E_p\inner X{\Theta_*}^2$, equal to $(3A^2/11)\frobnorm{\Theta_*}^2$ under the beta design, and tends to zero along nonzero parameters.

Menu size can also affect the signal. For $\Theta_*\ne0$ and $f(u)=L_f\max\{u,0\}$, selection in direction $v_i=-\Theta_*/\frobnorm{\Theta_*}$ chooses the minimum true index among the $K$ candidates. Its derivative is nonzero only when all $K$ indices are positive, so
\[
 \mu_i=L_f\bigl[\Pr_p\{\inner X{\Theta_*}>0\}\bigr]^K.
\]
The index has a density, so the derivative's value at zero is immaterial under our a.e. convention. If the admissible candidate law is centrally symmetric about zero, this signal is $L_f2^{-K}$. A uniform floor in Condition~\ref{cond:selected-signal} can therefore be at most $L_f2^{-K}$ in this example, making its possible dependence on $K$ explicit.

\section{Regret Proofs}
\label{app:regret-proof}

Under the assumptions and tuning of Theorem~\ref{thm:regret}, the matrix bound with explicit constants is
\begin{equation}
 \begin{aligned}
 \mathbb E R_T\leq100\min\biggl\{
 &L_fB\sqrt{K\norm{\Sigma_X}_{\rm op}}T,\\
 &\frac{L_f}{\mu_{\min}}\sqrt{VK\norm{\Sigma_X}_{\rm op}MrT G_K(d)
                  \log(2(d_1+d_2)T)}\biggr\}.
 \end{aligned}
 \label{eq:structured-regret-explicit}
\end{equation}
For transformed slices, replace $r\log(2(d_1+d_2)T)$ by $r_U\log(2q(d_1+d_2)T)$. The proof below also gives the smaller constant $72$. The relation $G_K(d)\asymp K(d+K)$ for $K\geq3$ gives the compact main-text bound with a universal constant.

\begin{mlemma}[Fresh-menu transfer]
\label{lem:transfer}
Under the distribution and link assumptions of Section~\ref{sec:bandit-assumptions} and the noise model of Section~\ref{sec:bandit-model}, with $f$ nondecreasing as in Section~\ref{sec:epoch-framework}, $\Theta_*\in\mathcal B(B)$, and Condition~\ref{cond:selected-signal}, let $\widehat\Theta$ and $\mu\geq\mu_{\min}$ be measurable before the current menu. Select greedily by $\widehat\Theta$, with uniform selection if it is zero. Then
\[
 \mathbb E[r_t\mid\mathcal F_{t-1}]
 \leq\frac{2L_f}{\mu_{\min}}\sqrt{K\norm{\Sigma_X}_{\rm op}}\,
                  \frobnorm{\widehat\Theta-\mu\Theta_*}.
\]
For every policy, including on histories where estimation failed, $\mathbb E[r_t\mid\mathcal F_{t-1}]\leq \overline\Delta:=2L_f B\sqrt{K\norm{\Sigma_X}_{\rm op}}$. This is a conditional mean bound; the deterministic reward-gap bound is $\Delta=2L_f A B$.
\end{mlemma}

If $\Theta_*=0$, all candidates have the same mean reward and regret is zero. Below assume $\Theta_*\ne0$.

\noindent\textbf{Fresh-menu transfer.}
Let $\overline X=\mathbb E_{p}X$ and $X^c_{t,a}=X_{t,a}-\overline X$. Subtracting the same array from every candidate preserves both fitted rankings and true index differences. Let $a_*$ maximize the true index and put $X^c_t=X_t-\overline X$ for the chosen action. For $E=\widehat\Theta-\mu\Theta_*$, greedy optimality gives
\[
 \mu\inner{X_{t,a_*}-X_t}{\Theta_*}
 \leq\inner{X^c_t-X^c_{t,a_*}}E
 \leq2\max_{a\leq K}|\inner{X^c_{t,a}}E|.
\]
If $\widehat\Theta=0$, every index is a fitted maximizer, so the same statement holds for uniform selection. Nondecreasingness and Lipschitz continuity convert this into a reward gap. Conditional on the past, $E$ is fixed and
\[
 \mathbb E\max_a|\inner{X^c_{t,a}}E|
 \leq\sqrt{\sum_{a=1}^K\mathbb E\inner{X^c_{t,a}}E^2}
 \leq\sqrt{K\norm{\Sigma_X}_{\rm op}}\frobnorm E.
\]
Since $\mu\geq\mu_{\min}$, this proves Lemma~\ref{lem:transfer}. For an arbitrary policy, the reward gap is at most $2L_f\max_a|\inner{X^c_{t,a}}{\Theta_*}|$; the same calculation with $\frobnorm{\Theta_*}\leq B$ gives the conditional mean cap $\overline\Delta=2L_f B\sqrt{K\norm{\Sigma_X}_{\rm op}}$. Entrywise budgets also give the deterministic cap $\Delta=2L_f A B$. Centering serves only this proof. The algorithm uses the original actions and need not know or subtract $\overline X$; the model and budgets remain in those coordinates.

\noindent\textbf{Consequences of the Stein assumptions.}
Write $\bar x=\mathbb E_{p}x$. Applying weak integration by parts to $x_i-\bar x_i$, with a cutoff equal to one on the support, gives
\[
 \mathbb E_{p}[(x_i-\bar x_i)(\vecop(S^p(X)))_i]=1.
\]
Cauchy--Schwarz, $\operatorname{Var}(x_i)\leq\mathbb E x_i^2\leq A^2$, and $\mathbb E (\vecop(S^p(X)))_i^2\leq M$ imply $M A^2\geq1$. For every unit $u\in\mathbb R^N$, the same calculation gives $\mathbb E[(u^\top(x-\bar x))(u^\top \vecop(S^p(X)))]=1$ and hence $M\norm{\Sigma_X}_{\rm op}\geq1$. These inequalities follow from the general regularity and moment certificates; they impose no new assumptions.

\noindent\textbf{Epoch summation and numerical constants.}
For epoch confidence $\delta_E=(T+1)^{-2}$, put $\ell_{\mathrm M,E}=\log(2(d_1+d_2)(T+1)^2)$ and $\ell_{U,E}=\log(2q(d_1+d_2)(T+1)^2)$. Valid error coefficients are
\[
 \begin{aligned}
 a_{\mathrm M,E}&=4\sqrt{VrMG_K(d)\ell_{\mathrm M,E}},&
 a_{U,E}&=4\sqrt{Vr_UMG_K(d)\ell_{U,E}}.
 \end{aligned}
\]
Each conditional estimation failure probability is at most $\delta_E$ after every start history, and hence also unconditionally. In epoch $i\geq1$, the preceding estimate $\widehat\Theta_{i-1}$ and its target $\mu_{i-1}\Theta_*$ are measurable before the current menu. On success, the per-round conditional mean regret is at most $(2L_f/\mu_{\min})\sqrt{K\norm{\Sigma_X}_{\rm op}}a/\sqrt{n_{i-1}}$; on failure use $\overline\Delta$. This avoids conditioning on the accuracy of a fit that uses the current or a future menu. If $\mathcal B_i$ is the failure event for the fit from epoch $i$, it is measurable before epoch $i+1$. Its contribution over the actual following length $h_{i+1}$ is at most $\overline\Delta h_{i+1}\mathbb P(\mathcal B_i)$. Summing these terms gives at most $\overline\Delta T\delta_E$, with no additional number-of-epochs factor.

An epoch $i\geq1$ has planned length $2^{i+1}$ and uses $2^{i-1}$ pairs from its predecessor. Its final time is at most $2^{i+2}-2<8\cdot2^{i-1}$. Thus $1/\sqrt{n_{\rm prev}(t)}\leq\sqrt{8/t}$ and summation gives
\begin{equation}
 \mathbb E R_T\leq\overline\Delta\min(T,2)
 +\frac{8\sqrt2 L_f}{\mu_{\min}}\sqrt{K\norm{\Sigma_X}_{\rm op}}\,a\sqrt T
 +\overline\Delta T\delta_E.
 \label{eq:epoch-master}
\end{equation}
This also holds when the final epoch is incomplete. The same policy separately satisfies $\mathbb E R_T\leq\overline\Delta T$.

For the matrix or slice method, respectively, set $R=r$ or $R=r_U$; in either case $R\geq1$. Let $\ell_E$ denote the corresponding $\ell_{\mathrm M,E}$ or $\ell_{U,E}$, and define
\[
 W_T=(L_f/\mu_{\min})\sqrt{VRK\norm{\Sigma_X}_{\rm op}MG_K(d)\ell_E T}.
\]
The leading term in Equation~\eqref{eq:epoch-master} is $32\sqrt2 W_T$. Using $V\geq2L_f^2A^2B^2$, $\mu_{\min}\leq L_f$, $M A^2\geq1$, $R\geq1$, $G_K(d)\geq8$ and $\ell_E\geq1$ gives
\[
 \frac{\overline\Delta}{W_T}
 =\frac{2\mu_{\min}B}{\sqrt{VRMG_K(d)\ell_E T}}
 \leq\frac{\sqrt2}{\sqrt{M A^2RG_K(d)\ell_E T}}\leq1.
\]
Initialization and failure add at most $(9/4)W_T$, since $T/(T+1)^2\leq1/4$. The total is at most $48W_T$. For dilation dimension $d_0\geq2$, $\log(2d_0(T+1)^2)\leq2\log(2d_0T)$. Combining this with the separate linear cap gives the constant $72$ in the matrix and slice cases of Equation~\eqref{eq:structured-regret-explicit}.

\paragraph{Remark (sub-Gaussian menus).}
If the base law also satisfies the projection certificate \eqref{eq:nonmonotone-joint-concentration}, a conditional log-sum-exp bound gives
\[
 \mathbb E\!\left[\max_{a\leq K}|\inner{X_{t,a}-m}E|
       \mid\mathcal F_{t-1}\right]
 \leq A\kappa_X\sqrt{2\log(2K)}\,\frobnorm E
\]
for every history-measurable $E$. This gives an alternative fresh-menu transfer coefficient to $\sqrt{K\norm{\Sigma_X}_{\rm op}}$ in Lemma~\ref{lem:transfer}. Substituting it into the epoch argument gives the corresponding regret refinement, with the selected-score factor $G_K(d)$ unchanged and initialization and failures controlled as above. The assumption concerns fresh candidate projections, not reward noise or selected-law projections; the separate deterministic regret cap remains valid.

\subsection{Detailed Structural Comparisons}
\label{app:structural-comparisons}

The bound has two menu-size factors: $G_K(d)\asymp K(d+K)$ from the selected-score moments and $\sqrt K$ from Lemma~\ref{lem:transfer}, so estimation and transfer contribute $\sqrt{G_K(d)}\sqrt K\asymp K\sqrt{d+K}$. Against the lower bound's $1/K$ prefactor this leaves a gap of order $K^2\sqrt{1+K/d}$, which we do not claim is necessary. The design constants leave a further gap of $(M\norm{\Sigma_X}_{\rm op})^{3/2}$ relative to Theorem~\ref{thm:lower-general}, which Remark~\ref{rem:lower-whitened} reduces to $\sqrt{M\norm{\Sigma_X}_{\rm op}}$ for the designs of Appendix~\ref{app:density-examples}.

\noindent\textbf{Product-beta specialization.}
For the density in Appendix~\ref{app:density-examples}, Equation~\eqref{eq:beta-density}, take $\norm{\Sigma_X}_{\rm op}=c_X=A^2/11$ and $M=M_0=12/A^2$ in Theorem~\ref{thm:regret}. The general guarantee then recovers the product-design bounds with $\norm{\Sigma_X}_{\rm op}M=12/11$, under the same exact-score and selected-signal assumptions.

\noindent\textbf{Vectorization and ESTOR.}
Representing the same matrix or tensor as an $N\times1$ rank-one matrix gives the vectorized specialization
\[
 \widetilde O\!\left(K\sqrt{(N+K)T}\right)
\]
of Theorem~\ref{thm:regret}, with signal and model/design constants held fixed. ESTOR's stated Theorem~3.5 gives $\widetilde O(NK^{3/2}\sqrt T)$ \citep{kang2026single}. This compares published upper bounds on a shared admissible bounded product-design subclass, with comparable norm budgets and signal floors and uniformly controlled model/design constants. It is neither an empirical separation nor an algorithm-independent lower bound, and it makes no dominance claim over all ESTOR designs.

\noindent\textbf{Comparison with unfolding.}
Any matricization preserves the original entrywise budgets and the joint density assumptions (it permutes coordinates), so the matrix specialization supplies a valid baseline using the rank and larger side length of the chosen unfolding.

Here let $d\geq3$ be the side length of a cubic tensor and let $a_1,\ldots,a_d$ be the columns of $H=\mathbf1\mathbf1^\top-2I_d$. Its eigenvalues are $d-2$ on the span of $\mathbf1$ and $-2$ on its orthogonal complement, so $H$ is an invertible sign matrix. Set
\[
 \Theta_*=\frac B{d^3}\sum_{k=1}^d a_k\otimes a_k\otimes e_k,
 \qquad U=I_d.
\]
Every entry has magnitude $B/d^3$, so the tensor is dense and $\entone{\Theta_*}=B$. Each frontal slice is the rank-one matrix $(B/d^3)a_ka_k^\top$, giving $r_U=d$ and slice structural scale $d^2$. In the first unfolding, fixing the second index $j$ gives the invertible $d\times d$ submatrix $(B/d^3)H\operatorname{diag}(H_{j1},\ldots,H_{jd})$. The second unfolding has the same property by symmetry. A relation with coefficients $c_k$ among the rows of the third unfolding would give $H\operatorname{diag}(c)H^\top=0$, hence $c=0$. All three unfoldings therefore have rank $d$ and larger side $d^2$, giving matrix structural scale $d^3$.

Use the same candidate law and parameter, and a fixed increasing linear link $f(z)=\beta z$, $0<\beta\leq L_f$, for both representations. This gives the shared selected-design signal floor $\mu_{\min}=\beta$. For the explicit $G_K$ bounds, including confidence factors, the ratio of the unfolding upper root term to the slice upper root term is
\[
 \left(\frac{G_K(d^2)\log(2(d+d^2)T)}
 {G_K(d)\log(4d^2T)}\right)^{1/2}.
\]
Before these logarithms, the ratio is $\sqrt{G_K(d^2)/G_K(d)}\asymp\sqrt{(d^2+K)/(d+K)}$. The compact bound in Theorem~\ref{thm:regret} has structural-factor ratio $\sqrt{(d^2+K)/(d+K)}$. Thus the slice bound is smaller by order $\sqrt d$ when $K=O(d)$, but this gain does not hold uniformly as $K$ grows. This compares upper bounds, not all unfolding-based algorithms or empirical performance.

\section{Lower-Bound Proofs}
\label{app:lower-proof}

Throughout this appendix $p$ is any candidate law satisfying Assumption~\ref{ass:candidate-distribution}, including the zero-extension regularity explained in Appendix~\ref{sec:candidate-law}, with score certificate $M$, covariance $\Sigma_X=\operatorname{Cov}_p(\vecop(X))$ and mean $m=\mathbb E_p\vecop(X)$. Candidate entries may be dependent and noncentered. The product-beta law~\eqref{eq:beta-density} appears only in Corollary~\ref{cor:lower-beta}, where its independence and exact moments sharpen the constants. Arrays are identified with their vectorizations, so inner products with $m$ and Frobenius norms are Euclidean.

We use the original-coordinate classes from Section~\ref{sec:notation-structures}, and the same realized-menu comparator as the upper bounds. Restrict the reward model to a disclosed linear link $f(z)=\beta z$, where $0<\mu_{\min}\leq\beta\leq L_f$, and independent Gaussian noise of variance $\sigma^2>0$. Define $\mathfrak R_T(\mathcal C;\beta,\sigma,p)= \inf_\pi\sup_{\Theta_*\in\mathcal C}\mathbb E_{\Theta_*,\pi}R_T$, with the infimum over all randomized adaptive policies, which may use $p$ and the disclosed link.

\begin{mlemma}[Menu comparison for every design]
\label{lem:lower-menu}
Let $p$ satisfy Assumption~\ref{ass:candidate-distribution} with score certificate $M$. For every nonzero $v\in\mathbb R^N$, the law of $\inner v{\vecop(X)}$ has a continuous density bounded by $\sqrt M/(2\norm v_2)$. For $K\geq2$ iid candidates and $g_K^c(v)=\mathbb E\max_{a\leq K}\inner v{\vecop(X_a)-m}$,
\begin{equation}
 \frac{\norm v_2}{8\sqrt M}
 \leq g_K^c(v)
 \leq\sqrt{K\,v^\top\Sigma_Xv}
 \leq\sqrt{K\norm{\Sigma_X}_{\rm op}}\,\norm v_2.
 \label{eq:menu-comparison-general}
\end{equation}
\end{mlemma}

\noindent\textbf{Proof.}
\emph{(a) Density bound.} Equation~\eqref{eq:projection-peak} of Appendix~\ref{app:stein-pair} is derived for an arbitrary unit direction. Its derivation uses Fubini's theorem, the zero-extension $W^{1,1}$ regularity, the conditional-score identity~\eqref{eq:projection-score-identity}, the bounded support and $\mathbb E_p[\vecop(S^p)\vecop(S^p)^\top]\preceq MI_N$; it uses no selection rule and no product structure. Applied to the unit direction $v/\norm v_2$, it bounds the continuous density of $\inner{v/\norm v_2}{\vecop(X)}$ by $\sqrt M/2$. Rescaling by $\norm v_2$ gives a continuous density $q_v$ of $\inner v{\vecop(X)}$ with $\norm{q_v}_\infty\leq\sqrt M/(2\norm v_2)$.

\emph{(b) Lower bound.} Let $Y_a=\inner v{\vecop(X_a)-m}$ and $W=Y_1-Y_2$. Given $Y_2=y$, $W$ is the translate $Y_1-y$, so $W$ has density $\int q_v(w+y+\inner vm)\,\mathbb P_{Y_2}(dy)\leq\norm{q_v}_\infty$: a bounded density convolved with a probability law. Hence $\mathbb P(|W|\leq\delta)\leq\delta\sqrt M/\norm v_2$. Taking $\delta=\norm v_2/(2\sqrt M)$ gives $\mathbb P(|W|>\delta)\geq1/2$ and $\mathbb E|W|\geq\delta\,\mathbb P(|W|>\delta)\geq\norm v_2/(4\sqrt M)$. Since $\max(Y_1,Y_2)=(Y_1+Y_2)/2+|W|/2$ and $\mathbb EY_a=0$, $g_2^c(v)=\mathbb E|W|/2\geq\norm v_2/(8\sqrt M)$. For $K\geq2$, the maximum over all $K$ candidates dominates the maximum over the first two, so $g_K^c\geq g_2^c$ pointwise.

\emph{(c) Upper bound.} Since $\max_aY_a\leq(\sum_aY_a^2)^{1/2}$, Jensen's inequality gives $g_K^c(v)\leq(\sum_a\mathbb EY_a^2)^{1/2}=(Kv^\top\Sigma_Xv)^{1/2}$, and $v^\top\Sigma_Xv\leq\norm{\Sigma_X}_{\rm op}\norm v_2^2$.

Unlike the product-beta bound~\eqref{eq:menu-comparison}, this argument uses no independence, symmetry, or centering of $p$.

\noindent\textbf{Rank-feasible orthonormal families.}
For matrices with $d_1\geq d_2$, use coordinate matrices in the first $r$ columns; for $d_2>d_1$, use the first $r$ rows. Their span has dimension $dr$, rank at most $r$, and coefficient entrywise cost one. For transformed slices, use the first $\rho_k$ columns (or rows) in slice $k$, then pull back the coordinate basis through $\mathcal T_U^{-1}$. The family remains orthonormal and has dimension $d r_U$. A basis vector from slice $k$ has original-coordinate entrywise norm $\sum_a|U_{ka}|$, so the triangle inequality gives cost $\chi$. For uniform caps $r$, use original-coordinate tensors supported in the same first $r$ matrix columns (or rows) across all trailing coordinates. Transforming only mixes those trailing coordinates, preserving the zero columns or rows. This alternative span has dimension $drq$ and cost one. No invariance of the candidate law under $U$ is invoked.

Each construction supplies $E_1,\ldots,E_{s_{\max}}$ whose entire span obeys the rank constraints and satisfies $\entone{\sum_jt_jE_j}\leq\chi\sum_j|t_j|$. Every active subset of $s\leq {s_{\max}}$ directions therefore permits $\Theta_\omega=\epsilon\sum_{j=1}^s\omega_jE_j$ with independent fair signs, $\frobnorm{\Theta_\omega}=\epsilon\sqrt s$ and $\entone{\Theta_\omega}\leq\chi s\epsilon$.

\noindent\textbf{Mean-orthogonal hard directions.}
For a noncentered design, the common reward offset $\beta\inner m{\Theta_\omega}$ would reveal the signs through every selected reward, and the per-coordinate relative-entropy bound below cannot absorb it. We therefore make every hard direction orthogonal to $m$. Let $\mu_j=\inner m{E_j}$. If $\mu_j=0$ for all $j$, as for every centered design including~\eqref{eq:beta-density}, set $u_j=E_j$, $s_{\rm lb}=s_{\max}$ and $\chi_{\rm lb}=\chi$. Otherwise, for $k\leq\lfloor s_{\max}/2\rfloor$ set
\[
 u_k=\frac{\mu_{2k}E_{2k-1}-\mu_{2k-1}E_{2k}}
          {(\mu_{2k-1}^2+\mu_{2k}^2)^{1/2}}
 \ \text{ if }(\mu_{2k-1},\mu_{2k})\ne0,
 \qquad u_k=E_{2k-1}\ \text{ otherwise},
\]
and put $s_{\rm lb}=\lfloor s_{\max}/2\rfloor$ and $\chi_{\rm lb}=\sqrt2\,\chi$. The $u_k$ are orthonormal, because they have disjoint supports in the $E$-basis. Each satisfies $\inner m{u_k}=0$. Their span lies in the span of the $E_j$, so it is rank-feasible. Finally, $\entone{\sum_kt_ku_k}\leq\chi_{\rm lb}\sum_k|t_k|$, because $|a|+|b|\leq\sqrt2$ when $a^2+b^2=1$. We require $s_{\rm lb}\geq1$, which holds automatically when $s_{\max}\geq2$ or the design is centered; then $s_{\max}/3\leq s_{\rm lb}\leq s_{\max}$.

\begin{mtheorem}[Same-menu lower bound for every design]
\label{thm:lower-general}
Let $p$ satisfy Assumption~\ref{ass:candidate-distribution}, let $\mathcal C$, $s_{\max}$ and $\chi$ be as in Section~\ref{sec:lower-bounds}, and let $s_{\rm lb}\geq1$ and $\chi_{\rm lb}$ be as above. Put
\[
 \gamma_{\rm lb}=\frac{\beta B}{16\,\chi_{\rm lb}\sqrt M},\qquad
 \tau_{\rm lb}=\frac{\sigma}{256\sqrt2\,K\,M\norm{\Sigma_X}_{\rm op}}.
\]
Then for every $T\geq1$,
\begin{equation}
 \mathfrak R_T(\mathcal C;\beta,\sigma,p)
 \geq2^{-1/4}\min\bigl\{\gamma_{\rm lb} T,\ \sqrt{\gamma_{\rm lb}\tau_{\rm lb}}\,T^{3/4},\
 \tau_{\rm lb}\sqrt{s_{\rm lb}\,T}\bigr\}.
 \label{eq:lower-envelope-general}
\end{equation}
The bound still holds if the learner knows $p$ and the subspace used in the proof.
\end{mtheorem}

The envelope crosses at
\begin{equation}
 T_-(p)=\Bigl(\frac{\tau_{\rm lb}}{\gamma_{\rm lb}}\Bigr)^2,\qquad
 T_+(p)=\Bigl(\frac{\tau_{\rm lb}\,s_{\rm lb}}{\gamma_{\rm lb}}\Bigr)^2
 =\frac{\sigma^2\chi_{\rm lb}^2s_{\rm lb}^2}{512\,K^2\beta^2B^2M\norm{\Sigma_X}_{\rm op}^2}.
 \label{eq:lower-threshold-general}
\end{equation}
For $T\geq\max\{1,T_+(p)\}$ the last branch is active, so
\[
 \mathfrak R_T(\mathcal C;\beta,\sigma,p)
 \geq\frac{c\,\sigma\sqrt{s_{\rm lb}\,T}}{K\,M\norm{\Sigma_X}_{\rm op}},
 \qquad c=\frac{2^{-1/4}}{256\sqrt2}\geq\frac1{431}.
\]

For $s_{\max}\geq2$, the construction gives $s_{\rm lb}\geq\lfloor s_{\max}/2\rfloor\geq s_{\max}/3$. Thus Theorem~\ref{thm:lower-general} implies Theorem~\ref{cor:lower-large} with the universal constant reduced by a factor $1/\sqrt3$; the certificate above gives $1/(431\sqrt3)$ for the main theorem.

A prior on $s\leq s_{\rm lb}$ signed mean-orthogonal directions balances the budget constraint against the information in selected rewards, and optimizing $s$ gives~\eqref{eq:lower-envelope-general}. For a general slice profile, $1\leq\chi\leq\sqrt q$ accounts for the inverse transform in the original-coordinate budget. The crossings are construction thresholds, not established minimax transitions; growing $\chi_{\rm lb}$ changes the matching horizon. For $T\geq T_+(p)$ the square-root branch matches the dimension and horizon powers of Theorem~\ref{thm:regret} up to logarithms. This holds for every fixed design, fixed $K$, and fixed link, budget, noise and design constants, including $M\norm{\Sigma_X}_{\rm op}$; since $s_{\max}/3\leq s_{\rm lb}\leq s_{\max}$, the matrix and slice dimensions $dr$ and $dr_U$ enter up to a factor of at most three. Not matched are the menu size $K$; the ratios $L_f/\mu_{\min}$ and $V/\sigma^2$; the design factor, $\sqrt{M\norm{\Sigma_X}_{\rm op}}$ in the upper bound versus $1/(M\norm{\Sigma_X}_{\rm op})$ here (Remark~\ref{rem:lower-whitened}); and horizons below $T_+(p)$, including the $T^{3/4}$ branch. If $M\norm{\Sigma_X}_{\rm op}$ grows with $N$ along a family of designs, the two bounds differ by the growing factor $(M\norm{\Sigma_X}_{\rm op})^{3/2}$, and neither bound is shown to have sharp dimension dependence for that family.

\noindent\textbf{Proof of Theorem~\ref{thm:lower-general}.}
\emph{Prior.} Fix $s\in[s_{\rm lb}]$ and take
\[
 \epsilon=\min\left\{\frac B{\chi_{\rm lb} s},\
   \frac{\sigma}{16\sqrt2\,K\beta\norm{\Sigma_X}_{\rm op}\sqrt{MT}}\right\},
 \qquad \rho=\epsilon\sqrt s.
\]
For independent fair signs $\omega\in\{-1,1\}^s$ put $\Theta_\omega=\epsilon\sum_{k=1}^s\omega_ku_k$. Then $\frobnorm{\Theta_\omega}=\rho$, $\inner m{\Theta_\omega}=0$ and $\entone{\Theta_\omega}\leq\chi_{\rm lb} s\epsilon\leq B$; the span of the $u_k$ is rank-feasible, so $\Theta_\omega\in\mathcal C$.

\emph{Menu comparison.} Lemma~\ref{lem:lower-menu} supplies both menu bounds below. The product-beta computation~\eqref{eq:menu-comparison} is used only for Corollary~\ref{cor:lower-beta}.

\emph{Adaptive information.} Let $P_\omega^t$ be the history law through round $t$, including all observed menus and the parameter-independent policy seed. At a common history and menu, flipping sign $j$ changes only the Gaussian reward mean, by $2\beta\epsilon\omega_j\inner{X_u}{u_j}=2\beta\epsilon\omega_j\inner{X_u-m}{u_j}$, because $\inner m{u_j}=0$. Menus and the conditional policy kernel add no parameter-dependent likelihood factor. The chosen action is one of the $K$ fresh candidates, so $\mathbb E[\inner{X_u-m}{u_j}^2\mid\mathcal F_{u-1}]\leq\sum_{a\leq K}\mathbb E\inner{X_{u,a}-m}{u_j}^2\leq K\norm{\Sigma_X}_{\rm op}$. The chain rule gives
\[
 \operatorname{KL}(P_\omega^t\|P_{\omega^{(j)}}^t)
 =\frac{2\beta^2\epsilon^2}{\sigma^2}
       \sum_{u=1}^t\mathbb E_\omega\inner{X_u-m}{u_j}^2
 \leq\frac{2K\beta^2\epsilon^2t\norm{\Sigma_X}_{\rm op}}{\sigma^2}.
\]
Selected actions need not have the same marginal law in the two experiments.

\emph{Posterior mean.} Let $\bar\omega_j=\mathbb E[\omega_j\mid\mathcal F_t]$ and $\overline\Theta_t=\epsilon\sum_j\bar\omega_ju_j$. Pinsker's inequality for a fair binary sign gives $\mathbb E\bar\omega_j^2\leq2I(\omega_j;\mathcal F_t)$. Prior independence implies $I(\omega_j;\mathcal F_t)\leq I(\omega_j;\mathcal F_t\mid\omega_{-j})$. For fixed other signs, convexity of relative entropy bounds this last quantity by one quarter of the two directed KL divergences. Therefore
\[
 \mathbb E\bar\omega_j^2
 \leq\frac{2K\beta^2\epsilon^2t\norm{\Sigma_X}_{\rm op}}{\sigma^2},
 \qquad
 \mathbb E\frobnorm{\overline\Theta_t}
 \leq\rho\,\frac{\beta\epsilon\sqrt{2Kt\norm{\Sigma_X}_{\rm op}}}{\sigma},
\]
and $\inner m{\overline\Theta_t}=0$. No posterior independence is used.

\emph{Bayes regret.} A fresh menu leaves the posterior mean unchanged. Since $\inner m{\overline\Theta_{t-1}}=0$, candidate $a$ has posterior expected reward $\beta\inner{X_{t,a}-m}{\overline\Theta_{t-1}}$, so the best posterior expected reward has mean $\beta g_K^c(\overline\Theta_{t-1})\leq\beta\sqrt{K\norm{\Sigma_X}_{\rm op}}\frobnorm{\overline\Theta_{t-1}}$ by Lemma~\ref{lem:lower-menu}. Likewise $\inner m{\Theta_\omega}=0$, so the oracle's expected reward is $\beta g_K^c(\Theta_\omega)\geq\beta\rho/(8\sqrt M)$. The comparator is the realized-menu maximum, so a common mean offset would cancel from the regret in any case. Hence
\[
 \mathbb E r_t
 \geq\beta\rho\left[\frac1{8\sqrt M}
  -\frac{\sqrt2\,K\norm{\Sigma_X}_{\rm op}\beta\epsilon\sqrt T}{\sigma}\right]
 \geq\frac{\beta\rho}{16\sqrt M}
\]
by the choice of $\epsilon$. Summing over rounds, the Bayes, and hence minimax, regret is at least
\[
 \frac{\beta\rho T}{16\sqrt M}
 =\min\{\gamma_{\rm lb} T/\sqrt s,\ \tau_{\rm lb}\sqrt{sT}\}.
\]

\emph{Optimizing $s$.} For $A_0,B_0>0$, the continuous maximum over $1\leq s\leq s_{\rm lb}$ of $\min(A_0/\sqrt s,B_0\sqrt s)$ is $\min(A_0,\sqrt{A_0B_0},B_0\sqrt{s_{\rm lb}})$. Rounding the crossing to either adjacent integer loses at most $2^{-1/4}$: if $j\leq u=A_0/B_0\leq j+1$, the ratio is at least $\max(\sqrt{j/u},\sqrt{u/(j+1)})\geq(j/(j+1))^{1/4}$. With $A_0=\gamma_{\rm lb} T$ and $B_0=\tau_{\rm lb}\sqrt T$, this proves Equation~\eqref{eq:lower-envelope-general}.

The two crossings are $T_\pm(p)$ in Equation~\eqref{eq:lower-threshold-general}; for $T\geq T_+(p)$ the last branch is selected. The argument needs $\sigma>0$ and concerns the disclosed linear submodel, for any $p$.

\begin{mcorollary}[Product-beta same-menu lower bound]
\label{cor:lower-beta}
Let $p$ be the product-beta law~\eqref{eq:beta-density}, and let $\mathcal C$, $s_{\max}$ and $\chi$ be as in Theorem~\ref{thm:lower-general}. For $T\geq1$ set $\gamma_{\rm pb}=\beta B\sqrt{c_X}/\chi$ and $\tau_{\rm pb}=\sigma\sqrt{c_XM_0}$. Then
\begin{equation}
 \mathfrak R_T(\mathcal C;\beta,\sigma,p)
 \geq\frac1{20K}
 \min\{\gamma_{\rm pb}T,\ \sqrt{\gamma_{\rm pb}\tau_{\rm pb}}\,T^{3/4},\ \tau_{\rm pb}\sqrt{s_{\max}T}\}.
 \label{eq:lower-envelope}
\end{equation}
The bound still holds if the learner is told the structural subspace used in its proof.
\end{mcorollary}

The envelope crosses at
\begin{equation}
 T_-=\frac{\tau_{\rm pb}^2}{\gamma_{\rm pb}^2},\qquad
 T_+=\frac{\tau_{\rm pb}^2{s_{\max}}^2}{\gamma_{\rm pb}^2}
     =\frac{\sigma^2\chi^2M_0{s_{\max}}^2}{\beta^2B^2}.
 \label{eq:lower-threshold}
\end{equation}
The moment in this corollary is the actual $M_0=12/A^2$, not a freely enlarged certificate $M$.

\noindent\textbf{Proof of Corollary~\ref{cor:lower-beta}.}
Law~\eqref{eq:beta-density} is centered, so every $\mu_j=0$, $u_j=E_j$, $s_{\rm lb}=s_{\max}$ and $\chi_{\rm lb}=\chi$. Its independence and exact moments sharpen Lemma~\ref{lem:lower-menu}. For a fresh menu, put $g_K(v)=\mathbb E\max_{a\leq K}\inner v{X_{t,a}}$, which equals $g_K^c(v)$ here. Independence and the beta coordinate moments in Appendix~\ref{app:density-examples} give
\[
 \mathbb E\inner vX^2=c_X\frobnorm v^2,\qquad
 \mathbb E\inner vX^4
 =3c_X^2\frobnorm v^4-\frac6{13}c_X^2\sum_i v_i^4
 \leq3c_X^2\frobnorm v^4.
\]
For $W=\inner v{X_{t,1}-X_{t,2}}$, these imply $\mathbb E W^2=2c_X\frobnorm v^2$ and $\mathbb E W^4\leq12c_X^2\frobnorm v^4$. H\"older gives $\mathbb E|W|\geq\sqrt{2/3}\sqrt{c_X}\frobnorm v$. The sum of the three pairwise distances of three scalars equals twice their range. Central symmetry then gives $g_3(v)=3\mathbb E|W|/4$. Monotonicity in $K$ and the second-moment upper bound yield
\begin{equation}
 \frac{\sqrt6}{4}\sqrt{c_X}\frobnorm v
 \leq g_K(v)\leq\sqrt{Kc_X}\frobnorm v.
 \label{eq:menu-comparison}
\end{equation}
This does not identify $g_K(v)$ with a direction-independent multiple of $\frobnorm v$.

Rerun the proof of Theorem~\ref{thm:lower-general} with Equation~\eqref{eq:menu-comparison} in place of Lemma~\ref{lem:lower-menu}, with $\norm{\Sigma_X}_{\rm op}=c_X$, and with
\[
 \epsilon=\min\left\{\frac B{\chi s},
             \frac{\sqrt3\,\sigma}{8K\beta\sqrt{Tc_X}}\right\}.
\]
The information and posterior steps give $\mathbb E\frobnorm{\overline\Theta_t}\leq\rho\sqrt{2K\beta^2\epsilon^2tc_X/\sigma^2}\leq\sqrt6\rho/(8\sqrt K)$. The best posterior expected reward therefore has mean at most $(\sqrt6/8)\beta\sqrt{c_X}\rho$, whereas the oracle's expected reward is at least $(\sqrt6/4)\beta\sqrt{c_X}\rho$. The Bayes, and hence minimax, regret is at least
\[
 \frac{\sqrt6}{8}\beta\sqrt{c_X}\rho T
 \geq\frac1{16K}\min\{\gamma_{\rm pb}T/\sqrt s,\ \sigma\sqrt{sT}\}.
\]
The rounding argument with $A_0=\gamma_{\rm pb}T$ and $B_0=\sigma\sqrt T$ gives a lower envelope with coefficient $2^{-1/4}/(16K)$ and $\sigma$ in place of $\tau_{\rm pb}$. Since $\sqrt{c_XM_0}=\sqrt{12/11}$ and $2^{-1/4}\sqrt{11/12}\geq4/5$, replacing $\sigma$ by $\tau_{\rm pb}$ gives Equation~\eqref{eq:lower-envelope} with coefficient $1/(20K)$. The score moment enters only through this fixed-law numerical conversion.

\refstepcounter{mtheorem}\label{rem:lower-whitened}
\noindent\textbf{Remark~\themtheorem\ (Scale-free constant).}
The factor $1/(M\norm{\Sigma_X}_{\rm op})$ in Theorem~\ref{thm:lower-general} comes from measuring the hard directions in Euclidean rather than covariance geometry. Let $\mathcal U=\operatorname{span}\{u_1,\ldots,u_{s_{\rm lb}}\}$, of dimension $s_{\rm lb}$, and let $\Sigma_{\mathcal U}$ be the compression of $\Sigma_X$ to $\mathcal U$. The identity $\mathbb E[(u^\top(x-\bar x))(u^\top\vecop(S^p(X)))]=1$ of Appendix~\ref{app:regret-proof} and Cauchy--Schwarz give $u^\top\Sigma_Xu\geq1/M$ for every unit $u$, so $\lambda_{\min}(\Sigma_{\mathcal U})\geq1/M>0$. Choose the hard directions $w_1,\ldots,w_{s_{\rm lb}}\in\mathcal U$ orthonormal for $\inner ab_\Sigma=a^\top\Sigma_Xb$, write $\norm v_\Sigma=\inner vv_\Sigma^{1/2}$, and define
\[
 a_p=\inf_{0\ne v\in\mathcal U}
     \frac{\mathbb E|\inner v{X_1-X_2}|}{2\sqrt{v^\top\Sigma_Xv}}
 \in(0,2^{-1/2}].
\]
The upper end follows from $\mathbb E|W|\leq(\mathbb EW^2)^{1/2}$. Run the proof of Theorem~\ref{thm:lower-general} with $s=s_{\rm lb}$, prior $\Theta_\omega=\epsilon\sum_k\omega_kw_k$ and three changes. First, $\norm{\Sigma_X}_{\rm op}$ becomes $1$ in the information and upper-menu steps, because $w_j^\top\Sigma_Xw_j=1$ and $g_K^c(v)\leq\sqrt K\norm v_\Sigma$; the posterior step then bounds $\mathbb E\norm{\overline\Theta_t}_\Sigma$, with $\rho=\norm{\Theta_\omega}_\Sigma=\epsilon\sqrt{s_{\rm lb}}$. Second, the oracle floor $1/(8\sqrt M)$ becomes $a_p$, because $g_K^c\geq g_2^c=\mathbb E|W|/2\geq a_p\norm v_\Sigma$ on $\mathcal U$. Third, $\epsilon=a_p\sigma/(2\sqrt2K\beta\sqrt T)$, which makes the per-round bracket $a_p/2$. For the budget, every $\Theta\in\mathcal U$ has $\entone\Theta\leq\chi_{\rm lb}\sqrt{s_{\rm lb}}\frobnorm\Theta\leq\chi_{\rm lb}\sqrt{s_{\rm lb}}\norm\Theta_\Sigma/\sqrt{\lambda_{\min}(\Sigma_{\mathcal U})}$, so $\entone{\Theta_\omega}\leq\chi_{\rm lb}s_{\rm lb}\epsilon/\sqrt{\lambda_{\min}(\Sigma_{\mathcal U})}$, which is at most $B$ once $T\geq T_+^\Sigma(p)$ below. Hence
\[
 \mathfrak R_T(\mathcal C;\beta,\sigma,p)
 \geq\frac{a_p^2\,\sigma\sqrt{s_{\rm lb}T}}{4\sqrt2\,K}
 \quad\text{for}\quad
 T\geq T_+^\Sigma(p)
 =\frac{a_p^2\sigma^2\chi_{\rm lb}^2s_{\rm lb}^2}
       {8K^2\beta^2B^2\lambda_{\min}(\Sigma_{\mathcal U})}.
\]
Two sufficient conditions bound $a_p$ below. Lemma~\ref{lem:lower-menu}(b) gives $\mathbb E|W|\geq\norm v_2/(4\sqrt M)$, hence $a_p\geq1/(8\sqrt{M\lambda_{\max}(\Sigma_{\mathcal U})})$, which recovers the large-horizon coefficient of Theorem~\ref{thm:lower-general} up to constants, under the threshold above. If $\mathbb E\inner v{X-m}^4\leq L(v^\top\Sigma_Xv)^2$ on $\mathcal U$, then H\"older's inequality $\mathbb E|W|\geq(\mathbb EW^2)^{3/2}(\mathbb EW^4)^{-1/2}$ and $\mathbb EW^4=2\mathbb EY^4+6(\mathbb EY^2)^2$, with $Y=\inner v{X-m}$, give $a_p^2\geq1/(L+3)$.

Independent coordinates with coordinate kurtosis at most $L_0$ give $L\leq\max\{L_0,3\}$. Law~\eqref{eq:beta-density} has $L_0=363/143$, so $L\leq3$. For the correlated designs of Appendix~\ref{app:density-examples}, $m=b$ and $\inner{\vecop(X)-b}v=\inner z{H^\top v}$, where $z$ has independent product-beta coordinates. Hence $L\leq3$ for every invertible $H$, and $\mathfrak R_T(\mathcal C;\beta,\sigma,p)\geq\sigma\sqrt{s_{\rm lb}T}/(24\sqrt2K)$ for $T\geq T_+^\Sigma(p)$. The constant does not depend on $H$; the threshold does, through $\lambda_{\min}(\Sigma_{\mathcal U})$ and the construction of $\mathcal U$. For this family Theorem~\ref{thm:regret} carries $\sqrt{M\norm{\Sigma_X}_{\rm op}}=\sqrt{12/11}\operatorname{cond}_2(H)$. The remaining design gap therefore lies on the upper-bound side: it is a factor of the analysis of Theorem~\ref{thm:regret} that may be loose, not a demonstrated suboptimality of T-ESTOR.

\section{T-BSTOR for Arbitrary Links}
\label{app:nonmonotone}

This appendix treats Lipschitz links that need not be monotone. T-BSTOR first estimates a signed multiple of the structured parameter from uniform observations, then learns rewards along the frozen estimated index; its scalar confidence bounds allow adaptive sampling and finite-variance noise. Matrix and transformed-slice fits are treated together.

\subsection{Local Assumptions and Structured Index Estimation}
\label{app:nonmonotone-model}

Keep the parameter classes, entrywise budgets, fresh iid menus, known base density and score, and conditional noise kernel from Section~\ref{sec:preliminaries}. Retain absolute continuity and $L_f$-Lipschitz continuity on $[-AB,AB]$, without imposing monotonicity. There is no restriction on the additive reward level $f(0)$. For a nonzero parameter assume only
\[
 \mu_*=\mathbb E_{p}f'(\inner X{\Theta_*}),\qquad
 |\mu_*|\geq\gamma>0.
 \label{eq:nonmonotone-signal}
\]
In particular, $|\mu_*|\leq L_f$, and no selected-design signal condition is imposed. A zero parameter has identically zero regret and is handled separately without evaluating $f'(0)$.

Candidate entries may be dependent and have arbitrary means; write $m=\mathbb E_pX$, and retain $\maxnorm X<A$, the zero-extension regularity $p\in W^{1,1}(\mathbb R^N)$, and the joint score certificate $\mathbb E_p[\vecop(S^p(X))\vecop(S^p(X))^\top]\preceq MI_N$ from Assumption~\ref{ass:candidate-distribution}. The prediction step uses the projection certificate $\kappa_X\geq1$ of Equation~\eqref{eq:nonmonotone-joint-concentration}, which concerns candidate projections rather than reward noise. A covariance bound alone does not establish that exponential inequality with the same constant. Every bounded law admits $\kappa_X=\sqrt N$: $\inner XZ$ lies in an interval of length at most $2A\entone Z\leq2A\sqrt N\frobnorm Z$, so Hoeffding's lemma~\citep{hoeffding1963probability} gives Equation~\eqref{eq:nonmonotone-joint-concentration} whatever the dependence. Independent bounded entries have $\kappa_X=1$ for any means, by applying the same lemma coordinatewise; Appendix~\ref{app:density-examples} gives correlated, noncentered designs with $\kappa_X=1$. Stronger dependence can force $\kappa_X$ to grow: Equation~\eqref{eq:nonmonotone-joint-concentration} implies $A^2\kappa_X^2\geq\norm{\Sigma_X}_{\rm op}$, which is of order $N$ when all coordinates share a random factor of constant variance. Write $V=2L_f^2A^2B^2+\sigma^2$ and $\Delta=2L_fAB$; every realized-menu mean-reward gap is at most $\Delta$.

Let $n_0\geq2$ be an even number of uniform-selection rounds. Its $n=n_0/2$ disjoint reward--score pairs are
\[
 Z_j=\frac{y_{2j-1}-y_{2j}}2
       [S^p(X_{2j-1})-S^p(X_{2j})],\qquad j\in[n_0/2].
\]
Uniform selection independently of each menu and the past gives iid selected observations with density $p$. This independence is between observations; entries of one observation need not be independent. Denote the resulting structured fit, before the additional projection below, by $\widehat\Theta^0$. This is the fitted matrix or transformed-slice estimator, not the unregularized robust moment.

\begin{mlemma}[Signed Stein estimation from uniform observations]
\label{lem:nonmonotone-index}
Under the local assumptions above, with no signal lower bound needed for this lemma, the paired mean is $\mathbb E Z_j=\mu_*\Theta_*$, and the paired second-moment inequalities of Lemma~\ref{lem:paired} hold with the base score and the same $V$. For $\delta_1\in(0,1)$, let $d=\max(d_1,d_2)$. Use Equation~\eqref{eq:slice-fit}, with its matrix specialization when $q=1$, and the following tuning and radii:
\begin{equation}
\begin{gathered}
 \text{Matrix:}\quad
 \nu=\sqrt{\frac{4\log(2(d_1+d_2)/\delta_1)}{n_0VdM}},\qquad
 \lambda=4\sqrt{\frac{VdM\log(2(d_1+d_2)/\delta_1)}{n_0}},\\
 e_{n_0}^2=\frac{32VMrd\log(2(d_1+d_2)/\delta_1)}{n_0};\\[3pt]
 \text{Transformed slices:}\quad
 \nu=\sqrt{\frac{4\log(2q(d_1+d_2)/\delta_1)}{n_0VdM}},\qquad
 \lambda=4\sqrt{\frac{VdM\log(2q(d_1+d_2)/\delta_1)}{n_0}},\\
 e_{n_0}^2=\frac{32VMdr_U\log(2q(d_1+d_2)/\delta_1)}{n_0}.
\end{gathered}
\label{eq:nonmonotone-estimation}
\end{equation}
Then
\[
 \mathbb P\{\frobnorm{\widehat\Theta^0-\mu_*\Theta_*}\leq e_{n_0}\}
 \geq1-\delta_1.
\]
The matrix and slice tuning does not use their rank caps.
\end{mlemma}

\noindent\textbf{Proof.}
Weak integration by parts gives $\mathbb ES^p(X)=0$ and $\mathbb E[yS^p(X)]=\mu_*\Theta_*$ without using the sign of $f'$, centering of $X$, or independence of its entries. Thus the proof of Lemma~\ref{lem:paired} applies with the base score: $\mathbb EZ_j=\mu_*\Theta_*$, and the raw second moments of $Z_j$ are at most $V$ times those of $S^p$. The row and column calculation of Appendix~\ref{app:score-moments} bounds the left and right base-score moments of every matrix or transformed slice by $dM$, so $w=dM$, without slice independence or a slice-count factor. The concentration and thresholding arguments of Appendices~\ref{app:matrix-proof}--\ref{app:slice-proof} use only these raw moments and upper bounds on the target ranks, which remain valid for negative or zero $\mu_*$. Applying Theorem~\ref{thm:structured-estimation} with $n=n_0/2$ and $w=dM$ gives Equation~\eqref{eq:nonmonotone-estimation}. The nonzero base-signal condition is needed for the subsequent prediction and regret analysis, not for this estimation lemma.

\subsection{Projection, Scalar Approximation, and T-BSTOR}
\label{app:nonmonotone-algorithm}

For the estimation branch with even $2\leq n_0<T$, project the fitted index onto a closed convex entrywise ball:
\begin{equation}
 \widehat\Theta=\mathop{\arg\min}_{Z:\,\entone Z\leq L_fB}
                  \frobnorm{Z-\widehat\Theta^0}^2,
 \qquad \mathcal I=[-AL_fB,AL_fB],\quad W=2AL_fB.
 \label{eq:nonmonotone-projection}
\end{equation}
Since $\entone{\mu_*\Theta_*}\leq |\mu_*|B\leq L_fB$, the Euclidean projection property gives $\frobnorm{\widehat\Theta-\mu_*\Theta_*}\leq \frobnorm{\widehat\Theta^0-\mu_*\Theta_*}$. Projection may change the fitted ranks; no subsequent step uses those ranks. For every candidate, $|\inner X{\widehat\Theta}|\leq AL_fB$, even when index estimation fails.

Condition on the exploration history and put $E=\widehat\Theta-\mu_*\Theta_*$ and $c_E=\inner mE$. For analysis only, define the shifted link
\[
 g_E(u)=f\!\left(\operatorname{clip}_{[-AB,AB]}
                   \left(\frac{u-c_E}{\mu_*}\right)\right),
 \qquad G=\frac{L_f}{|\mu_*|}\leq\frac{L_f}\gamma.
 \label{eq:nonmonotone-scalar-link}
\]
The function $g_E$ is fixed conditional on exploration and is $G$-Lipschitz. Since
\[
 \frac{\inner X{\widehat\Theta}-c_E}{\mu_*}
 =\inner X{\Theta_*}+\frac{\inner{X-m}E}{\mu_*},
\]
and the true index belongs to $[-AB,AB]$, nonexpansiveness of clipping around that index gives
\begin{equation}
 \left|f(\inner X{\Theta_*})-g_E(\inner X{\widehat\Theta})\right|
 \leq G|\inner{X-m}E|.
 \label{eq:nonmonotone-scalar-shift-error}
\end{equation}
This holds for either sign of $\mu_*$. The candidate mean contributes a common horizontal shift of the unknown scalar link. The learner does not compute $m,E,c_E,\mu_*$, or $g_E$: it continues to bin the uncentered fitted indices in the same interval $\mathcal I$. No feature centering, whitening, extra intercept, or score change is required.

The algorithm takes $T,K,A,B,L_f,\sigma^2,M$, exact base-score access, the chosen structural fit and its required structural inputs, a deterministic actual exploration count $n_0\leq T$, with either $n_0=T$ or even $2\leq n_0<T$, a bin width $w_{\rm bin}>0$, and confidence allocations $\delta_1,\delta_2,\delta_3\in(0,1)$ with $\delta_1+\delta_2+\delta_3<1$. The allocation $\delta_2$ is used only to report the prediction bound. The radii $e_{n_0}$, $\gamma$, $\mu_*$, $g_E$, and the approximation allowance below are not action-rule inputs. If $n_0=T$, the algorithm selects uniformly throughout and performs no fit or bin construction. This branch permits odd $T$. The following definitions apply only when $n_0<T$.

Put $Q=T-n_0>0$ and $J_{\rm bin}=\lceil W/w_{\rm bin}\rceil$. Partition $\mathcal I$ into $J_{\rm bin}$ equal bins $I_b$ of length $W/J_{\rm bin}\leq w_{\rm bin}$, with centers $c_b$. The bins are left-closed and right-open, except that $I_{J_{\rm bin}}$ includes the right endpoint. Let $k_{\rm blk}$ be the smallest odd positive integer satisfying
\begin{equation}
 k_{\rm blk}\geq
 \max\left\{1,\frac2{\log2}\log\frac{J_{\rm bin}T}{\delta_3}\right\}.
 \label{eq:nonmonotone-block-count}
\end{equation}
For second-phase round $t$, let $b_t$ denote the chosen bin and $n_b(t)=\sum_{u=n_0+1}^t\mathbf1\{b_u=b\}$ its selection count through round $t$, with $n_b(n_0)=0$. For $s\geq k_{\rm blk}$, its confidence score is
\begin{equation}
 U_{b,s}=\widehat m_{b,s}+\operatorname{rad}(s),\qquad
 \operatorname{rad}(s)=4\sigma\sqrt{k_{\rm blk}/s}.
 \label{eq:nonmonotone-bin-score}
\end{equation}
Here $\ell_s=\lfloor s/k_{\rm blk}\rfloor$ and the median uses the first $k_{\rm blk}\ell_s$ stored rewards in $k_{\rm blk}$ consecutive equal blocks:
\begin{equation}
 \widehat m_{b,s}=\operatorname{median}_{1\leq j\leq k_{\rm blk}}
 \left\{\frac1{\ell_s}
   \sum_{a=(j-1)\ell_s+1}^{j\ell_s}y_{b,a}\right\}.
 \label{eq:nonmonotone-mom}
\end{equation}
The remainder is ignored for this evaluation and retained in the reward list. Uniform draws below are independent of the current menu and past.

\begin{algorithm}[!ht]
\caption{T-BSTOR: uniform structured exploration and robust available-bin learning.}
\label{alg:nonmonotone}
\begin{algorithmic}[1]
\algrenewcommand{\algorithmicrequire}{\textbf{Input:}}
\Require Model, structural fit, and tuning inputs specified above.
\If{$n_0=T$}
  \State Select uniformly and independently from each fresh menu for all $T$ rounds
  \State \Return
\EndIf
\State Clear the exploration buffer
\For{$t=1,\ldots,n_0$}
  \State Observe $\mathcal X_t$; draw $a\sim\operatorname{Unif}([K])$ and play $X_{t,a}$
  \State Observe $y_t$ and store $(y_t,S^p(X_{t,a}))$
\EndFor
\State Form $n_0/2$ disjoint pairs by \eqref{eq:paired-observation} with the base score
\State Compute the regularized fit $\widehat\Theta^0$ with tuning \eqref{eq:nonmonotone-estimation}
\State Project $\widehat\Theta^0$ by \eqref{eq:nonmonotone-projection} and freeze $\widehat\Theta$
\State Set $Q,J_{\rm bin}$, bins $I_b$, and $k_{\rm blk}$ as above; initialize empty reward lists and $n_b(n_0)=0$
\For{$t=n_0+1,\ldots,T$}
  \State Observe $\mathcal X_t$ and compute $z_{t,a}=\inner{X_{t,a}}{\widehat\Theta}$ for every $a\in[K]$
  \State $\mathcal A_t\leftarrow\{b:\text{some }z_{t,a}\in I_b\}$ \Comment{Currently available bins}
  \If{some $b\in\mathcal A_t$ has $n_b(t-1)<k_{\rm blk}$}
    \State $b_t\leftarrow\min\{b\in\mathcal A_t:n_b(t-1)<k_{\rm blk}\}$
  \Else
    \State Evaluate $U_{b,n_b(t-1)}$ by \eqref{eq:nonmonotone-bin-score}--\eqref{eq:nonmonotone-mom} for $b\in\mathcal A_t$
    \State Choose $b_t$ maximizing this score, breaking ties by smallest bin number
  \EndIf
  \State $a\leftarrow\min\{a:z_{t,a}\in I_{b_t}\}$; play $X_{t,a}$ and observe $y_t$
  \State Append $y_t$ to bin $b_t$'s list and increment only its count
\EndFor
\end{algorithmic}
\end{algorithm}

Only currently available bins are compared. Initialization selects a bin with fewer than $k_{\rm blk}$ rewards, so each bin receives at most $k_{\rm blk}$ such selections. Their total count is at most $\min\{T-n_0,J_{\rm bin}k_{\rm blk}\}$ selections; these selections need not form a contiguous phase because bins become available at different times.

\begin{mlemma}[Uniform prediction and bin approximation]
\label{lem:nonmonotone-prediction}
Let $Q\geq1$ and use a frozen projected estimate satisfying $\mathbb P\{\frobnorm{\widehat\Theta-\mu_*\Theta_*}\leq e_{n_0}\}\geq1-\delta_1$, as provided by Lemma~\ref{lem:nonmonotone-index} and projection. Set
\[
 \varepsilon_{\rm pred}=\frac{L_fA\kappa_X}{\gamma}e_{n_0}
            \sqrt{2\log\frac{2KQ}{\delta_2}},\qquad
 \epsilon_{\rm bin}=\varepsilon_{\rm pred}+\frac{Gw_{\rm bin}}{2},\qquad m_b=g_E(c_b).
 \label{eq:nonmonotone-approximation}
\]
With probability at least $1-\delta_1-\delta_2$, every future candidate whose fitted index falls in bin $b$ satisfies
\[
 |f(\inner X{\Theta_*})-m_b|\leq\epsilon_{\rm bin}.
\]
\end{mlemma}

\noindent\textbf{Proof.}
Conditional on the exploration history $\mathcal F_{n_0}$, the error $E=\widehat\Theta-\mu_*\Theta_*$ is fixed and future candidates are independent of that history with their original law $p$. Applying Equation~\eqref{eq:nonmonotone-joint-concentration} in the fixed direction $E$ and optimizing the exponential Markov bound gives, for $E\ne0$,
\[
 \mathbb P\{|\inner{X-m}E|>u\mid\mathcal F_{n_0}\}
 \leq2\exp\!\left(-\frac{u^2}{2A^2\kappa_X^2\frobnorm E^2}\right).
\]
If $E=0$, the projection error is zero directly. On the index-confidence event $\frobnorm E\leq e_{n_0}$, a union bound over the $KQ$ future candidates gives $\max_{t,a}|\inner{X_{t,a}-m}E|\leq A\kappa_Xe_{n_0}\sqrt{2\log(2KQ/\delta_2)}$, except with conditional probability at most $\delta_2$. Equation~\eqref{eq:nonmonotone-scalar-shift-error} bounds the resulting reward-prediction error by $\varepsilon_{\rm pred}$, since $G\leq L_f/\gamma$. Every bin center is at most $w_{\rm bin}/2$ from its indices. Adding this within-bin error and the estimation failure probability $\delta_1$ proves the assertion. The union is over candidate arrays, not directions: $E$ is fixed by conditioning. No covering argument or dimension factor beyond the supplied $\kappa_X$ is needed.

\subsection{Confidence for Adaptively Sampled Bins}
\label{app:nonmonotone-bin-proof}

Condition only on $\mathcal F_{n_0}$, which fixes the fitted index, partition, shifted link $g_E$, and reference means $m_b=g_E(c_b)$. Let $\mathcal H_t$ contain the past, current menu, action randomization, and selected action, before observing its reward. The retained causal noise kernel implies
\[
 \mathbb E[\eta_t\mid\mathcal H_t]=0,\qquad
 \mathbb E[\eta_t^2\mid\mathcal H_t]\leq\sigma^2.
 \label{eq:nonmonotone-causal-noise}
\]
The rewards assigned to a bin need not be iid or have a constant conditional mean. On the approximation event $\mathcal E_{\rm app}$ of Lemma~\ref{lem:nonmonotone-prediction}, their representation is
\[
 y_{b,a}=m_b+d_{b,a}+\eta_{b,a},\qquad
 |d_{b,a}|\leq\epsilon_{\rm bin}.
\]
The drift is determined before its own reward noise and may vary with the adaptively selected candidate.

\begin{mlemma}[Adaptive-bin confidence]
\label{lem:nonmonotone-mom}
For the block count in Equation~\eqref{eq:nonmonotone-block-count}, there is an event $\mathcal E_{\rm noise}$ satisfying $\mathbb P(\mathcal E_{\rm noise}^c\mid\mathcal F_{n_0})\leq\delta_3$ such that, on $\mathcal E_{\rm noise}\cap\mathcal E_{\rm app}$,
\begin{equation}
 |\widehat m_{b,s}-m_b|\leq\epsilon_{\rm bin}
                          +4\sigma\sqrt{k_{\rm blk}/s}
 \label{eq:nonmonotone-bin-confidence}
\end{equation}
simultaneously for every bin and every actual observed prefix $k_{\rm blk}\leq s\leq n_b(T)$.
\end{mlemma}

\noindent\textbf{Proof: adaptive times and padding.}
For each bin $b$, consider its chronological noise martingale and count
\[
 W_t^b=\sum_{u=n_0+1}^{\min(t,T)}\mathbf1\{b_u=b\}\eta_u,
 \qquad C_t^b=\sum_{u=n_0+1}^{\min(t,T)}\mathbf1\{b_u=b\}.
\]
Conditioning first on $\mathcal H_t$ proves that $W_t^b$ is a square-integrable martingale and $(W_t^b)^2-\sigma^2C_t^b$ is a supermartingale with respect to the chronological post-reward history. The selection indicator need not be measurable before the current menu; its measurability in $\mathcal H_t$ is sufficient.

A bin may never receive a prescribed number of observations. For the proof only, after time $T$ pad its sequence with rewards $m_b$, zero drift, and zero noise, one per artificial round. The padded $a$th selection time $\tau_{b,a}$ is a stopping time bounded by $T+a$; the post-horizon history contains no new randomness. Take $\tau_{b,0}=n_0$ and $\mathcal G_{b,a}=\mathcal F_{\tau_{b,a}}$. Optional sampling of the two processes above gives
\begin{equation}
 \mathbb E[\eta_{b,a}\mid\mathcal G_{b,a-1}]=0,\qquad
 \mathbb E[\eta_{b,a}^2\mid\mathcal G_{b,a-1}]
 \leq\sigma^2.
 \label{eq:nonmonotone-stopped-noise}
\end{equation}
Indeed, $W^b_{\tau_{b,a}}-W^b_{\tau_{b,a-1}}$ is this noise increment, and $C^b$ increases by at most one between these times; martingale centering removes the cross term in the squared increment. All intervening menus and other bins' rewards are included in the stopped history. No conditioning on the eventual number of pulls is used, and padding never changes an actual observed prefix.

\noindent\textbf{Proof: deterministic prefixes and dependent blocks.}
Fix a bin and a deterministic integer $s\in[k_{\rm blk},T]$. The block size $\ell_s=\lfloor s/k_{\rm blk}\rfloor$ is deterministic. For the noise average $\bar\eta_j$ in block $j$, iterating Equation~\eqref{eq:nonmonotone-stopped-noise} cancels all cross terms conditional on the block-start history, so
\[
 \mathbb E[\bar\eta_j^2\mid
     \mathcal G_{b,(j-1)\ell_s}]\leq\frac{\sigma^2}{\ell_s}.
\]
Conditional Chebyshev bounds the probability of $|\bar\eta_j|>\sqrt{8\sigma^2/\ell_s}$ by $1/8$. When $\sigma=0$, all actual and padded noises are zero almost surely, and this conclusion holds directly. For any specified set of $a$ bad blocks, iterated conditioning in reverse temporal order bounds their joint probability by $8^{-a}$. Thus, without block independence,
\[
 \mathbb P\{\text{at least }(k_{\rm blk}+1)/2
       \text{ bad blocks}\mid\mathcal F_{n_0}\}
 \leq2^{k_{\rm blk}}8^{-k_{\rm blk}/2}
 =2^{-k_{\rm blk}/2}.
\]
There are at most $J_{\rm bin}T$ pairs $(b,s)$, all defined by the padded sequences. A union bound and Equation~\eqref{eq:nonmonotone-block-count} give an event of conditional failure probability at most $\delta_3$ on which every such prefix has a strict majority of good blocks. This event covers the random counts used by the algorithm because it was constructed for deterministic prefix sizes.

Now intersect with $\mathcal E_{\rm app}$. Every actual block's average drift has magnitude at most $\epsilon_{\rm bin}$, so a strict majority of its reward block averages lie within $\epsilon_{\rm bin}+\sqrt{8\sigma^2/\ell_s}$ of $m_b$. Their median does too. Since $s\geq k_{\rm blk}$ implies $\ell_s\geq s/(2k_{\rm blk})$, this radius is at most that in Equation~\eqref{eq:nonmonotone-bin-confidence}. The noise event was proved conditional only on exploration history, not on $\mathcal E_{\rm app}$ or all future menus. The final intersection uses a union bound, not independence of the two events.

\subsection{Upper Regret Bound and Exploration Schedule}
\label{app:nonmonotone-upper}

\begin{mtheorem}[Uniform exploration and robust-bin regret]
\label{thm:nonmonotone-upper-full}
Under the local assumptions in Appendix~\ref{app:nonmonotone-model}, with $\Theta_*\ne0$, run Algorithm~\ref{alg:nonmonotone} with any deterministic even $n_0$ satisfying $2\leq n_0<T$, width $w_{\rm bin}>0$, and the stated confidence allocations. Use any deterministic valid radius $e_{n_0}$ with $\mathbb P\{\frobnorm{\widehat\Theta^0-\mu_*\Theta_*}\leq e_{n_0}\}\geq1-\delta_1$, as in Equation~\eqref{eq:nonmonotone-estimation}. With probability at least $1-\delta_1-\delta_2-\delta_3$,
\begin{equation*}
 \begin{split}
 R_T\leq\mathcal U_T:={}&\Delta n_0+\Delta\min\{Q,J_{\rm bin}k_{\rm blk}\}
       +16\sigma\sqrt{k_{\rm blk}J_{\rm bin}Q}\\
       &+4Q\varepsilon_{\rm pred}+2GQw_{\rm bin}.
 \end{split}
 \label{eq:nonmonotone-upper-full}
\end{equation*}
Consequently,
\begin{equation}
 \mathbb E R_T\leq
 \min\{\Delta T,\ \mathcal U_T+\Delta T(\delta_1+\delta_2+\delta_3)\}.
 \label{eq:nonmonotone-upper-expected}
\end{equation}
The unknown $G$ can be replaced by its upper bound $L_f/\gamma$ when reporting these bounds. The all-exploration branch $n_0=T$ has $R_T\leq\Delta T$ directly; it uses no $Q$-dependent logarithm or undefined estimator when $Q=0$.
\end{mtheorem}

\noindent\textbf{Proof.}
If $n_0=T$, the per-round gap bound gives $R_T\leq\Delta T$ directly. Now take even $2\leq n_0<T$, so $Q>0$. Lemma~\ref{lem:nonmonotone-prediction} bounds the approximation-event failure probability by $\delta_1+\delta_2$. The conditional noise bound in Lemma~\ref{lem:nonmonotone-mom}, integrated over the exploration history, contributes at most $\delta_3$. Work on their intersection. The exploration phase costs at most $\Delta n_0$. At most $k_{\rm blk}$ initialization selections are charged to each bin, hence their total count is at most $\min\{Q,J_{\rm bin}k_{\rm blk}\}$.

On any other round, every available bin has at least $k_{\rm blk}$ observations. Let $b_\star$ contain a true reward-maximizing current candidate, let $b_t$ be selected, and write $s=n_{b_t}(t-1)$. The confidence event and maximization of Equation~\eqref{eq:nonmonotone-bin-score} give
\[
 m_{b_\star}\leq U_{b_\star,n_{b_\star}(t-1)}+\epsilon_{\rm bin}
 \leq U_{b_t,s}+\epsilon_{\rm bin}
 \leq m_{b_t}+2\operatorname{rad}(s)+2\epsilon_{\rm bin}.
\]
Both actual candidate rewards are within $\epsilon_{\rm bin}$ of their respective references, so $r_t\leq2\operatorname{rad}(s)+4\epsilon_{\rm bin}$. The allowance $\epsilon_{\rm bin}$ is common to every bin, which is why the algorithm need not add it to its scores or know its value. For each bin the pre-selection count takes each integer value at most once. Therefore
\[
 \sum_{\text{noninitialization }t}n_{b_t}(t-1)^{-1/2}
 \leq\sum_{b=1}^{J_{\rm bin}}\sum_{s=1}^{n_b(T)}s^{-1/2}
 \leq2\sum_{b=1}^{J_{\rm bin}}\sqrt{n_b(T)}\leq2\sqrt{J_{\rm bin}Q}.
\]
The confidence-width sum is at most $16\sigma\sqrt{k_{\rm blk}J_{\rm bin}Q}$. Finally, $4Q\epsilon_{\rm bin}=4Q\varepsilon_{\rm pred}+2GQw_{\rm bin}$. This proves the realized-menu mean-regret bound. The deterministic cap $R_T\leq\Delta T$ adds at most $\Delta T$ times the combined failure probability when taking expectation, proving Equation~\eqref{eq:nonmonotone-upper-expected}.

\paragraph{Explicit schedule and structural inputs.}
For a nonzero parameter class, supply $D\geq1$ satisfying
\begin{equation}
 D\geq rd\quad\text{for the matrix fit},\qquad
 D\geq dr_U\quad\text{for the transformed-slice fit}.
 \label{eq:nonmonotone-structural-dimensions}
\end{equation}
Set
\begin{equation}
 \begin{aligned}
 n_0&=\min\!\left\{T,\;2\left\lceil\frac{(\kappa_X^2D)^{1/3}T^{2/3}}2\right\rceil\right\},\\
 w_{\rm bin}&=WT^{-1/3},\qquad
 \delta_1=\delta_2=\delta_3=\frac1{3(T+1)^2}.
 \end{aligned}
 \label{eq:nonmonotone-schedule}
\end{equation}
If $n_0=T$, select uniformly throughout without fitting or initializing bins, even when $T$ is odd. Otherwise $2\leq n_0<T$ is even, so the algorithm fits using $n_0/2$ pairs and runs the bin policy for the remaining rounds. In particular, $T=1,2$ use the all-exploration branch. A known zero parameter class may use any policy and has zero regret.

The exact bounds above display every score, noise, signal, concentration, and confidence factor. To read their rate, hold $A,B,L_f,M,\gamma,\sigma,K,h$ fixed, retain $D$ and $\kappa_X$ explicitly, and suppress logarithms in dimensions and $T$. For the estimation branch with even $2\leq n_0<T$, Equation~\eqref{eq:nonmonotone-estimation} gives $e_{n_0}=\widetilde O(\sqrt{D/n_0})$ and the upper bound becomes
\[
 \mathbb E R_T=\widetilde O\!\left(
 n_0+T\kappa_X\sqrt{D/n_0}+Tw_{\rm bin}
       +\sqrt{T/w_{\rm bin}}+1/w_{\rm bin}\right).
\]
Under Equation~\eqref{eq:nonmonotone-schedule} in this branch, $J_{\rm bin}=\lceil T^{1/3}\rceil$, $k_{\rm blk}=O(\log T)$, and hence
\begin{equation}
 \mathbb E R_T=\widetilde O\!\left(
 \min\{T,[(\kappa_X^2D)^{1/3}+1]T^{2/3}\}\right).
 \label{eq:nonmonotone-rate}
\end{equation}
Before capping, rounding adds fewer than two rounds. If $n_0=T$, then $2\lceil(\kappa_X^2D)^{1/3}T^{2/3}/2\rceil\geq T$; use $R_T\leq\Delta T$ directly. Either $T\leq3$ or $(\kappa_X^2D)^{1/3}T^{2/3}>T-2\geq T/2$, so its linear cost obeys the same rate. The initialization cost is of order $T^{1/3}$ times a logarithm under this schedule and is explicitly retained in the finite-sample theorem.

\begin{mcorollary}[Sensitivity to the exploration schedule]
\label{cor:nonmonotone-schedule-sensitivity}
Retain the local model assumptions of Theorem~\ref{thm:nonmonotone-upper-full}. Let $D\geq1$ satisfy Equation~\eqref{eq:nonmonotone-structural-dimensions} and let $\kappa_X\geq1$ be a valid projection certificate. These are analytical bounds; the schedule may use a misspecified multiple of their product, without supplying rank caps to the fit. For an integer horizon $T\geq1$ and $\varrho>0$, set
\[
 a=\varrho^{1/3}(\kappa_X^2D)^{1/3}T^{2/3},\qquad
 n_0=\min\{T,\,2\lceil a/2\rceil\}.
\]
If $n_0<T$, keep the bin width and confidence allocations of Equation~\eqref{eq:nonmonotone-schedule}. If $n_0=T$, select uniformly throughout without fitting or constructing bins, including for odd $T$. Then
\[
 \mathbb E R_T=\widetilde O\!\left(
 \min\!\left\{T,\,
 \left[(\varrho^{1/3}+\varrho^{-1/6})(\kappa_X^2D)^{1/3}+1\right]
 T^{2/3}\right\}\right).
 \label{eq:nonmonotone-schedule-sensitivity}
\]
As in Equation~\eqref{eq:nonmonotone-rate}, the model quantities $A,B,L_f,M,\gamma,\sigma,K,h$ are held fixed and logarithms in dimensions and $T$ are suppressed; dependence on $\varrho,D,\kappa_X$ is explicit. A zero parameter still has zero regret.
\end{mcorollary}

\noindent\textbf{Proof.}
In the estimation branch, $\max\{2,a\}\leq n_0<a+2$. Thus the exploration cost contributes at most $\Delta(a+2)$, and
\[
 T\kappa_X\sqrt{D/n_0}
 \leq \varrho^{-1/6}(\kappa_X^2D)^{1/3}T^{2/3}.
\]
The general bound above Equation~\eqref{eq:nonmonotone-rate}, with the unchanged bin settings, therefore gives the displayed rate: exploration contributes $\varrho^{1/3}$, estimation contributes $\varrho^{-1/6}$, and the scalar-learning and rounding terms are absorbed by the $+1$. The deterministic cap $R_T\leq\Delta T$ supplies the minimum with $T$. If $n_0=T$, use that direct cap without constructing an estimator or bins. Either $T\leq3$, or $2\lceil a/2\rceil\geq T$ implies $a>T-2\geq T/2$. Hence the linear cost is absorbed by the same rate, with the $+1$ covering the small horizons. This proves both branches.

For fixed $\varrho>0$ and fixed model quantities, schedule misspecification changes the prefactor but preserves the $T^{2/3}$ horizon exponent. This is not uniform over factors that deteriorate with $T$: a bounded exploration count can leave an order-$T$ estimation term. Scheduling without $\kappa_X$ corresponds to $\varrho=\kappa_X^{-2}$ and gives leading factor $(1+\kappa_X)D^{1/3}$. Scheduling without either input corresponds to $\varrho=(\kappa_X^2D)^{-1}$ and gives $1+\kappa_X\sqrt D$, in each case with the existing binning contribution and linear cap. This corollary does not relax the input assumptions of Theorem~\ref{thm:nonmonotone-main}.

The schedule uses $D$ and $\kappa_X$ only to choose $n_0$; the fits need no rank caps. The analogous vectorized analysis has leading factor $(\kappa_X^2N)^{1/3}$ in place of $(\kappa_X^2D)^{1/3}$, a ratio of $(D/N)^{1/3}$ under the same certificate and comparable model constants. With the universal $\kappa_X=\sqrt N$, the structured rate is $\widetilde O(\min\{T,[(ND)^{1/3}+1]T^{2/3}\})$. These comparisons concern upper bounds, and the exponent $2/3$ treats the dimensions, $\kappa_X$ and the model and signal constants as fixed. The bounds do not cover links with zero base first-order signal.

\subsection{A Fixed-Menu Lower Bound for Arbitrary Links}
\label{app:nonmonotone-lower}

The lower bound below fixes the structured parameter and varies only the link. It holds for every candidate law satisfying the local assumptions of Appendix~\ref{app:nonmonotone-model}; the proof uses only the support bound, the zero-extension $W^{1,1}$ regularity and the score certificate $M$, through the projection-density bound~\eqref{eq:projection-peak}. It therefore matches the horizon exponent of Theorem~\ref{thm:nonmonotone-main} for each design at fixed dimensions, menu size, and model/design constants, though not its dependence on $\kappa_X$, $D$, or the law of the disclosed index.

\begin{mtheorem}[Fixed-$K$ horizon lower bound]
\label{thm:nonmonotone-lower}
Fix finite dimensions, $K\geq3$, $A,B,L_f,\sigma>0$, and $0<\gamma<L_f$, all independent of $T$. Let $p$ satisfy the local assumptions of Appendix~\ref{app:nonmonotone-model}. Let $\mathcal C$ be any of the parameter classes in Section~\ref{sec:notation-structures} containing a nonzero parameter, and fix $0\ne\Theta_0\in\mathcal C$. There is a fixed class $\mathcal F_0=\mathcal F_0(p,\Theta_0)$ of $C^\infty$ links whose derivatives change sign on the actual index support, such that every $f\in\mathcal F_0$ is $L_f$-Lipschitz on $[-AB,AB]$ and satisfies
\[
 \mathbb E_{p}f'(\inner X{\Theta_0})\geq\gamma.
\]
For rewards $y_t=f(\inner{X_t}{\Theta_0})+\eta_t$ with independent $\mathcal N(0,\sigma^2)$ noise, there exist $c>0$ and an integer $T_0\geq1$, depending only on the fixed quantities above, $p$ and $\Theta_0$, such that
\begin{equation}
 \inf_\pi\sup_{f\in\mathcal F_0}
       \mathbb E_{\Theta_0,f,\pi}R_T\geq cT^{2/3}
       \qquad(T\geq T_0).
 \label{eq:nonmonotone-lower-rate}
\end{equation}
The policies may know $\Theta_0$, the candidate law, and the link class. The class can additionally satisfy $\inf_{v:\frobnorm v=1\text{ or }v=0} \mathbb E_{p^v}f'(\inner X{\Theta_0})\geq\gamma$, where $p^v$ denotes the law in Equation~\eqref{eq:selected-score} for a frozen direction $v$, and $p^0=p$. No uniform bound on higher derivatives is imposed.
\end{mtheorem}

\paragraph{The projected density.}
Every class of Section~\ref{sec:notation-structures} with a nonzero cap contains a nonzero parameter: a coordinate spike for matrices, or the pull-back through $\mathcal T_U^{-1}$ of a single nonzero transformed entry for transformed slices, scaled to meet the budget. The proof applies to any fixed $0\ne\Theta_0\in\mathcal C$. Write $\theta_0=\vecop(\Theta_0)$, $r_0=\frobnorm{\Theta_0}>0$, $b_0=A\entone{\Theta_0}\leq AB$, and $Z=\inner X{\Theta_0}$, so $|Z|\leq b_0$; let $b^*=\operatorname{ess\,sup}Z\leq b_0$. By Equation~\eqref{eq:projection-peak}, applied to the unit direction $\theta_0/r_0$ and rescaled as in Lemma~\ref{lem:lower-menu}, $Z$ has a continuous density $p_Z$ with
\[
 0<p_+:=\norm{p_Z}_\infty\leq\frac{\sqrt M}{2r_0}.
\]
Since $Z$ has no atoms, $\Pr(Z\leq u)<1$ for every $u<b^*$ and $\Pr(Z\leq u)\uparrow1$ as $u\uparrow b^*$. For every $u<b^*$, $\Pr(Z>u)>0$, so continuity gives a point $z^*>u$ with $p_Z(z^*)>0$ and some $\eta>0$ such that $z^*-2\eta>u$ and $p_Z\geq p_Z(z^*)/2$ on $[z^*-2\eta,z^*+2\eta]$. This interval lies in $(u,b^*]$, since $\Pr(Z>b^*)=0$. Under the product-beta law~\eqref{eq:beta-density}, $p_Z>0$ on $(-b_0,b_0)$, so $b^*=b_0$, and $p_+\leq315/(256A\max_i|\theta_{0,i}|)$.

\paragraph{A flat baseline, signed bumps, and an anchor.}
Choose a slope $d_0\in(\gamma,L_f)$, then a point $u_0<b^*$ such that $d_0\alpha^K>\gamma$, where $\alpha=\Pr(Z\leq u_0)$. This is possible because $\Pr(Z\leq u)\uparrow1$ as $u\uparrow b^*$. Then choose $u_1\in(u_0,b^*)$. Choose a smooth cutoff $\zeta:\mathbb R\to[0,1]$ equal to one on $[-AB,u_0]$ and zero on $[u_1,AB]$, and define the baseline $\mathfrak b(z)=d_0\int_{u_1}^z\zeta(s)\,ds$. The cutoff can have compact support and remain zero for $z\geq u_1$. Thus $\mathfrak b'=d_0$ on $[-AB,u_0]$, $\mathfrak b'\geq0$, and $\mathfrak b=0$ on $[u_1,AB]$.

Apply the preceding paragraph with $u=u_1$, and set $\mathcal I_{\rm b}=[z^*-2\eta,z^*-\eta]$ for bumps and $\mathcal I_{\rm a}=[z^*+\eta,z^*+2\eta]$ for anchors. These disjoint nondegenerate intervals lie in $(u_1,b^*]$, where the baseline is zero, and
\[
 \ell_{\rm b}=|\mathcal I_{\rm b}|=\eta,\qquad
 p_- =\min_{z\in\mathcal I_{\rm b}}p_Z(z)\geq\frac{p_Z(z^*)}2>0,\qquad
 q_{\rm a}=\Pr(Z\in\mathcal I_{\rm a})\geq\frac{\eta\,p_Z(z^*)}2>0.
\]
Take a fixed smooth bump $\varphi$ with $0\leq\varphi\leq1$, supported on $[-1/2,1/2]$, and equal to one on $[-1/4,1/4]$. Choose $a_0>0$ with $a_0\norm{\varphi'}_\infty\leq L_f$. For $0<w\leq\ell_{\rm b}/2$, place $m_w=\lfloor\ell_{\rm b}/w\rfloor\geq\ell_{\rm b}/(2w)$ intervals $I_j$ of length $w$ with disjoint interiors in $\mathcal I_{\rm b}$, with centers $z_j$ and cores $J_j=[z_j-w/4,z_j+w/4]$. For $\omega\in\{-1,1\}^{m_w}$, set
\[
 B_{\omega,w}(z)=a_0w\sum_{j=1}^{m_w}\omega_j
          \varphi\bigl((z-z_j)/w\bigr),\qquad
 f_{\omega,w}(z)=\mathfrak b(z)+B_{\omega,w}(z).
 \label{eq:nonmonotone-bumps}
\]
At most one bump is nonzero at each point, and the baseline is flat throughout their intervals, so $\norm{B_{\omega,w}}_\infty\leq a_0w$ and $\norm{f_{\omega,w}'}_\infty\leq\max\{d_0, a_0\norm{\varphi'}_\infty\}\leq L_f$. Each link has both derivative signs on the actual support, since the baseline is flat on every nonzero signed bump. Its value on $J_j$ is $\omega_ja_0w$, whereas its value on the anchor interval is zero. For any fixed upper cutoff on $w$, the links are uniformly bounded on $[-AB,AB]$. Their higher derivatives need not be uniformly bounded as $w\downarrow0$.

\paragraph{Preserving the signal without normalizing the parameter.}
For any projection-greedy law $p^v$, including $p^0=p$, let $S^{p^v}$ be its full joint score. Weak integration by parts gives
\[
 \mathbb E_{p^v} B_{\omega,w}'(Z)
 =\frac1{r_0^2}\mathbb E_{p^v}
       [B_{\omega,w}(Z)\inner{S^{p^v}(X)}{\Theta_0}].
 \label{eq:nonmonotone-directional-stein}
\]
Indeed, the directional derivative along $\theta_0$ of $B_{\omega,w}(\theta_0^\top x)$ is $r_0^2B_{\omega,w}'(Z)$. The zero-extended density is in $W^{1,1}$, and a smooth ambient cutoff equal to one on the support justifies integration by parts. Thus no boundary condition on individual projected fibers is being assumed. The joint score certificates in Appendix~\ref{app:score-moments} give
\[
 |\mathbb E_{p}B_{\omega,w}'(Z)|
 \leq\frac{a_0w\sqrt{M}}{r_0},\qquad
 \sup_v|\mathbb E_{p^v}B_{\omega,w}'(Z)|
 \leq\frac{a_0w\sqrt{\overline J}}{r_0},\qquad
 \overline J=MG_K(1).
\]
The factor $r_0^{-1}$ is required when the revealed parameter is not unit-Frobenius. On the event that all $K$ candidate indices satisfy $Z\leq u_0$, every selection rule chooses such an index; this event has probability $\alpha^K$. Since the baseline derivative is nonnegative, $\mathbb E_{p^v}\mathfrak b'(Z)\geq d_0\alpha^K$ uniformly in $v$, including $v=0$. Set
\[
 w_0=\min\left\{\ell_{\rm b}/2,\ 1,\,
        \frac{(d_0\alpha^K-\gamma)r_0}{a_0\sqrt{\overline J}}\right\}>0.
 \label{eq:nonmonotone-bump-cutoff}
\]
For every $w\leq w_0$ and sign vector, $\mathbb E_{p^v}f_{\omega,w}'(Z)\geq\gamma$ uniformly in $v$. In particular the base signal is positive. This stronger conclusion concerns projection-greedy laws; it does not assert a signal floor under every nonlinear selection policy. Define the fixed class $\mathcal F_0=\{f_{\omega,w}:0<w\leq w_0, \ \omega\in\{-1,1\}^{m_w}\}$.

\paragraph{Adaptive information is limited by menu exposure.}
Let $H_t$ contain the observed menus, actions, rewards, and private policy randomness through round $t$. For neighboring signs $\omega$ and $\omega^{(j)}$, the menu laws and conditional decision kernels agree at each common history. Only the selected Gaussian reward contributes to their relative entropy. Writing $Z_{t,a}=\inner{X_{t,a}}{\Theta_0}$ and $Z_t=\inner{X_t}{\Theta_0}$, the chain rule therefore gives
\begin{equation}
 \begin{aligned}
 \operatorname{KL}(P_\omega^{H_T}\|P_{\omega^{(j)}}^{H_T})
 &=\frac1{2\sigma^2}\mathbb E_\omega\sum_{t=1}^T
    \{f_{\omega,w}(Z_t)-f_{\omega^{(j)},w}(Z_t)\}^2\\
 &\leq\frac{2a_0^2w^2}{\sigma^2}
           \sum_{t=1}^T\Pr_\omega(Z_t\in I_j)
 \leq\frac{2a_0^2Kp_+}{\sigma^2}Tw^3.
 \end{aligned}
 \label{eq:nonmonotone-exposure-kl}
\end{equation}
The last inequality holds for every adaptive strategy: it can select a candidate in $I_j$ only when its current menu contains one, an event of probability at most $Kp_+w$. Observing all the other coordinates does not change this argument. Take
\[
 c_0=\left(\frac{\sigma^2}{16a_0^2Kp_+}\right)^{1/3},\qquad
 w=c_0T^{-1/3},\qquad
 T_0=\max\{1,\lceil(c_0/w_0)^3\rceil\}.
\]
For $T\geq T_0$, every constructed link belongs to $\mathcal F_0$ and the KL in Equation~\eqref{eq:nonmonotone-exposure-kl} is at most $1/8$; replacing $p_+$ by its bound $\sqrt M/(2r_0)$ in $c_0$ preserves this. Marginalization and Pinsker's inequality imply $\operatorname{TV}(P_\omega^{H_{t-1}}, P_{\omega^{(j)}}^{H_{t-1}})\leq1/4$ at every round.

\paragraph{Testing conditional on the current informative menu.}
Let $E_{t,j}$ be the event that precisely one candidate index is in $J_j$ and all other indices are in $\mathcal I_{\rm a}$. These events are disjoint over $j$, and
\[
 \Pr(E_{t,j})=K\Pr(Z\in J_j)q_{\rm a}^{K-1}
 \geq c_Ew,\qquad c_E=Kp_-q_{\rm a}^{K-1}/2>0.
\]
There is at least one anchor because $K\geq3$. Fix signs $\omega^+,\omega^-$ differing only at $j$, with that sign positive and negative respectively. For each fixed full menu $\mathcal X\in E_{t,j}$, let $\pi_t(H,\mathcal X)\in[0,1]$ be the policy's probability of selecting the unique bump candidate. Current menus are independent of past histories and have the same law under both environments. Hence
\[
 \mathbb E_{\omega^+}[1-\pi_t(H_{t-1},\mathcal X)]
 +\mathbb E_{\omega^-}[\pi_t(H_{t-1},\mathcal X)]
 \geq1-\operatorname{TV}(P_{\omega^+}^{H_{t-1}},
                         P_{\omega^-}^{H_{t-1}})\geq3/4.
\]
On this menu, a wrong choice has regret $a_0w$: the bump is better than the anchors under $\omega^+$ and worse under $\omega^-$. Integrating the last inequality over these current menus gives
\[
 \mathbb E_{\omega^+}[r_t\mathbf1_{E_{t,j}}]
 +\mathbb E_{\omega^-}[r_t\mathbf1_{E_{t,j}}]
 \geq\tfrac34a_0w\Pr(E_{t,j})\geq\tfrac34a_0c_Ew^2.
\]
The event probability multiplies a constant testing error conditional on the current menu. Subtracting a constant unconditional total variation bound from the rare-event probability would not establish this result.

Average over the $2^{m_w}$ sign patterns, pairing neighbors for each $j$. Nonnegative regret and disjointness of the informative-menu events yield
\[
 \begin{aligned}
 2^{-m_w}\sum_\omega\mathbb E_\omega R_T
 &\geq\sum_{t=1}^T\sum_{j=1}^{m_w}
             2^{-m_w}\sum_\omega\mathbb E_\omega
                              [r_t\mathbf1_{E_{t,j}}]\\
 &\geq\tfrac38a_0c_Em_wTw^2
 \geq\tfrac3{16}a_0c_E\ell_{\rm b}c_0T^{2/3}.
 \end{aligned}
\]
This proves Equation~\eqref{eq:nonmonotone-lower-rate} with $c=3a_0c_E\ell_{\rm b}c_0/16$. All choices defining $c$ and $T_0$ depend only on the fixed problem quantities, the law $p$ and the fixed parameter; they depend on $p$ through $p_+$, $p_-$, $q_{\rm a}$, $\ell_{\rm b}$ and $\alpha$. Although the hardest link may depend on $T$, the class $\mathcal F_0$ does not. Revealing $\Theta_0$ only makes learning easier, so for the same $p$ the lower bound applies to the full model with an unknown structured index.

The bound matches the horizon exponent only, for fixed dimensions, $K$, and positive noise and signal constants; $c_E$ contains $q_{\rm a}^{K-1}$, so it gives no uniform statement in $K$. A zero-parameter class or $K=1$ has zero regret. The constructed links are individually smooth, but the class has no common bound on higher derivatives.

\section{Experimental details}\label{app:experiments}

\subsection{Dense synthetic design and six links}
\label{app:nonmonotone-experiment}
Let $a_1\in\{\pm1\}^{16}$ and $b_1\in\{\pm1\}^{64}$ have independent uniform signs. Set $a_2=a_1\odot q_a$, where $q_a$ is a uniformly permuted balanced sign vector, and $b_2=b_1\odot q_b$, where each of the four length-16 blocks of $q_b$ is independently balanced and permuted. Define
\[
 M_*=\frac{a_1b_1^\top+\tfrac12a_2b_2^\top}{1024},\qquad
 \Theta_{*,:, :,k}=M_{*,:,16(k-1)+1:16k},\quad k=1,\ldots,4.
\]
The two components are orthogonal globally and within each frontal slice. Every entry is nonzero, $\|\Theta_*\|_1=1$, and $\|\Theta_*\|_F^2=1.25/1024$. The matrix singular values are $1/32$ and $1/64$; each slice also has rank two and a $2:1$ singular-value ratio. The slice transform is $U=I_4$. The matrix scale $2\max\{16,64\}$ and slice scale $8\max\{16,16\}$ are both 128, so this family does not imply a matrix-versus-slice separation. Matrices horizontally concatenate frontal slices without interleaving their entries. All three representations preserve scalar inner products exactly.

Each round supplies $K=10$ independent candidate arrays with independent entries $X_j=2\operatorname{Beta}(5,5)-1$. The covariance and score moments are $\mathbb E[\operatorname{vec}(X)\operatorname{vec}(X)^\top]=I/11$ and $\mathbb E[\operatorname{vec}(S(X))\operatorname{vec}(S(X))^\top]=12I$, where $S_j(x)=8x_j/(1-x_j^2)$. The candidate bound is $A=1$ and coefficient budget is $B=1$. Write $s_0=\sqrt{1.25/(11\cdot1024)}\simeq0.01053437$ and $v=u/s_0$. The six fixed links are
\begin{align*}
 f_{\rm linear}(u)&=u,&
 f_{\rm valley}(u)&=\tfrac u2+s_0(\sqrt{1+v^2}-1),\\
 f_{\rm tanh}(u)&=u+\tfrac{s_0}{2}\tanh(v),&
 f_{\rm peak}(u)&=\tfrac u2-s_0(\sqrt{1+v^2}-1),\\
 f_{\rm convex}(u)&=u+\tfrac{s_0}{2}\log\cosh(v),&
 f_{\rm wave}(u)&=\tfrac u2+\tfrac{s_0}{2}(\cos(2v)-1).
\end{align*}
All vanish at zero, are $1.5$-Lipschitz, and are bounded in magnitude by 1.5 on $[-1,1]$. The first three are strictly increasing. Symmetry gives base Stein multipliers $1$, a value in $[1,1.5]$, $1$, $1/2$, $1/2$, and $1/2$, respectively. Valley and peak turn at $-s_0/\sqrt3$ and $s_0/\sqrt3$; wave derivatives vanish at standardized indices $\pi/12+k\pi$ and $5\pi/12+k\pi$. These turns occur on the index's standard-deviation scale. Noise is Student-$t_3$ multiplied by $s_0/(2\sqrt3)$, giving standard deviation $s_0/2$ and finite variance.

Coefficient, context, potential-noise, and uniform-action streams are independent. Within each seed, all six links and methods share the coefficient, candidate menus, potential-noise realizations, and uniform choices. Policies receive only their candidate features and chosen rewards. Hidden coefficients, true means, derivatives, and signal checks stay in the simulator. Regret is
\[
 R_T=\sum_{t=1}^T\left[\max_{a\le K}f(\langle X_{t,a},\Theta_*\rangle)
               -f(\langle X_{t,A_t},\Theta_*\rangle)\right],
\]
including exploration, bin initialization, and fallback rounds.

\subsection{Practical adaptive T-BSTOR}
The three T-BSTOR variants use the same uniform exploration and robust scalar learner, with radial vector influence or spectral matrix/slice influence in their paired base-score fit. At checkpoints $n=100,200,\ldots$, each fit uses all $n/2$ consecutive observation pairs and recomputes its sample-size-dependent influence and threshold parameters. The reconstructed estimate $\widehat\Theta_n^0$ is projected onto the original-coordinate entrywise $\ell_1$ ball of radius $L_fB$, with $L_f=1.5$, to give $\widehat\Theta_n$.

For numerically valid nonzero fits, form $v_n=\operatorname{vec}(\widehat\Theta_n)/\|\widehat\Theta_n\|_F$ and compare $\min\{\|v_n-v_{n-100}\|_2,\|v_n+v_{n-100}\|_2\}$ with the tuned tolerance. Exploration stops after two consecutive successful comparisons and at least 400 observations. An invalid or zero fit resets the success count and breaks successive-fit comparisons. The cap is $n_{\max}=2\lfloor0.4T/2\rfloor$; fitting at this even cap occurs even when it is not a multiple of 100. A zero/invalid capped fit triggers uniform selection for the remaining horizon. This stability rule is a practical stopping heuristic, not a certificate of estimator accuracy.

At the final fit freeze $s=\|\widehat\Theta\|_F$ and $v=\operatorname{vec}(\widehat\Theta)/s$. Scalar decisions use $\langle x,v\rangle$; the unit direction is not projected again onto an $\ell_1$ ball. Divide $[-R,R]$ into $J_{\rm bin}=\lceil c_JT^{1/3}\rceil$ equal bins, where
\[
 R=\sqrt{\frac2{11}\log\frac{2KT}{\delta_R}},\qquad
 \delta_R=(T+1)^{-2}.
\]
The registered $c_J\in\{0.5,1,2\}$ gives $J_{\rm bin}=21,41,81$ at the final synthetic horizon. The right endpoint belongs to the last bin. Indices outside the interval map to its edge bins; candidate and selected-action overflows are counted. Only chosen rewards collected after exploration initialize or update the bins. The minimum odd block count, consecutive equal-size median-of-means blocks, confidence allocation, initialization priority among available bins, and tie breaking are shared across all three representations. The radius is $4c_{\rm conf}\sigma\sqrt{k_{\rm blk}/m}$ for $m$ observations in a bin. Tuned confidence multipliers do not remove mandatory bin initialization.

\paragraph{Index range.}
After normalization the bin-approximation term scales with the fitted norm $s$: if $g$ is $G$-Lipschitz, $z\mapsto g(sz)$ is $Gs$-Lipschitz, giving $Gs(2R/J_{\rm bin})$. Each beta coordinate has moment-generating function at most $\exp(\lambda^2/22)$, so every unit projection of a fresh candidate has this bound. Conditional on the exploration stopping history and frozen direction, future menus remain independent. A union bound over at most $KT$ candidates keeps indices in $[-R,R]$ except with probability $\delta_R$, costing at most $\Delta T\delta_R$ in expectation when gaps are bounded by $\Delta$. This controls the index range only; extending the deterministic-exploration theorem to adaptive stopping would require a simultaneous estimation argument.

\subsection{Other algorithms and numerical implementations}
\label{app:numerical-oracle}
T-ESTOR retains doubling epochs $2,4,8,\ldots$, latest-epoch paired spectral estimation, and projection-greedy selection. A frozen direction $v$ determines the projection CDF $F_v$ and density $F_v'$ under the product-beta base law, yielding the exact selected score
\[
 S^{p_v}(X)=S^p(X)-(K-1)
 \frac{F_v'(\inner Xv)}{F_v(\inner Xv)}v.
\]
The numerical implementation uses double-precision Fourier convolution with all nonzero direction coefficients, 4,096 grid points, a 5\% support margin, and CDF floor $10^{-12}$. For multicoordinate directions, construction multiplies the coordinate characteristic functions, inverts the Fourier transform, normalizes the density, and integrates it to obtain the CDF. Subsequent scores project each action, interpolate the density and CDF, apply the floored ratio, and subtract the directional correction from the explicit beta base score. Oracle construction and evaluation are separate from the paired robust fits and their SVDs. A zero fit uses uniform play and base scores. T-ESTOR is evaluated only on monotone links. The numerical scores approximate the theorem's exact oracle. Grid comparisons establish neither uniform score-error certificates nor the selected-score and epoch conditions of Appendix~\ref{app:unknown-greedy}.

\paragraph{Oracle accuracy and cost.}
For three evaluation seeds and the three monotone links, we use the first, middle, and last epochs with nonzero incoming fits: 27 directions. Against a 16,384-point numerical reference, the 4,096-point oracle's maximum discrepancies at selected actions are $3.8\times10^{-4}$ in density, $1.3\times10^{-4}$ in CDF, and $5.9\times10^{-3}$ in score correction (Table~\ref{tab:oracle-diagnostics}). Tail discrepancies are larger: all three grids activate the floor on 37 of 459 tail queries, but none of 5,284 selected queries; nine full evaluation runs have no activations in 587,668 nonzero-fit queries. These checks do not certify uniform error.

\begin{table}[t]
\centering\small
\setlength{\tabcolsep}{3pt}
\caption{Oracle discrepancies against a 16,384-point numerical reference over all 27 frozen directions. Actual: the first up to 256 chosen projections per selected epoch; tails: 17 points from $-8/\sqrt{11}$ to $8/\sqrt{11}$. Entries are absolute; $g_v=F_v'$ and $c=(K-1)g_v/\max\{F_v,10^{-12}\}$. Floor counts refer to the compared grid.}
\label{tab:oracle-diagnostics}
\begin{tabular}{rlrrrrr}
\toprule
Grid & Queries & Max $|\Delta g|$ & Max $|\Delta F|$ & 95th $|\Delta c|$ & Max $|\Delta c|$ & Floors \\
\midrule
4,096 & Actual & $3.80\!\times\!10^{-4}$ & $1.26\!\times\!10^{-4}$ & $1.81\!\times\!10^{-3}$ & $5.85\!\times\!10^{-3}$ & 0/5284 \\
4,096 & Tails & $1.66\!\times\!10^{-4}$ & $1.17\!\times\!10^{-4}$ & $1.16\!\times\!10^{0}$ & $1.65\!\times\!10^{0}$ & 37/459 \\
8,192 & Actual & $1.03\!\times\!10^{-4}$ & $2.88\!\times\!10^{-5}$ & $4.11\!\times\!10^{-4}$ & $1.72\!\times\!10^{-3}$ & 0/5284 \\
8,192 & Tails & $4.04\!\times\!10^{-5}$ & $2.77\!\times\!10^{-5}$ & $2.38\!\times\!10^{-1}$ & $3.27\!\times\!10^{-1}$ & 37/459 \\
\bottomrule
\end{tabular}
\end{table}

On an Intel i7-1065G7 shared desktop, sequential double-precision diagnostics with one BLAS/OpenMP thread give median construction cost 1.52\,s per direction and correction evaluation cost $0.344\,\mu$s per query from precomputed 256-projection batches. Construction is fresh and uncached after one warm-up; evaluation uses five warm-ups and seven repeats of 100 batches. Both exclude reconstruction, feature projection, base scores, and robust fitting/SVD. Construction accounts for 43.1\% of pooled online time (31.6--51.6\% per run); the combined score timer does not isolate correction evaluation.

ZoomSIB-UCB follows Algorithm~1 and Appendix~G of \citet{dey2026optimal}: clipped reward-score exploration, entrywise $\ell_1$ normalization, adaptive normalized-estimate comparisons every 100 rounds, and sample-mean UCB over available bins. Its minimum is 200 observations with two checkpoints; its cap is $0.4T$. Its tuned stability criterion is unchanged and differs from the two-success Frobenius rule used by the T-BSTOR variants. At each checkpoint clipping is recomputed using all observations and the current sample size, an explicit implementation choice where the adaptive presentation is underspecified. The frozen direction uses the published scalar procedure; exploration rewards do not initialize bins. Its synthetic index window is $[-1,1]$.

GSTOR follows \citet{kang2026single}, including Appendix~K: two independent uniform stages of $T_1=\lceil c_nT^{3/4}\rceil$ observations, each capped at $\lfloor(T-1)/2\rfloor$. The first fits coordinatewise clipped reward-score products, optionally soft-thresholded by the declared penalty; the second fits the uniform-kernel predictor with bandwidth $c_hT_1^{-1/3}$. The coefficient normalizer is $\|\Sigma^{1/2}\widehat\theta\|_1$, with $\Sigma^{1/2}=I/\sqrt{11}$ and synthetic window $[-\sqrt{11},\sqrt{11}]$. Both estimators remain frozen during greedy commitment. Sorted indices and prefix sums evaluate the same closed kernel windows exactly. Empty and out-of-window predictions equal zero with first-candidate ties, including for negative rewards. The feasible set is the full vector space; selected nonzero penalties are identified as regularized GSTOR implementations. None of these baselines receives a true coefficient or link. The beta design differs from GSTOR's Gaussian theorem setting, and Student-$t_3$ noise is not sub-Gaussian as assumed by ZoomSIB-UCB; this comparison does not assert all external guarantees under the common environment.

\subsection{Registered tuning and reporting}
All arbitrary-link methods have 48 distinct configurations. For matrix and slice T-BSTOR, cross stability tolerances $\{0.05,0.1\}$, bin-count multipliers $\{0.5,1,2\}$, confidence multipliers $\{1/30,0.1\}$, and thresholds $c_\tau\in\{10^{-4},3\cdot10^{-4},10^{-3},3\cdot10^{-3}\}$, fixing the influence multiplier $c_\nu=16$. T-BSTOR-Vec has no low-rank threshold and instead crosses $c_\nu\in\{1,16,64,256\}$ with the same first three sets. ZoomSIB-UCB crosses tolerances $\{0.025,0.05,0.1\}$, width multipliers $\{0.0625,0.125,0.5,1\}$, confidence multipliers $\{0.25,1\}$, and truncation multipliers $\{0.1,1\}$. GSTOR retains the previous 24 settings with $c_n\in\{0.125,0.5,2\}$, bandwidth multipliers $\{0.25,1\}$, penalties $\{0,0.005\}$, and truncation multipliers $\{0.1,1\}$, adding 24 that use bandwidth multipliers $\{0.0625,0.125\}$ and penalties $\{0,0.0001\}$ with the same other sets. T-ESTOR retains its 24 settings crossing $c_\nu\in\{1,16,256\}$ with $c_\tau\in\{0.0003,0.001,0.003,0.01,0.03,0.1,0.3,1\}$.

Synthetic screening uses three seeds at $T=8{,}190$ and advances eight configurations per method and link. Second-stage validation uses ten separate seeds at $T=32{,}766$; final evaluation uses 50 fresh paired seeds at $T=65{,}534$. Each selection minimizes mean final regret, with prespecified configuration order breaking ties. The monotone method screens fewer configurations and has the same second-stage budget. All choices freeze before evaluation; held-out results do not change grids or selected settings.

Figures use 10,000 paired whole-seed percentile bootstrap draws. Bands are pointwise 95\% intervals, not simultaneous bands. Final-regret contrasts resample the same paired seedwise differences; these intervals are unadjusted for multiple comparisons. Direction errors average the sign-invariant distance $\min\{\|v-v_*\|_2,\|v+v_*\|_2\}$ over valid nonzero fits, where $v_*=\operatorname{vec}(\Theta_*)/\|\Theta_*\|_F$. They are undefined for zero fits and do not certify T-ESTOR's orientation. No structural rank is assigned to the trivial vector matrix encoding. Tables~\ref{tab:comparison-v2-synthetic-diagnostics} and~\ref{tab:comparison-v2-synthetic-fit-diagnostics} summarize exploration, initialization, bins, and fits. Diagnostics include stopping times, phase regret, within-selected-bin loss, norms, pre/post-projection ranks, overflows, and fallbacks; baseline phase summaries do not separately identify bin initialization.

Online times exclude shared environment construction and offline preparation/calibration; cached fits/projections and parallel runs make these operational timings, not hardware-independent speed comparisons. In an earlier signed-diagonal design, no arbitrary-link T-BSTOR variant had the lowest mean regret in any panel: monotone means were 1,645--1,658, versus 813 for ZoomSIB-UCB and 127 for T-ESTOR; GSTOR had the lowest means on the old nonmonotone panel. These descriptive means are not direct improvement estimates: the design, signal/noise scale, and practical algorithm changed. The revision was designed to test whether dense low-rank structure and the revised practical implementation yield a finite-horizon advantage.

\subsection{Synthetic settings and diagnostic summaries}
Table~\ref{tab:comparison-v2-synthetic-settings} lists the selected hyperparameters for every method and link. Table~\ref{tab:comparison-v2-synthetic-paired} gives the prespecified paired contrasts; Tables~\ref{tab:comparison-v2-synthetic-diagnostics} and~\ref{tab:comparison-v2-synthetic-fit-diagnostics} report exploration, initialization, and fit diagnostics.

Both structured T-BSTOR variants have smaller mean sign-invariant direction errors than T-BSTOR-Vec on each link: their means range from 0.413 to 0.626, compared with 0.661 to 0.804 for vectorization. Their mean later-scalar regret and within-selected-bin loss are also lower on every link. Initialization still requires about 1,100--2,072 rounds. These descriptive totals span different numbers of scalar rounds; their magnitudes alone do not compare per-round prediction quality or identify a sole cause of the regret gap. The fitted ranks are estimated rather than supplied from the true tensor, and the direction diagnostic does not certify the stopping rule.

% Permit ordinary placement without changing the generated table sources.
\begingroup
\let\savedexperimenttable\table
\let\endsavedexperimenttable\endtable
\renewenvironment{table}[1][]{\savedexperimenttable[!htbp]}{\endsavedexperimenttable}
\begin{table}[t]
\centering\small
\caption{Frozen synthetic configurations. Every tuned field appears in its method's tuple order. $c_\nu$ is the influence multiplier; $c_w$ and $c_h$ are width and bandwidth multipliers, $\lambda_{\rm reg}$ is the GSTOR penalty, and $c_{\rm trunc}$ is the truncation multiplier. Nonzero $\lambda_{\rm reg}$ denotes regularized GSTOR.}
\label{tab:comparison-v2-synthetic-settings}
\setlength{\tabcolsep}{5pt}
\begin{tabular}{lll}
\toprule
Panel & Policy & Parameter tuple \\
\midrule
\multicolumn{3}{l}{T-BSTOR: $(\varepsilon,c_J,c_{\rm conf},c_\nu,c_\tau)$} \\
Linear & T-BSTOR-Vec & $(0.1,\,0.5,\,\tfrac1{30},\,64,\,0)$ \\
Linear & T-BSTOR-Mat & $(0.1,\,1,\,\tfrac1{30},\,16,\,0.001)$ \\
Linear & T-BSTOR-Slice & $(0.1,\,1,\,\tfrac1{30},\,16,\,0.001)$ \\
Tanh & T-BSTOR-Vec & $(0.1,\,0.5,\,\tfrac1{30},\,256,\,0)$ \\
Tanh & T-BSTOR-Mat & $(0.1,\,1,\,\tfrac1{30},\,16,\,0.001)$ \\
Tanh & T-BSTOR-Slice & $(0.1,\,1,\,\tfrac1{30},\,16,\,0.001)$ \\
Convex & T-BSTOR-Vec & $(0.1,\,1,\,\tfrac1{30},\,256,\,0)$ \\
Convex & T-BSTOR-Mat & $(0.1,\,1,\,\tfrac1{30},\,16,\,0.001)$ \\
Convex & T-BSTOR-Slice & $(0.1,\,1,\,\tfrac1{30},\,16,\,0.001)$ \\
Valley & T-BSTOR-Vec & $(0.1,\,1,\,\tfrac1{30},\,256,\,0)$ \\
Valley & T-BSTOR-Mat & $(0.05,\,1,\,\tfrac1{30},\,16,\,0.001)$ \\
Valley & T-BSTOR-Slice & $(0.05,\,1,\,\tfrac1{30},\,16,\,0.001)$ \\
Peak & T-BSTOR-Vec & $(0.1,\,0.5,\,\tfrac1{30},\,1,\,0)$ \\
Peak & T-BSTOR-Mat & $(0.1,\,1,\,\tfrac1{30},\,16,\,0.001)$ \\
Peak & T-BSTOR-Slice & $(0.1,\,0.5,\,0.1,\,16,\,0.001)$ \\
Wave & T-BSTOR-Vec & $(0.1,\,0.5,\,0.1,\,256,\,0)$ \\
Wave & T-BSTOR-Mat & $(0.1,\,0.5,\,0.1,\,16,\,0.001)$ \\
Wave & T-BSTOR-Slice & $(0.1,\,0.5,\,\tfrac1{30},\,16,\,0.001)$ \\
\midrule
\multicolumn{3}{l}{ZoomSIB-UCB: $(\varepsilon,c_w,c_{\rm conf},c_{\rm trunc})$} \\
Linear & ZoomSIB-UCB & $(0.1,\,0.0625,\,0.25,\,0.1)$ \\
Tanh & ZoomSIB-UCB & $(0.1,\,0.0625,\,0.25,\,0.1)$ \\
Convex & ZoomSIB-UCB & $(0.1,\,0.0625,\,0.25,\,0.1)$ \\
Valley & ZoomSIB-UCB & $(0.05,\,0.0625,\,0.25,\,0.1)$ \\
Peak & ZoomSIB-UCB & $(0.1,\,0.125,\,0.25,\,0.1)$ \\
Wave & ZoomSIB-UCB & $(0.1,\,0.125,\,0.25,\,0.1)$ \\
\midrule
\multicolumn{3}{l}{GSTOR: $(c_n,c_h,\lambda_{\rm reg},c_{\rm trunc})$} \\
Linear & GSTOR & $(2,\,1,\,0,\,0.1)$ \\
Tanh & GSTOR & $(0.5,\,0.25,\,0,\,0.1)$ \\
Convex & GSTOR & $(2,\,1,\,0,\,0.1)$ \\
Valley & GSTOR & $(2,\,1,\,0,\,0.1)$ \\
Peak & GSTOR & $(0.5,\,1,\,0,\,0.1)$ \\
Wave & GSTOR & $(2,\,1,\,0,\,0.1)$ \\
\midrule
\multicolumn{3}{l}{T-ESTOR (monotonicity-aware): $(c_\nu,c_\tau)$} \\
Linear & T-ESTOR & $(16,\,0.0003)$ \\
Tanh & T-ESTOR & $(1,\,0.0003)$ \\
Convex & T-ESTOR & $(256,\,0.0003)$ \\
\bottomrule
\end{tabular}
\end{table}

\begin{table}[t]
\centering\small
\caption{Paired final-regret differences: the predeclared structured reference minus each comparator. Negative values favor the reference. Brackets give unadjusted paired 95\% percentile bootstrap intervals. The reference is T-BSTOR-Slice.}
\label{tab:comparison-v2-synthetic-paired}
\begin{tabular}{lrrr}
\toprule
Panel & Versus T-BSTOR-Vec & Versus ZoomSIB-UCB & Versus GSTOR \\
\midrule
Linear & $-134.2\,[-137.6,-130.5]$ & $-17.9\,[-22.9,-13.4]$ & $-109.5\,[-112.9,-106.0]$ \\
Tanh & $-154.5\,[-158.5,-150.5]$ & $-6.8\,[-11.9,-2.0]$ & $-78.9\,[-84.6,-73.7]$ \\
Convex & $-154.9\,[-159.9,-149.8]$ & $-43.4\,[-50.2,-36.8]$ & $-126.1\,[-130.9,-121.1]$ \\
Valley & $-181.2\,[-185.0,-177.5]$ & $-42.3\,[-47.5,-37.0]$ & $-119.1\,[-123.4,-114.8]$ \\
Peak & $-35.2\,[-37.1,-33.4]$ & $-7.5\,[-10.0,-5.2]$ & $-38.6\,[-40.4,-36.9]$ \\
Wave & $-26.1\,[-27.0,-25.1]$ & $-5.2\,[-6.6,-3.7]$ & $-40.9\,[-41.8,-39.9]$ \\
\bottomrule
\end{tabular}
\end{table}

\begin{table}[t]
\centering\scriptsize
\caption{T-BSTOR diagnostics averaged over evaluation seeds. $n_0$ is exploration length; $m_{\rm init}$ counts scalar-bin initialization rounds. The three regret columns charge exploration, initialization, and later scalar rounds. $D_{\rm bin}$ sums loss relative to the best candidate in the selected bin. Overflow is the percentage of binned candidate indices outside the frozen interval. Uniform-fallback rounds are excluded from the later-scalar column and retained separately in the accompanying summary.}
\label{tab:comparison-v2-synthetic-diagnostics}
\setlength{\tabcolsep}{3pt}
\begin{tabular}{llrrrrrrr}
\toprule
Panel & Policy & $n_0$ & $m_{\rm init}$ & $R_{\rm exp}$ & $R_{\rm init}$ & $R_{\rm later}$ & $D_{\rm bin}$ & Overflow \% \\
\midrule
Linear & T-BSTOR-Vec & 4138.0 & 1102.0 & 67.2 & 20.6 & 261.6 & 145.5 & 0.000 \\
Linear & T-BSTOR-Mat & 2062.0 & 2072.3 & 33.5 & 40.4 & 112.4 & 53.6 & 0.000 \\
Linear & T-BSTOR-Slice & 2972.0 & 2069.6 & 48.2 & 40.0 & 126.9 & 56.5 & 0.000 \\
Tanh & T-BSTOR-Vec & 4018.0 & 1102.9 & 83.4 & 25.9 & 315.6 & 180.3 & 0.000 \\
Tanh & T-BSTOR-Mat & 2392.0 & 2070.5 & 49.7 & 50.2 & 120.0 & 61.3 & 0.000 \\
Tanh & T-BSTOR-Slice & 3134.0 & 2068.3 & 65.0 & 49.7 & 155.8 & 69.9 & 0.000 \\
Convex & T-BSTOR-Vec & 4198.0 & 2061.4 & 80.3 & 39.2 & 307.2 & 95.8 & 0.000 \\
Convex & T-BSTOR-Mat & 2126.0 & 2070.8 & 40.7 & 40.0 & 152.3 & 71.3 & 0.000 \\
Convex & T-BSTOR-Slice & 3054.0 & 2067.8 & 58.3 & 39.8 & 173.6 & 75.3 & 0.000 \\
Valley & T-BSTOR-Vec & 5002.0 & 2060.8 & 67.7 & 22.3 & 340.2 & 82.4 & 0.000 \\
Valley & T-BSTOR-Mat & 4920.0 & 2061.1 & 66.5 & 20.0 & 154.1 & 62.3 & 0.000 \\
Valley & T-BSTOR-Slice & 6914.0 & 2053.5 & 93.4 & 19.8 & 135.7 & 57.7 & 0.000 \\
Peak & T-BSTOR-Vec & 5102.0 & 1099.2 & 25.7 & 11.7 & 103.8 & 48.6 & 0.000 \\
Peak & T-BSTOR-Mat & 2590.0 & 2067.6 & 13.1 & 25.0 & 68.1 & 25.0 & 0.000 \\
Peak & T-BSTOR-Slice & 3632.0 & 1105.0 & 18.3 & 13.8 & 73.8 & 46.0 & 0.000 \\
Wave & T-BSTOR-Vec & 5140.0 & 1098.8 & 33.3 & 9.4 & 139.6 & 54.0 & 0.000 \\
Wave & T-BSTOR-Mat & 2518.0 & 1106.8 & 16.3 & 9.9 & 124.7 & 46.6 & 0.000 \\
Wave & T-BSTOR-Slice & 3512.0 & 1104.8 & 22.7 & 9.8 & 123.7 & 46.3 & 0.000 \\
\bottomrule
\end{tabular}
\end{table}

\begin{table}[t]
\centering\small
\caption{Frozen-fit and bin diagnostics for T-BSTOR policies. $s$ is the fitted Frobenius norm before unit normalization. $r_{\rm th}$ and $r_{\rm pr}$ are the rank after thresholding and after entrywise projection: matrix rank or sum of slice ranks, omitted for vectors. Bins counts occupied scalar bins. Numeric summaries average observed runs; zero and fallback columns count final zero fits and uniform-fallback runs out of all evaluation seeds. Dashes denote unavailable or inapplicable quantities. $d$ is sign-invariant direction error averaged over valid nonzero fits; parentheses give their count.}
\label{tab:comparison-v2-synthetic-fit-diagnostics}
\setlength{\tabcolsep}{3pt}
\begin{tabular}{llrrrrrrr}
\toprule
Panel & Policy & $s$ & $r_{\rm th}$ & $r_{\rm pr}$ & Bins & Zero & Fallback & $d$ (count) \\
\midrule
Linear & T-BSTOR-Vec & 0.04509 & -- & -- & 12.8 & 0/50 & 0/50 & 0.673 (50) \\
Linear & T-BSTOR-Mat & 0.02027 & 2.8 & 2.8 & 23.7 & 0/50 & 0/50 & 0.444 (50) \\
Linear & T-BSTOR-Slice & 0.02395 & 17.4 & 17.4 & 23.8 & 0/50 & 0/50 & 0.481 (50) \\
Tanh & T-BSTOR-Vec & 0.05618 & -- & -- & 12.8 & 0/50 & 0/50 & 0.661 (50) \\
Tanh & T-BSTOR-Mat & 0.03267 & 6.6 & 6.6 & 23.7 & 0/50 & 0/50 & 0.413 (50) \\
Tanh & T-BSTOR-Slice & 0.03615 & 24.8 & 24.8 & 23.8 & 0/50 & 0/50 & 0.483 (50) \\
Convex & T-BSTOR-Vec & 0.04427 & -- & -- & 23.6 & 0/50 & 0/50 & 0.677 (50) \\
Convex & T-BSTOR-Mat & 0.02051 & 3.1 & 3.1 & 23.8 & 0/50 & 0/50 & 0.445 (50) \\
Convex & T-BSTOR-Slice & 0.02421 & 18.0 & 18.0 & 23.7 & 0/50 & 0/50 & 0.486 (50) \\
Valley & T-BSTOR-Vec & 0.02533 & -- & -- & 23.9 & 0/50 & 0/50 & 0.804 (50) \\
Valley & T-BSTOR-Mat & 0.007656 & 1.9 & 1.9 & 23.7 & 0/50 & 0/50 & 0.505 (50) \\
Valley & T-BSTOR-Slice & 0.009414 & 8.4 & 8.4 & 23.7 & 0/50 & 0/50 & 0.478 (50) \\
Peak & T-BSTOR-Vec & 0.02574 & -- & -- & 12.9 & 0/50 & 0/50 & 0.797 (50) \\
Peak & T-BSTOR-Mat & 0.005133 & 1.1 & 1.1 & 24.0 & 0/50 & 0/50 & 0.620 (50) \\
Peak & T-BSTOR-Slice & 0.007434 & 7.5 & 7.5 & 12.9 & 0/50 & 0/50 & 0.626 (50) \\
Wave & T-BSTOR-Vec & 0.02476 & -- & -- & 12.8 & 0/50 & 0/50 & 0.777 (50) \\
Wave & T-BSTOR-Mat & 0.004731 & 1.1 & 1.1 & 12.9 & 0/50 & 0/50 & 0.617 (50) \\
Wave & T-BSTOR-Slice & 0.007078 & 7.0 & 7.0 & 12.8 & 0/50 & 0/50 & 0.626 (50) \\
\bottomrule
\end{tabular}
\end{table}

\endgroup

\clearpage
\section{CCLE-derived bandit experiments}
\label{app:ccle}
\suppressfloats[t]

We conduct semisynthetic bandit experiments using a fixed cell--compound--dose response block from the Cancer Cell Line Encyclopedia (CCLE) \citep{barretina2012ccle}. T-ESTOR has the lowest mean Linear regret among the compared methods, including G-LowTESTR with the correct linear link.

\subsection{Data and preprocessing}
\label{app:ccle-data}

\paragraph{Cohort and preprocessing.}
We use Broad's legacy pharmacological release dated February 24, 2015, retrieved on September 20, 2026.\footnote{Broad Institute distribution: \href{https://data.broadinstitute.org/ccle_legacy_data/pharmacological_profiling/CCLE_NP24.2009_Drug_data_2015.02.24.csv}{CCLE legacy pharmacological data}. Release counts need not equal those in the original article.} The file contains 93,251 observed activity values for 504 cell lines, 24 compounds, and eight concentrations. Restricting to the 282 cell lines observed at every compound--dose combination gives a $282\times24\times8$ complete-case tensor without imputation. Concentrations are ordered as $0.0025,0.008,0.025,0.08,0.25,0.8,2.53,8$ $\mu$M. We preserve raw reported activity, including negative values, without clipping, sign reversal, or replacement by an IC50 summary. We make no therapeutic interpretation of the response sign.

\subsection{Fixed bandit environment and evaluation}
\label{app:ccle-design}

\paragraph{Empirical coefficient and simulated feedback.}
We use a fixed, reproducible subset of 16 complete-case cell lines, selected before model fitting and policy comparison. Cell and compound axes are ordered lexicographically, and doses follow the concentration order above. For the resulting $16\times24\times8$ raw response block $Y$, set
\[
 \Theta_* = \frac{Y-\bar Y}{\|Y-\bar Y\|_1},\qquad
 s_0=\frac{\|\Theta_*\|_F}{\sqrt{11}}
     =0.006945307344.
\]
Here $\bar Y$ is the grand mean of the selected block, and the denominator is the entrywise $\ell_1$ norm. This coefficient is not projected onto an exact low-rank class. The DCT methods use a fixed orthonormal DCT-II transform along the dose axis, based on dose order without assuming equally spaced physical concentrations. Each of the $K=10$ fresh candidate tensors has independent entries distributed as $2\operatorname{Beta}(5,5)-1$. Rewards have independent Gaussian noise with standard deviation $s_0/2$. We use
\[
 f_{\rm Linear}(u)=u,\qquad
 f_{\rm Wave}(u)=\tfrac12u+
    \tfrac12s_0\{\cos(2u/s_0)-1\}.
\]
These actions are dense measurements of a CCLE-derived coefficient, not drug prescriptions; feedback is simulated, not replayed laboratory outcomes. The base score is $S_j(x)=8x_j/(1-x_j^2)$. Policies receive their current menus and selected noisy feedback; coefficients, true means, and direction diagnostics remain evaluator-only.

\paragraph{Methods and pairing.}
We compare 12 methods on Linear and ten on Wave at $T=20{,}000$ on ten shared trajectory seeds, totaling 220 trajectories. Methods use the same coefficient, scale, candidate menus, and potential-noise realizations within each seed. Cumulative pseudo-regret uses the best true mean in each realized menu and includes every exploration, initialization, and fallback round. The matrix representations unfold the cell, compound, or dose axis; Mat-Drug is the declared primary unfolding, and all three are reported.

The Gaussian sample-mean T-BSTOR variants use 4,000 uniform observations (2,000 independent pairs), 28 scalar bins, and confidence multiplier $c_{\rm conf}=0.25$. The common scalar configuration was selected on DCT Slice and transferred to the other representations; each representation fits its own direction. The dagger ($\dagger$) in the table and figure marks the DCT Slice variant using median-of-means scalar bins; the other T-BSTOR variants use the Gaussian sample-mean learner. GSTOR and ZoomSIB-UCB retain their own estimators and policies \citep{kang2026single,dey2026optimal}. G-LowTESTR retains tensor GLM estimation and full-coordinate LowGLM-UCB, including complementary-subspace regularization \citep{yi2024efficient}. The method assumes a supplied GLM link. We supply identity on both tasks: it is correct on Linear and misspecified on Wave; the true Wave link is not given to this policy. T-ESTOR and dense ESTOR use monotonicity and are evaluated only on Linear. Tuning opportunities and exploration schedules differ across families.

\paragraph{Epoch methods.}
T-ESTOR uses the DCT slice estimator with $c_\eta=16$ and $c_\tau=0.00015$. ESTOR uses the dense unpaired estimator of \citet{kang2026single}, with initial epoch size $T_0=50$ and coordinatewise clipping multiplier $c_\tau=1$. Each method retains its own doubling-epoch schedule. Their last completed fits use 8,192 observations (4,096 pairs) for T-ESTOR and 6,400 unpaired observations for ESTOR; terminal partial epochs continue with the previous fit. Both methods evaluate the selected-law score using their own estimated direction, an 8,192-point numerical projection oracle, and CDF floor $10^{-12}$. Numerical refinement checks do not provide a uniform tail-error certificate.

\paragraph{Selection and uncertainty.}
Development uses three seeds. Settings were fixed before each reported run; the ten-seed evaluation panel had been inspected during development and does not constitute independent confirmation. All ten seeds are retained. Bandit intervals use 2,000 paired trajectory-seed bootstrap resamples, conditional on the fixed coefficient and settings, without adjustment for multiple comparisons or adaptive development. Curve bands are pointwise, not simultaneous.

% Generated by experiments/export_realdata_appendix.py; do not edit.
\begin{table}[t]
\centering\small
\setlength{\tabcolsep}{5pt}
\caption{CCLE-calibrated cumulative pseudo-regret at $T=20{,}000$: means over ten shared trajectory seeds, with every round included (lower is better). T-ESTOR and ESTOR are monotone-only and inapplicable on Wave (---). G-LowTESTR assumes the identity link, correctly specified on Linear and misspecified on Wave. The seed panel was previously inspected; results are conditional on the fixed coefficient and settings.}
\label{tab:apprealdata-ccle-results}
\begin{tabular}{lrr}
\toprule
Method & Linear & Wave \\
\midrule
T-ESTOR & 57.139 & --- \\
ESTOR & 90.608 & --- \\
G-LowTESTR (identity link) & 59.173 & 42.648 \\
T-BSTOR-Slice-DCT & 102.385 & 46.390 \\
T-BSTOR-Slice-DCT$^\dagger$ & 110.678 & 49.549 \\
T-BSTOR-Slice-Identity & 107.808 & 50.074 \\
T-BSTOR-Mat-Cell & 101.514 & 45.905 \\
T-BSTOR-Mat-Drug & 106.483 & 48.126 \\
T-BSTOR-Mat-Dose & 104.045 & 46.188 \\
T-BSTOR-Vec & 122.918 & 54.333 \\
GSTOR & 116.139 & 55.453 \\
ZoomSIB-UCB & 159.043 & 65.639 \\
\bottomrule
\end{tabular}
\end{table}

\subsection{Bandit results}
\label{app:ccle-results}
\label{app:ccle-limits}

% Generated by experiments/export_realdata_appendix.py; do not edit.
\begin{figure}[t]
\centering
\includegraphics[width=\linewidth]{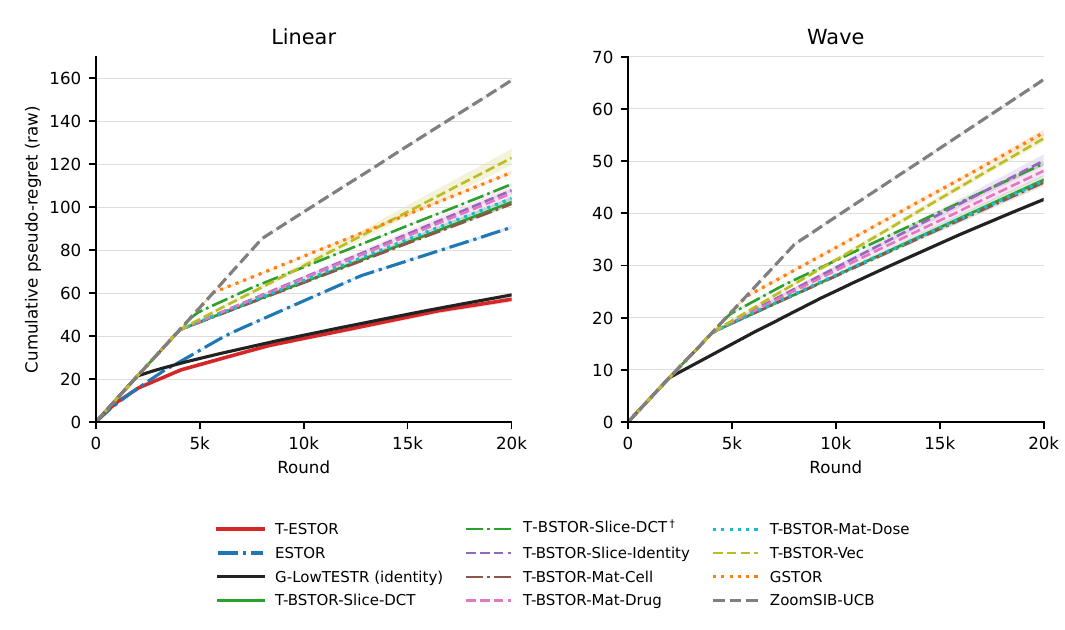}
\caption{CCLE-calibrated mean cumulative pseudo-regret over ten shared trajectory seeds on Linear and Wave. Shading gives pointwise 95\% paired-seed percentile intervals, not simultaneous bands. T-ESTOR and ESTOR are monotone-only; G-LowTESTR assumes identity on both links. Curves connect recorded cumulative-regret values. Results are conditional on the fixed coefficient and settings; intervals do not adjust for the previously inspected seed panel or adaptive configuration selection.}
\label{fig:apprealdata-ccle-regret}
\end{figure}

On Linear, T-ESTOR's mean regret is $57.139$, versus $59.173$ for G-LowTESTR with the correct linear link and $90.608$ for ESTOR (Table~\ref{tab:apprealdata-ccle-results}). Its paired difference from G-LowTESTR is $-2.034$ with interval $[-3.373,-0.746]$: a 3.44\% lower matched mean and lower regret on seven of ten seeds. Relative to ESTOR, the difference is $-33.470$ with interval $[-35.335,-31.775]$, a 36.94\% reduction in matched mean regret. These contrasts remain descriptive under the reused-panel qualifications.

Figure~\ref{fig:apprealdata-ccle-regret} shows regret on both links. T-BSTOR-Slice-DCT has lower mean final regret than Vec, GSTOR, and ZoomSIB-UCB on both tasks. On Wave, its mean is $46.390$, versus $55.453$ for GSTOR and $65.639$ for ZoomSIB-UCB, reductions of 16.3\% and 29.3\%. These gains support its usefulness against the two external unknown-link baselines in this setting. Cell unfolding ($45.905$), Dose unfolding, and identity-link G-LowTESTR ($42.648$) have lower means. T-ESTOR and ESTOR require monotonicity and are not applied to Wave.

G-LowTESTR's lower Wave regret does not imply knowledge of the true link. The Wave mean decomposes as $f_{\rm Wave}(u)=\tfrac12u-s_0\sin^2(u/s_0)$. Under the centrally symmetric exploration distribution, the even nonlinear term has zero linear cross-moment with the action. Thus the unpenalized population least-squares coefficient is $\Theta_*/2$, retaining the true index direction despite misspecification. This identity concerns the base sampling distribution, not the finite-sample regularized or adaptive fit, and does not guarantee the correct nonlinear action ranking. G-LowTESTR also uses 2,000 initial exploration rounds and continued online coefficient updates, whereas T-BSTOR-Slice-DCT uses 4,000 and then fixes its direction while learning scalar bins. These algorithm differences may contribute to the observed ranking; the experiment does not isolate their causal effects.

These results concern one complete-case CCLE coefficient that was not projected onto an exact low-rank class; they establish neither population generalization nor a uniform tensor-over-matrix advantage, and do not directly instantiate a small-rank regret theorem. Numerical scores, Gaussian sample-mean bins, and reduced confidence multipliers do not automatically inherit the exact-score and finite-variance MOM guarantees. Differences in estimators, tuning, and update schedules prevent attributing every regret gap solely to structure. The simulated feedback supports no claim about clinical decisions.

\stopcontents[appendix]
\end{document}